%% file: sample_paper.tex
\documentclass[twoside]{article}

\usepackage[accepted]{aistats2023}
\usepackage[round]{natbib}

\usepackage{amsmath}

\usepackage{algorithm}
\usepackage{algorithmic}
\usepackage{amsthm}
\usepackage{subfigure}
\usepackage{epsfig}
\usepackage{graphicx}
\usepackage{url}
\usepackage{multirow}
\usepackage{amssymb}
\usepackage{amsthm}
\usepackage{mathrsfs}
\usepackage{epstopdf}
\usepackage{multicol}
\usepackage{wrapfig}
\usepackage{wrapfig,lipsum,booktabs}
\usepackage{multirow}
\usepackage{mathtools}

\usepackage{pifont}

\newtheorem{theorem}{\textbf{Theorem}}
\newtheorem{assumption}{\textbf{Assumption}}
\newtheorem{lemma}{\textbf{Lemma}}
\newtheorem{corollary}{\textbf{Corollary}}
\newtheorem{remark}{\textbf{Remark}}

\begin{document}

%

%

\twocolumn[

\aistatstitle{On the Convergence of Distributed Stochastic Bilevel Optimization Algorithms over a Network}

\aistatsauthor{ Hongchang Gao \And Bin Gu \And  My T. Thai }

\aistatsaddress{Temple University \And  MBZUAI \And University of Florida } ]

\begin{abstract}
  	Bilevel optimization has been applied to a wide variety of machine learning models and numerous stochastic bilevel optimization algorithms have been developed in recent years. However, most existing algorithms restrict their focus on the single-machine setting so that they are incapable of handling the distributed data. To address this issue, under the setting where all participants compose a network and perform  peer-to-peer communication in this network, we developed two novel decentralized stochastic bilevel optimization algorithms based on the gradient tracking communication mechanism and two different gradient estimators. Additionally,  we established their convergence rates for nonconvex-strongly-convex problems with novel theoretical analysis strategies. 
  To our knowledge, this is the first work achieving these theoretical results.   Finally, we applied our algorithms to practical machine learning models, and the experimental results confirmed the efficacy of our algorithms. 
\end{abstract}

\section{Introduction}
Bilevel optimization is an important learning paradigm in machine learning.  It consists of an upper-level optimization problem and a lower-level optimization problem, where the objective function of the upper-level optimization problem  depends on the solution of the lower-level one. This kind of learning paradigm covers numerous machine learning models, such as hyperparameter optimization \citep{feurer2019hyperparameter,franceschi2017forward}, meta-learning \citep{franceschi2018bilevel,rajeswaran2019meta}, neural architecture search \citep{liu2018darts}, etc.  Thus, it is of importance and necessity to develop efficient optimization algorithms to solve bilevel optimization problems.

In fact,  the bilevel structure makes it difficult to compute the gradient of the outer-level optimization problem since it involves the computation of  Hessian and Jacobian matrices.  In the past few years, numerous  optimization algorithms have been proposed to address this challenge. For instance, \cite{ghadimi2018approximation,hong2020two,ji2021bilevel,chen2021tighter}  developed  stochastic-gradient-based optimization algorithms, which are able to efficiently estimate Hessian and Jacobian matrices. Recently, to accelerate the convergence speed, \cite{guo2021stochastic} developed a momentum-based optimization algorithm, and \cite{yang2021provably,khanduri2021near,guo2021randomized} proposed the variance-reduced optimization algorithms. Both categories are able to improve the estimation of the full gradient. Thus, they can achieve faster convergence speed than the vanilla stochastic gradient based algorithms \citep{ghadimi2018approximation,hong2020two,ji2021bilevel,chen2021tighter} . However, all these bilevel optimization algorithms restrict their focus on the non-parallel setting. As a result, they are not applicable to the distributed setting.

In this work, we aim to develop decentralized bilevel optimization algorithms to solve the following   bilevel distributed optimization problem:
\begin{equation} \label{loss_bilevel}
	\begin{aligned}
		& \min_{x\in \mathbb{R}^{d_x}} \frac{1}{K}\sum_{k=1}^{K} f^{(k)}(x, y^*(x)), \\
		&\quad s.t. \quad   \ y^*(x) =\arg\min_{y\in \mathbb{R}^{d_y}} \frac{1}{K}\sum_{k=1}^{K} g^{(k)}(x, y) \ , 
	\end{aligned}
\end{equation}
where $x\in \mathbb{R}^{d_x}$ and $y\in \mathbb{R}^{d_y}$ are the model parameters, $g^{(k)}(x, y)=\mathbb{E}_{\zeta\sim \mathcal{S}_g^{(k)} }[g^{(k)}(x, y; \zeta)]$ denotes the objective function of the lower-level subproblem in the $k$-th participant and $\mathcal{S}_g^{(k)}$ is the data distribution of the $k$-th participant for the low-level function, $f^{(k)}(x, y)=\mathbb{E}_{\xi\sim \mathcal{S}_f^{(k)}}[f^{(k)}(x, y; \xi)]$ is the objective function of the upper-level subproblem in the $k$-th participant and correspondingly $\mathcal{S}_f^{(k)}$ is the data distribution of the $k$-th participant for the upper-level function,  $K$ is the total number of the participants. 
From Eq.~(\ref{loss_bilevel}), it is easy to know that each participant possesses its own data, which will be used to learn the model parameters $(x, y)$ via the  collaboration among all participants. In this work, it is assumed that all participants compose a network where the participant performs  peer-to-peer communication. Thus, it is a decentralized bilevel optimization problem. 


Decentralized optimization has been extensively studied in recent years due to its great potential in real-world machine learning tasks, such as data analysis on Internet-of-Things (IoT) devices \citep{gao2023decentralized}. To address the challenges under various settings,  several decentralized optimization algorithms have been proposed.  
For example, \cite{lian2017can} developed the decentralized stochastic gradient descent (DSGD) algorithm based on the gossip communication mechanism and established its convergence rate for nonconvex problems.  \cite{sun2020improving,xin2021hybrid,zhang2021low,zhan2022efficient} proposed the decentralized stochastic variance-reduced gradient descent algorithms based on the gradient tracking communication mechanism, which improve the convergence rate of DSGD. Additionally, some  efforts \citep{koloskova2019decentralized,tang2019deepsqueeze,vogels2020powergossip,li2019communication,gao2020periodic} were made to improve the communication complexity of decentralized optimization algorithms by compressing the communicated variables or skipping the communication round.

However, all aforementioned decentralized optimization algorithms  restrict their focus on the  \textit{single-level} minimization problem, which are not applicable to the \textit{bilevel} optimization problem.  In particular,  bilevel optimization involves the computation of Hessian and Jacobian matrices.  If communicating these two matrices, it will incur a large communication complexity. Thus, it is not clear whether these matrices should be communicated like the gradient. 
Moreover,  considering the interaction between the bilevel structure and the communication mechanism, how will the decentralized bilevel optimization algorithm converge? Particularly, the stochastic hypergradient regarding $x$ of the outer-level objective function is a biased estimation of the full gradient. How does this biased estimator affect the consensus error? All these problems regarding the algorithmic design and theoretical analysis for decentralized bilevel optimization are still  unexplored.


To address the aforementioned problems, we proposed two novel decentralized bilevel optimization algorithms. Specifically, on the algorithmic design side, we developed a momentum-based decentralized stochastic bilevel optimization (MDBO) algorithm, which takes advantage of the momentum  to update model parameters,  and a variance-reduction-based decentralized  stochastic bilevel  optimization (VRDBO) algorithm, which leverages the variance-reduced gradient estimator, {STORM \citep{cutkosky2019momentum}}, to update model parameters. Both of them employ the gradient tracking communication mechanism. Importantly, in our two algorithms, only  model parameters and  gradient estimators are communicated among participants.  In this way, the computation of Hessian and Jacobian matrices is restricted in each participant, avoiding large communication overhead.  On the theoretical analysis side, we established the convergence rate of our two algorithms. Specifically, we investigated how the biased gradient estimator affects the consensus error in the presence of the momentum and variance-reduced gradient. With the help of well-designed potential functions, we show that MDBO achieves the $O(\frac{1}{\epsilon^2(1-\lambda)^2})$ convergence rate and VRDBO enjoys the $O(\frac{1}{K\epsilon^{3/2}(1-\lambda)^2})$ convergence rate to obtain the $\epsilon$-accuracy solution under mild conditions.  We further show that MDBO can also achieve linear speedup when employing stronger assumptions as existing works \citep{yang2022decentralized}. 
In summary, our work have made the following  contributions:
\begin{itemize}
	\item We developed two novel decentralized bilevel optimization algorithms for solving Eq.~(\ref{loss_bilevel}), which demonstrated how to update model parameters locally and communicate them across different participants. 
	\item We established the convergence rate of our proposed algorithms, which demonstrated how the bilevel structure, the gradient estimator, and  the communication mechanism affect the convergence rate. 
	\item We applied our algorithms to the practical machine learning task. The empirical results confirm the superiority of our algorithms. 
\end{itemize}

\section{Related Works}
Bilevel optimization has been widely applied to numerous machine learning applications.  For instance, in the hyperparameter optimization task, the upper-level problem optimizes the hyperparameter and the lower-level problem optimizes the machine learning model's parameter. In the meta-learning task, the upper-level problem learns the task-shared model parameters while the lower-level problem learns the task-specific model parameters \citep{ji2021bilevel}.  When optimizing these kinds of bilevel machine learning models, the challenge lies in 
the computation of the inverse Hessian matrix $(\nabla_{yy}^2 g^{(k)})^{-1}$.  To address this issue, \cite{ghadimi2018approximation} developed a Hessian inverse approximation strategy, which employs stochastic samples to compute an approximation for $(\nabla_{yy}^2 g^{(k)})^{-1}$. Meanwhile, it employs the double-loop mechanism,  where $y$ is updated for multiple times before updating $x$, to obtain a good approximation for $y^*(x)$.  \cite{ji2021bilevel} further employed a large batch size to improve the approximation for $y^*(x)$. On the contrary, \cite{hong2020two} developed a single-loop method, which employs different step sizes for the model parameters $x$ and $y$ such that each update $y$ is a good approximation for the optimal solution $y^*(x)$.  It is worth noting that these single-loop and double-loop algorithms employ stochastic gradients to update model parameters, which suffer from a large estimation variance. 
To address this problem, \cite{guo2021stochastic} developed a single-loop algorithm, which utilizes the momentum to update model parameters. \cite{khanduri2021momentum} proposed another single-loop algorithm, which leverages a variance-reduced gradient estimator to update the model parameter $x$. However, they fail to achieve a better theoretical convergence rate than the vanilla stochastic gradient based algorithms. Recently, \cite{yang2021provably,khanduri2021near,guo2021randomized}  resorted to more advanced variance-reduced gradient estimators to  accelerate the convergence rate. Specifically, \cite{yang2021provably,khanduri2021near,guo2021randomized}  combined the STORM \citep{cutkosky2019momentum} gradient estimator and the single-loop mechanism so that they can achieve a better theoretical convergence rate than the stochastic-gradient-based algorithm and the momentum-based algorithm. Moreover, \cite{yang2021provably} combined the SPIDER \citep{fang2018spider} gradient estimator and the double-loop mechanism, which actually can achieve the same theoretical convergence rate with that based on STORM. However, all these algorithms only investigate the non-parallel situation. Thus, their theoretical analysis does not hold anymore for the distributed scenario.

Decentralized optimization has also been applied to a wide variety of machine learning applications in recent years. Compared to the parameter-server setting, the decentralized communication is robust to the single-node failure since the participant conducts  peer-to-peer communication. Recently,  \cite{lian2017can} investigated the convergence rate of the standard decentralized stochastic gradient descent (SGD) algorithm for nonconvex problems. \cite{yu2019linear} developed a decentralized stochastic gradient descent with  momentum algorithm, which has the same theoretical convergence rate as \citep{lian2017can}.  \cite{pu2021distributed,lu2019gnsd} developed a decentralized SGD based on the gradient tracking communication mechanism. Later, some variance-reduced algorithms were proposed to accelerate the convergence rate. For instance, \cite{sun2020improving} utilized the SPIDER \citep{fang2018spider} gradient estimator, \cite{xin2021hybrid,zhang2021low} employed the STORM \citep{cutkosky2019momentum} gradient estimator, and \cite{zhan2022efficient} resorted to the ZeroSARAH \citep{li2021zerosarah} gradient estimator for improving the sample and communication complexities. Additionally, some works  consider decentralized optimization problems that are beyond standard minimization problems, e.g., compositional optimization problems \citep{gao2021fast}, minimax optimization problems \citep{tsaknakis2020decentralized,xian2021faster,ZhangLLZL21,gao2022decentralized}, constraint optimization problems \citep{wai2017decentralized,mokhtari2018decentralized,gao2021sample}. However, all these decentralized optimization algorithms are not applicable to the decentralized bilevel optimization problem. On the one hand, they focus on the single-level problem. Thus, their theoretical analysis is incapable of handling the interaction between two levels of functions. On the other hand, those algorithms are based on the standard stochastic gradient, which is an unbiased estimator of the full gradient. On the contrary, the stochastic hypergradient  is a biased estimator, which incurs new challenges when bounding the consensus error. Thus, it is necessary to develop new theoretical analysis strategies to investigate the convergence rate of decentralized bilevel optimization algorithms. 

Recently, we are aware of two concurrent works \citep{chen2022decentralized,yang2022decentralized} on decentralized bilevel optimizaiton. In particular, \cite{chen2022decentralized} developed the decentralized optimization algorithms for Eq.~(\ref{loss_bilevel}) based on the \textit{full} gradient and the vanilla \textit{stochastic gradient}. On the contrary, our work leverages advanced gradient estimators, i.e., momentum and variance-reduced gradient, which makes our theoretical analysis more challenging.  As for \citep{yang2022decentralized}, it requires to communicate  Hessian and Jacobian matrices, while our methods do not require that. Thus, our method is more efficient in communication. Moreover, \cite{yang2022decentralized} employed the gossip communication mechanism, while our methods leverage the gradient tracking scheme and advanced gradient estimators. As a result, the theoretical analysis of our methods is more challenging.  \textit{It is worth noting that both of them \citep{chen2022decentralized,yang2022decentralized} require much stronger assumption for the loss function. In particular, they  require that the upper-level objective function is Lipschitz continuous with respect to $x$ and the lower-level objective function is Lipschitz continuous with respect to $y$ (See Assumption 2.1 in \citep{chen2022decentralized}, Assumption 3.3 (iii) and 3.4 (iv) in \citep{yang2022decentralized}). } On the contrary, our theoretical analysis does not require these two strong assumptions, which incurs more challenges for theoretical analysis.   All in all, our work is significantly different from these two concurrent works and the theoretical analysis is more challenging than them.

\section{Preliminaries}

\textbf{Stochastic Hypergradient.} Throughout this paper, we denote  $F^{(k)}(x) = f^{(k)}(x, y^*(x)) $ and $F(x) = \frac{1}{K}\sum_{k=1}^{K}F^{(k)}(x)$ and assume all participants have  i.i.d. datesets. Then, according to Lemma~1 of \citep{gao2022convergence},  we can compute the  gradient of $F^{(k)}(x)$ as follows:
\begin{equation}
	\begin{aligned}
		& \nabla F^{(k)}(x)  = \nabla_x f^{(k)}(x, y^*(x)) \\
		&   - \nabla_{xy}^2 g^{(k)}(x, y^*(x)) H_*^{-1} \nabla_y f^{(k)}(x, y^*(x)) \ , 
	\end{aligned}
\end{equation}
where  $\nabla_{xy}^2 g^{(k)}(x, y^*(x))$ is  Jacobian matrix and $H_*=\nabla_{yy}^2g^{(k)}(x, y^*(x))$ is  Hessian matrix. Note that $ \nabla F^{(k)}(x)$ is also called \textit{hypergradient}.  Since $y^*(x)$ is typically not easy to obtain in each iteration, following \citep{ghadimi2018approximation}, we can  approximate it as follows:
\begin{equation}
	\begin{aligned}
		&  \nabla F^{(k)}(x, y)  = \nabla_x f^{(k)}(x, y)\\
		& \quad -  \nabla_{xy}^2g^{(k)}(x, y)H^{-1}\nabla_y f^{(k)}(x, y)  \ , \\
	\end{aligned}
\end{equation}
where $H=\nabla_{yy}^2g^{(k)}(x, y)$.  Note that we just use $\nabla F^{(k)}(x, y)$ to approximate $\nabla F^{(k)}(x) $. It does not mean we introduce a function $F^{(k)}(x, y)$. 
Here, because the inverse of Hessian matrix is difficult to compute, following the Hessian inverse approximation strategy proposed in  \citep{ghadimi2018approximation}, we can use the following stochastic hypergradient to  approximate it:
\begin{equation}
	\begin{aligned}
		& \nabla \tilde{F}^{(k)}(x, y; \tilde{\xi})  = \nabla_x f^{(k)}(x, y; \xi) \\
		& \quad -  \nabla_{xy}^2g^{(k)}(x, y; \zeta_0) \frac{J}{L_{g_{y}}}\tilde{H}_{\tilde{J}}\nabla_y f^{(k)}(x, y; \xi)  \ , \\
	\end{aligned}
\end{equation}
where $\tilde{H}_{\tilde{J}}=\prod_{j=1}^{\tilde{J}}(I-\frac{1}{L_{g_{y}}}\nabla_{yy}^2g^{(k)}(x, y; \zeta_j))$, $L_{g_y}$ is the Lipschitz-continuous constant, which is defined in Assumption~\ref{assumption_lower_smooth}, $\tilde{\xi}=\{\xi, \zeta_0, \zeta_1, \cdots, \zeta_{J}\}$ and $\tilde{J}$ are randomly selected from $\{0, 1, 2, \cdots, J\}$  where $J$ is a positive integer. Note that we let $\prod_{j=1}^{\tilde{J}}(I-\frac{1}{L_{g_{y}}}\nabla_{yy}^2g^{(k)}(x, y; \zeta_j))=I$ when $\tilde{J}=0$.  Moreover, we denote the expectation of the stochastic hypergradient as follows:
\begin{equation}
	\begin{aligned}
		& \nabla \tilde{F}^{(k)}(x, y) \triangleq  \mathbb{E}[\nabla \tilde{F}^{(k)}(x, y; \tilde{\xi})]=\nabla_x f^{(k)}(x, y) \\
		& \quad -  \nabla_{xy}^2g^{(k)}(x, y) \mathbb{E}\Big[\frac{J}{L_{g_{y}}}\tilde{H}_{\tilde{J}}\Big]\nabla_y f^{(k)}(x, y)  \ , \\
	\end{aligned}
\end{equation}
Since  $\nabla \tilde{F}^{(k)}(x, y)\neq \nabla F^{(k)}(x, y)$, the stochastic hypergradient $\nabla \tilde{F}^{(k)}(x, y; \tilde{\xi})$ is a biased estimator for $\nabla F^{(k)}(x, y)$. The detailed bias is shown in Lemma~\ref{lemma_hypergrad_bias}. 

\textbf{Notations.} Throughout this paper,  $x_{t}^{(k)}$ and $y_{t}^{(k)}$ denote the model parameters of the $k$-th participant in the $t$-th iteration.  Moreover, we denote $\bar{x}_t =  \frac{1}{K}\sum_{k=1}^{K} x_{t}^{(k)}$ and $\bar{y}_t =  \frac{1}{K}\sum_{k=1}^{K} y_{t}^{(k)}$.  Additionally,  we denote
\begin{equation}
	\begin{aligned}
		& \Delta_t^{\tilde{F}_{\tilde{\xi}_t}} = [\nabla\tilde{F}^{(1)}(x_t^{(1)}, y_t^{(1)}; \tilde{\xi}_t^{(1)}),  \nabla\tilde{F}^{(2)}(x_t^{(2)}, y_t^{(2)}; \tilde{\xi}_t^{(2)}), \\
		& \quad \quad \quad \quad  \cdots,  \nabla\tilde{F}^{(k)}(x_t^{(K)}, y_t^{(K)}; \tilde{\xi}_t^{(K)})]  \ , \\
		& \Delta_t^{g_{\zeta_t}} = [\nabla_y g^{(k)}(x_t^{(1)}, y_t^{(1)}; \zeta_t^{(1)}), \nabla_y g^{(k)}(x_t^{(2)}, y_t^{(2)}; \zeta_t^{(2)}), \\
		& \quad \quad \quad \quad \cdots,  \nabla_y g^{(k)}(x_t^{(K)}, y_t^{(K)}; \zeta_t^{(K)})] \ , \\ 
		& X_t=[x_t^{(1)}, x_t^{(2)},  \cdots, x_t^{(K)}] \ , Y_t=[y_t^{(1)}, y_t^{(2)},  \cdots, y_t^{(K)}] \ .
	\end{aligned}
\end{equation}
Furthermore, the adjacency matrix of the communication network is denoted by $W=[w_{ij}]\in \mathbb{R}_{+}^{K\times K}$, where $w_{ij}>0$ indicates the $i$-th participant is connected with the $j$-th participant and otherwise $w_{ij}=0$. The adjacency matrix satisfies the following assumption. 
\begin{assumption} \label{assumption_graph}
	$W$ satisfies $W^T=W$ and $W\mathbf{1}=\mathbf{1}$. Its eigenvalues satisfy $|\lambda_n|\leq \cdots \leq |\lambda_2|< |\lambda_1|=1$.
\end{assumption} 
Then, the spectral gap of $W$ can be represented by $1-\lambda$ where $\lambda \triangleq |\lambda_2|$ and $1-\lambda \in (0, 1]$. 

\begin{algorithm}[]
	\caption{MDBO}
	\label{alg_MDBO}
	\begin{algorithmic}[1]
		\REQUIRE ${x}_{0}^{(k)}={x}_{0}$, ${y}_{0}^{(k)}={y}_{0}$, $\eta>0$, $\alpha_1>0$, $\alpha_2>0$, $\beta_1>0$, $\beta_2>0$.
		\FOR{$t=0,\cdots, T-1$} 
		\IF {$t==0$}
		\STATE $U_t  = \Delta_t^{\tilde{F}_{\tilde{\xi}_t}}$,  $V_t = \Delta_t^{g_{\zeta_t}}$, 
		$Z_t^{\tilde{F}}=\Delta_t^{\tilde{F}_{\tilde{\xi}_t}}$, $Z_t^{g}= \Delta_t^{g_{\zeta_t}}$, \\
		\ELSE
		\STATE $U_{t}  = (1-\alpha_1\eta)U_{t-1}+\alpha_1\eta  \Delta_t^{\tilde{F}_{\tilde{\xi}_t}}$,   \\
		$V_{t} = (1-\alpha_2\eta)V_{t-1}+\alpha_2 \eta\Delta_t^{g_{\zeta_t}}$, \\
		\STATE$Z_t^{\tilde{F}} = Z_{t-1}^{\tilde{F}}W + U_{t} - U_{t-1}$ ,  \\
		$Z_t^{g} = Z_{t-1}^{g}W + V_{t} - V_{t-1}$ , \\
		\ENDIF
		\STATE 
		$X_{t+1}=X_{t}-\eta X_{t} (I-W) - \beta_1\eta Z_t^{\tilde{F}} $ ,  \\
		$Y_{t+1}=Y_{t}-\eta Y_{t} (I-W) - \beta_2\eta Z_t^{g} $,  \\
		\ENDFOR
	\end{algorithmic}
\end{algorithm}

\section{Decentralized Stochastic Bilevel Optimization Algorithms}

\paragraph{Momentum-based Decentralized Stochastic Bilevel Optimization Algorithm.}
In Algorithm~\ref{alg_MDBO}, we developed a momentum-based decentralized stochastic bilevel optimization (MDBO) algorithm. The main idea is to use the momentum to update the model parameters $x$ and $y$ at each participant and then perform communication. 
Specifically, the momentum is  updated as follows:
\begin{equation}
	\begin{aligned}
		& U_{t}  = (1-\alpha_1\eta)U_{t-1}+\alpha_1\eta  \Delta_t^{\tilde{F}_{\tilde{\xi}_t}},  \\
		&    V_{t} = (1-\alpha_2\eta)V_{t-1}+\alpha_2 \eta\Delta_t^{g_{\zeta_t}} \ ,
	\end{aligned}
\end{equation}
where $\alpha_1$, $\alpha_2$, and $\eta$ are positive, $\alpha_1\eta<1$,   $\alpha_2\eta<1$,  $U_{t}=[u_t^{(1)}, u_t^{(2)}, \cdots, u_t^{(K)}]\in \mathbb{R}^{d_x\times K}$ is the momentum of the stochastic hypergradient $\Delta_t^{\tilde{F}_{\tilde{\xi}_t}}$,  and $V_t=[v_t^{(1)}, v_t^{(2)}, \cdots, v_t^{(K)}]\in \mathbb{R}^{d_y\times K}$ is the momentum of the stochastic gradient $\Delta_t^{g_{\zeta_t}}$.  Here, $u_t^{(k)}$ and $v_t^{(k)}$ are the momentum in the $k$-th participant.  Their updates are restricted in the corresponding participant. Then, MDBO employs the gradient tracking communication mechanism to exchange the momentum and model parameter across participants. In detail,  $Z_t^{\tilde{F}}\in \mathbb{R}^{d_x\times K}$ and $Z_t^{g}\in \mathbb{R}^{d_y\times K}$ are the tracked momentum for $U_t$ and $V_t$, respectively. In the first iteration, they are initialized as the stochastic gradient as shown in Line 3 of Algorithm~\ref{alg_MDBO}. In other iterations, they are updated as follows:
\begin{equation}
	\begin{aligned}
		&	Z_t^{\tilde{F}} = Z_{t-1}^{\tilde{F}}W + U_{t} - U_{t-1}, \\
		&   Z_t^{g} = Z_{t-1}^{g}W + V_{t} - V_{t-1} \ ,
	\end{aligned}
\end{equation}
where $Z_{t-1}^{\tilde{F}}W$ denotes the communication operation.  Based on the tracked momentum, the model parameters $x$ and $y$ are updated as follows:
\begin{equation}
	\begin{aligned}
		& X_{t+1}=X_{t}-\eta X_{t} (I-W) - \beta_1\eta Z_t^{\tilde{F}}\ , \\
		&  Y_{t+1}=Y_{t}-\eta Y_{t} (I-W) - \beta_2\eta Z_t^{g}  \ , 
	\end{aligned}
\end{equation}
where $\eta\in (0, 1)$,  $\beta_1$  and $\beta_2$ are positive, $X_{t}W$ and $Y_{t}W$ indicate the communication of model parameters across participants.  In fact, by reformulating this updating rule, it is easy to know that  $X_{t+1}$ is  the combination of the local model parameter $X_{t}$ and the  update $X_{t}W - \beta_1 Z_t^{\tilde{F}}$  that is based on the neighboring participants' information.
In summary, the computation of stochastic gradients/hypergradients, Hessian matrix, and Jacobian matrix is restricted in each participant. Only the momentum and model parameters are communicated across participants.

\begin{algorithm}[]
	\caption{VRDBO}
	\label{alg_VRDBO}
	\begin{algorithmic}[1]
		\REQUIRE ${x}_{0}^{(k)}={x}_{0}$, ${y}_{0}^{(k)}={y}_{0}$, $\eta>0$, $\alpha_1>0$, $\alpha_2>0$, $\beta_1>0$, $\beta_2>0$.
		\FOR{$t=0,\cdots, T-1$} 
		\IF {$t==0$}
		\STATE With the mini-batch size $B$: \\
		$U_t  = \Delta_t^{\tilde{F}_{\tilde{\xi}}}$,  $V_t = \Delta_t^{g_{\zeta}}$,  $Z_t^{\tilde{F}}=\Delta_t^{\tilde{F}_{\tilde{\xi}}}$, $Z_t^{g}= \Delta_t^{g_{\zeta}}$ , \\
		\ELSE
		\STATE $U_{t}  = (1-\alpha_1\eta^2)(U_{t-1}+\Delta_t^{\tilde{F}_{\tilde{\xi}_t}} - \Delta_{t-1}^{\tilde{F}_{\tilde{\xi}_t}})+\alpha_1\eta^2  \Delta_t^{\tilde{F}_{\tilde{\xi}_t}}$, \\
		$V_{t} = (1-\alpha_2\eta^2)(V_{t-1}+\Delta_t^{g_{\zeta_t}} - \Delta_{t-1}^{g_{\zeta_t}})+\alpha_2 \eta^2\Delta_t^{g_{\zeta_t}}$, \\
		\STATE $Z_t^{\tilde{F}} = Z_{t-1}^{\tilde{F}}W + U_{t} - U_{t-1}$ ,  \\
		$Z_t^{g} = Z_{t-1}^{g}W + V_{t} - V_{t-1}$ , \\
		\ENDIF
		\STATE 
		$X_{t+1}=X_{t}-\eta X_{t} (I-W) - \beta_1\eta Z_t^{\tilde{F}}$ , \\
		$Y_{t+1}=Y_{t}-\eta Y_{t} (I-W) - \beta_2\eta Z_{t}^{g}$,  \\
		\ENDFOR
	\end{algorithmic}
\end{algorithm}

\paragraph{Variance-Reduction-based Decentralized Stochastic Bilevel Optimization Algorithm.}
Existing non-parallel algorithms have shown that the momentum-based approach does not achieve a better \textit{theoretical} convergence rate even though it demonstrates better \textit{empirical} convergence performance \citep{guo2021stochastic,khanduri2021momentum}.  Thus, we further developed a new algorithm: variance-reduction-based decentralized stochastic bilevel optimization (VRDBO) algorithm, which takes advantage of the variance-reduced gradient estimator to accelerate the convergence rate. The details are shown in Algorithm~\ref{alg_VRDBO}. Specifically, VRDBO utilizes the STORM \citep{cutkosky2019momentum} gradient estimator to control the variance of stochastic gradients/hypergradients  as follows:
\begin{equation}
	\begin{aligned}
		& U_{t}  = (1-\alpha_1\eta^2)(U_{t-1}+\Delta_t^{\tilde{F}_{\tilde{\xi}_t}} - \Delta_{t-1}^{\tilde{F}_{\tilde{\xi}_t}})+\alpha_1\eta^2  \Delta_t^{\tilde{F}_{\tilde{\xi}_t}} \ ,
	\end{aligned}
\end{equation}
where $\alpha>0$, $\eta>0$, and $\alpha\eta^2\in (0, 1)$. Note that $\Delta_{t-1}^{\tilde{F}_{\tilde{\xi}_t}}$ denotes the stochastic hypergradient which is computed based on the model parameters $X_{t-1}$ and $Y_{t-1}$ in the $(t-1)$-th iteration, as well as the selected samples in the $t$-th iteration. $V_t$ is updated in the same way.  Then, based on this variance-reduced gradient estimator, each participant leverages the gradient tracking communication mechanism to exchange the tracked gradients and model parameters to update local model parameters, which is shown in Lines 6 and 8 of Algorithm~\ref{alg_VRDBO}. 

In summary,  VRDBO utilizes a variance-reduced gradient estimator to control the variance of stochastic gradients. Thus, it can achieve a better convergence rate than MDBO, which will be  shown in next section.  Additionally, both MDBO and VRDBO do NOT require to communicate Hessian and Jacobian matrices. Only model parameters and tracked gradients are communicated.  To the best of our knowledge,  we are the first one developing the variance-reduced decentralized bilevel optimization algorithm.

\section{Convergence Analysis}
To investigate the convergence rate of our two algorithms, we first introduce two common assumptions for both algorithms and then introduce  algorithm-specific assumptions. 

\begin{assumption} \label{assumption_bi_strong}
	For any  fixed $x\in \mathbb{R}^{d_x}$ and $k\in \{1,2,\cdots, K\}$, the lower-level function $g^{(k)}(x, y)$ is $\mu$-strongly convex with respect to  $y$.
\end{assumption}
\begin{assumption} \label{assumption_variance}
	For any  $k\in \{1,2,\cdots, K\}$, the first and second order stochastic gradients of all loss functions have bounded variance $\sigma^2$ where $\sigma>0$.
\end{assumption}

\subsection{Convergence Rate of Algorithm~\ref{alg_MDBO}}
Similar to the non-parallel algorithms \citep{ghadimi2018approximation,chen2021tighter,hong2020two,ji2021bilevel},  our Algorithm~\ref{alg_MDBO}  requires a weaker assumption regarding the smoothness of the loss function compared with Algorithm~\ref{alg_VRDBO}, which is shown as follows.  

\begin{assumption} \label{assumption_upper_smooth}
	For any  $k\in \{1,2,\cdots, K\}$,   $\nabla_x f^{(k)}(x, y)$ is Lipschitz continuous with the constant  $L_{f_x}>0$, 	$\nabla_y f^{(k)}(x, y)$  is Lipschitz continuous with the constant $L_{f_y}>0$.
	Moreover,  $\|\nabla_y f^{(k)}(x, y)\|\leq C_{f_y}$ with the constant $C_{f_y}>0$ for $(x, y)\in\mathbb{R}^{d_x}\times \mathbb{R}^{d_y} $. 
	
\end{assumption}

\begin{assumption} \label{assumption_lower_smooth}
	For any  $k\in \{1,2,\cdots, K\}$,  $\nabla_y g^{(k)}(x, y)$ is  Lipschitz continuous with the constant  $L_{g_y}>0$, $\nabla_{xy}^2 g^{(k)}(x, y)$ is Lipschitz continuous with the constant $L_{g_{xy}}>0$,  $\nabla_{yy}^2 g^{(k)}(x, y)$  is Lipschitz continuous with the constant $L_{g_{yy}}>0$.
	Moreover,  $\|\nabla_{xy}^2 g^{(k)}(x, y)\|\leq C_{g_{xy}}$ with the constant $C_{g_{xy}}>0$ and  $\mu \mathbf{1} \preceq \nabla_{yy}^2 g^{(k)}(x, y; \zeta)\preceq  L_{g_y} \mathbf{1}$  for  $(x, y)\in\mathbb{R}^{d_x}\times \mathbb{R}^{d_y} $. 
\end{assumption}

Based on these assumptions, we are able to establish the convergence rate of Algorithm~\ref{alg_MDBO} as follows.
\begin{theorem} \label{theorem_mdbo}
	Given Assumptions~\ref{assumption_graph}-\ref{assumption_lower_smooth}, if $\alpha_1>0$, $\alpha_2>0$, $\eta<\min \{1, \frac{1}{2\beta_1L_{F}^{*}}, \frac{1}{\alpha_1}, \frac{1}{\alpha_2}\}$, $\beta_1 \leq \min \{\beta_{1,a} , \beta_{1,b}, \beta_{1,c} \}$,  and $\beta_2\leq \min\{\beta_{2,a} , \beta_{2,b}, \beta_{2,c}\}$,  where
	\begin{equation}
		\begin{aligned}
			& \beta_{1,a} = \frac{\beta_2\mu} {15L_{y} {L}_{F}} \ , \beta_{2,c} = \frac{1}{6L_{g_y}}  \ ,  \\
			& \beta_{1,b} = \frac{\mu}{4L_{g_y}\sqrt{((2+8/\alpha_1^2)L_{\tilde{F}}^2+(100+200/\alpha_2^2){L}_{F}^2)}} \ , \\
			& \beta_{1,c} = \frac{\mu(1-\lambda)^2}{ 4L_{g_y}\sqrt{(6+18/\alpha_1^2){L}_{\tilde{F}}^2+(250+450/\alpha_2^2){L}_{F}^2}}  \ , \\
			& \beta_{2,a} =\frac{9\mu{L}_{F}^2}{2((4+16/\alpha_1^2)L_{\tilde{F}}^2+(200+400/\alpha_2^2){L}_{F}^2)L_{g_y}^2}  \ ,  \\
			& \beta_{2,b} =\frac{5 (1-\lambda)^2{L}_{F}}{ 2L_{g_y}\sqrt{(12+36/\alpha_1^2){L}_{\tilde{F}}^2+(500+900/\alpha_2^2){L}_{F}^2}}  \ , \\
		\end{aligned}
	\end{equation}
	the convergence rate of Algorithm~\ref{alg_MDBO} is 
	\begin{equation}
		\begin{aligned}
			& \quad \frac{1}{T}\sum_{t=0}^{T-1}(\mathbb{E}[\|\nabla F(\bar{  {x}}_{t})\|^2]  + L_F^2 \mathbb{E}[\|\bar{y}_{t} - y^*(\bar{x}_{t})\|^2] ) \\
			& \leq \frac{2(F(x_0) -F(x_*))}{\eta\beta_1T} +  \frac{12{L}_{F}^2}{\beta_2\mu\eta T}\|\bar{   {y}}_{0} -    {y}^{*}(\bar{   {x}}_{0})\| ^2  \\
			& \quad +  \frac{6C_{g_{xy}}^2C_{f_y}^2}{\mu^2}(1-\frac{\mu}{L_{g_{y}}})^{2J} + \frac{250 \alpha_2\eta{L}_{F}^2\sigma^2}{\mu^2}     \\
			& 		\quad +9\alpha_1\eta\sigma_{\tilde{F}}^2 +  \frac{10\sigma_{\tilde{F}}^2 }{\alpha_1\eta T} +\frac{300{L}_{F}^2\sigma^2}{\alpha_2 \mu^2\eta T} \ ,  \\
		\end{aligned}
	\end{equation}
	where the definition of  $L_y$, $L_F^*$, $L_F$, $L_{\tilde{F}}$, $\sigma_{\tilde{F}}$ is shown in Lemmas~\ref{lemma_hypergrad_smooth_optimal},~\ref{lemma_hypergrad_var},~\ref{lemma_hypergrad_smooth}. 
\end{theorem}

\begin{corollary} \label{corollary_mdbo}
	Given the same condition with Theorem~\ref{theorem_mdbo}, by choosing $T=O(\frac{1}{\epsilon^2(1-\lambda)^2})$, $\eta=O(\epsilon)$, $J=O(\log \frac{1}{\epsilon})$, $\beta_1=O((1-\lambda)^2)$,  $\beta_2=O((1-\lambda)^2)$, $\alpha_1=O(1)$, and $\alpha_2=O(1)$, Algorithm~\ref{alg_MDBO} can achieve the $\epsilon$-accuracy solution: $\frac{1}{T}\sum_{t=0}^{T-1}(\mathbb{E}[\|\nabla F(\bar{  {x}}_{t})\|^2  + L_F^2\|\bar{y}_{t} - y^*(\bar{x}_{t})\|^2] )\leq O(\epsilon)$. 
	Then, the communication complexity is $O(\frac{1}{\epsilon^2(1-\lambda)^2})$, the gradient complexity and Jacobian-vector product complexity is $O(\frac{1}{\epsilon^2(1-\lambda)^2})$, and the Hessian-vector product complexity is $\tilde{O}(\frac{1}{\epsilon^2(1-\lambda)^2})$.  
\end{corollary}

\paragraph{Discussion.}
The concurrent work \citep{chen2022decentralized} developed a decentralized stochastic bilevel optimization algorithm based on the vanilla stochastic gradient (See Algorithm 5 in \citep{chen2022decentralized}). Another concurrent work \citep{yang2022decentralized} developed a momentum-based decentralized stochastic bilevel optimization algorithm, which employs the gossip communication mechanism and requires to communicate  stochastic gradients/hypergradients, Hessian matrix, and Jacobian matrix. Obviously, on the algorithmic design side, our Algorithm~\ref{alg_MDBO} is significantly different from those two works. Moreover, on the theoretical analysis side, those two works have much stronger assumptions. In particular,  their theoretical analyses require that the  second-order moments of (stochastic) hypergradient regarding $x$ of the upper-level function and the (stochastic) gradient regarding $y$ of the low-level function  are upper bounded, which is shown in Assumption~\ref{assumption_continuous}.

\begin{assumption} \label{assumption_continuous}
	For any  $k\in \{1,2,\cdots, K\}$, $(x, y)\in\mathbb{R}^{d_x}\times \mathbb{R}^{d_y} $,   $\nabla_x f^{(k)}(x, y)$  and $\nabla_y g^{(k)}(x, y)$  satisfies:
	\begin{equation}
		\begin{aligned}
			& \|\nabla_x f^{(k)}(x, y)\|\leq C_{f_x} \ ,  \ \|\nabla_y g^{(k)}(x, y)\|\leq C_{g_y} \ , \\
		\end{aligned}
	\end{equation}
	where  the constant $C_{f_x}>0$ and  $C_{g_y}>0$. 
\end{assumption}
With this additional assumption,  the convergence rate in \citep{yang2022decentralized}  is able to achieve  linear speedup with respect to the number of participants, while that in \citep{chen2022decentralized} fails to achieve  linear  speedup even with this strong assumption (See Lemma A.35 and Theorem 3.3 in \citep{chen2022decentralized}).

In the following, we show that our Algorithm~\ref{alg_MDBO} can also achieve the linear speedup effect  as \citep{yang2022decentralized} with this  additional assumption. 
\begin{theorem} \label{theorem_mdbo_bounded_gradient_norm}
	Given Assumptions~\ref{assumption_graph}-\ref{assumption_lower_smooth} and  \textbf{Assumption~\ref{assumption_continuous}}, if $\alpha_1>0$, $\alpha_2>0$, $\eta<\min \{1, \frac{1}{2\beta_1L_{F}^{*}}, \frac{1}{\alpha_1}, \frac{1}{\alpha_2}\}$, $\beta_1 \leq \min \{ \frac{\beta_2\mu} {15L_{y} {L}_{F}}, \frac{\mu} {4{L}_{g_{y}}\sqrt{6{L}_{\tilde{F}}^2/\alpha_1^2+100L_F^2/\alpha_2^2}} \}$,  and $\beta_2 \leq  \min\{\frac{1}{6L_{g_y}}, \frac{9\mu L_F^2}{2(12{L}_{\tilde{F}}^2/\alpha_1^2+200L_F^2/\alpha_2^2){L}_{g_y}^2}\}$, 
	the convergence rate of Algorithm~\ref{alg_MDBO} is 
	\begin{equation}
		\begin{aligned}
			&  \quad \frac{1}{T}\sum_{t=0}^{T-1}( \mathbb{E}[\|\nabla F(\bar{  {x}}_{t})\|^2] + L_F^2 \mathbb{E}[\|\bar{y}_{t} - y^*(\bar{x}_{t})\|^2]) \\
			& \leq \frac{2(F(x_0)- F(x_*))}{\eta\beta_1T}  +   \frac{12{L}_{F}^2}{\beta_2\mu\eta T}\|\bar{   {y}}_{0} -    {y}^{*}(\bar{   {x}}_{0})\| ^2 \\
			& \quad  +  \frac{6C_{g_{xy}}^2C_{f_y}^2}{\mu^2}(1-\frac{\mu}{L_{g_{y}}})^{2J}+    \frac{24 \hat{C}_{\tilde{F}}^2}{\alpha_1\eta T} + \frac{400 {L}_{F}^2\hat{C}_{g_y}^2}{\alpha_2\mu^2 \eta T}  \\
			& \quad +\frac{48\alpha_1^2 \beta_1^2\eta^2\hat{C}_{\tilde{F}}^2{L}_{\tilde{F}}^2}{(1-\lambda)^4}+ \frac{48\alpha_2^2\beta_2^2\eta^2\hat{C}_{g_y}^2{L}_{\tilde{F}}^2}{(1-\lambda)^4}\\
			& \quad + \frac{800\alpha_1^2 \beta_1^2\eta^2\hat{C}_{\tilde{F}}^2 L_{g_y}^2L_F^2}{\mu^2(1-\lambda)^4} +  \frac{800 \alpha_2^2\beta_2^2\eta^2\hat{C}_{g_y}^2L_F^2L_{g_y}^2}{\mu^2(1-\lambda)^4}\\
			& \quad +\frac{432 \beta_1^2\eta^2\hat{C}_{\tilde{F}}^2{L}_{\tilde{F}}^2}{(1-\lambda)^4}  +\frac{432\alpha_2^2\beta_2^2\eta^2\hat{C}_{g_y}^2{L}_{\tilde{F}}^2}{\alpha_1^2(1-\lambda)^4} \\
			& \quad +\frac{7200\alpha_1^2 \beta_1^2\eta^2\hat{C}_{\tilde{F}}^2L_F^2{L}_{g_{y}}^2}{\alpha_2^2\mu^2(1-\lambda)^4}+\frac{7200\beta_2^2\eta^2\hat{C}_{g_y}^2L_F^2{L}_{g_y}^2}{\mu^2(1-\lambda)^4}  \\
			& \quad   +\frac{6\alpha_1\eta \sigma_{\tilde{F}}^2}{K}  + \frac{ 100\alpha_2\eta \sigma^2 L_F^2}{\mu^2K} \ ,  \\
		\end{aligned}
	\end{equation}
	where the definition of  $L_y$, $L_F^*$, $L_F$, $L_{\tilde{F}}$, $\sigma_{\tilde{F}}$, $\hat{C}_{\tilde{F}}$, $\hat{C}_{g_y}$ is shown in Lemmas~\ref{lemma_hypergrad_smooth_optimal},~\ref{lemma_hypergrad_var},~\ref{lemma_hypergrad_smooth}, ~\ref{lemma_grad_norm}. 
\end{theorem}

\begin{corollary}
	Given the same condition with Theorem~\ref{theorem_mdbo_bounded_gradient_norm}, by choosing $T=O(\frac{1}{K\epsilon^2})$, $\eta=O(K\epsilon)$, $J=O(\log \frac{1}{\epsilon})$, $\beta_1=O(1)$, $\beta_2=O(1)$, $\alpha_1=O(1)$, $\alpha_2=O(1)$,  Algorithm~\ref{alg_MDBO} can achieve the $\epsilon$-accuracy solution. 
	Then, the communication complexity is $O(\frac{1}{K\epsilon^2})$, the gradient complexity and Jacobian-vector product complexity is $O(\frac{1}{K\epsilon^2})$, and the Hessian-vector product complexity is $\tilde{O}(\frac{1}{K\epsilon^2})$, which indicates the linear speedup regarding $K$. 
\end{corollary}

\begin{remark}
	Due to the additional assumption~\ref{assumption_continuous}, it is much easier to bound the consensus error. For instance,  to bound the consensus error $\mathbb{E}[\|Z_{t}^{\tilde{F}} - \bar{Z}_{t}^{\tilde{F}}\|_F^2]$,  it is easy to get $\frac{1}{K}\mathbb{E}[\|Z_{t}^{\tilde{F}} - \bar{Z}_{t}^{\tilde{F}}\|_F^2] 	\leq \frac{2\alpha_1^2\eta^2\hat{C}_{\tilde{F}}^2}{(1-\lambda)^2}$ (See Lemma~\ref{lemma_consensus_z_f_bounded_gradient_norm}) since the second moment of the stochastic hypergradient is upper bounded.  In fact, with this strong assumption, we can decouple different consensus errors to simplify the theoretical analysis. On the contrary, without Assumption~\ref{assumption_continuous}, there exists inter-dependence between different  consensus errors (See Lemmas~\ref{lemma_hyper_momentum_var},~\ref{lemma_consensus_z_f},~\ref{lemma_incremental_x}), which makes  it much more challenging to establish the convergence rate.

\end{remark}

\subsection{Convergence Rate of Algorithm~\ref{alg_VRDBO}}
Since Algorithm~\ref{alg_VRDBO} employs the variance-reduced gradient estimator, we introduce the following mean-square Lipschitz smoothness assumption for the upper-level and lower-level objective functions, which is also used by existing variance-reduced bilevel optimization algorithms \citep{yang2021provably,guo2021randomized}.  Please note that all variance-reduced gradient descent algorithms \citep{cutkosky2019momentum,fang2018spider} require the mean-square Lipschitz smoothness assumption to establish the convergence rate. In the following, we use $z_i$ to denote $(x_i, y_i)$ where $i\in\{1, 2\}$.

\begin{assumption} \label{assumption_upper_smooth_vr}
	For any  $k\in \{1,2,\cdots, K\}$,   $\nabla_x f^{(k)}(x, y)$ is Lipschitz continuous with the constant  $\ell_{f_x}>0$, 	$\nabla_y f^{(k)}(x, y)$  is Lipschitz continuous with the constant $\ell_{f_y}>0$,  i.e.,
	\begin{equation}
	\begin{aligned}
		&\quad \mathbb{E}[ \|\nabla_x f^{(k)}(z_1; \xi) - \nabla_x f^{(k)}(z_2; \xi)\|] \leq \ell_{f_x}\|z_1-z_2\| \ , \\
		& \quad \mathbb{E}[\|\nabla_y f^{(k)}(z_1; \xi) - \nabla_y f^{(k)}(z_2; \xi)\|]  \leq \ell_{f_y}\|z_1-z_2\|  \ ,   \\
	\end{aligned}
\end{equation}
	hold for any $(x_1, y_1), (x_2, y_2)\in\mathbb{R}^{d_x}\times \mathbb{R}^{d_y}$. 
	Moreover,  $\mathbb{E}[\|\nabla_y f^{(k)}(x, y; \xi)\|]\leq c_{f_y}$ with the constant $c_{f_y}>0$ for $(x, y)\in\mathbb{R}^{d_x}\times \mathbb{R}^{d_y} $. 
	
\end{assumption}

\begin{assumption} \label{assumption_lower_smooth_vr}
	For any  $k\in \{1,2,\cdots, K\}$,  $\nabla_y g^{(k)}(x, y)$ is  Lipschitz continuous with the constant  $\ell_{g_y}>0$, $\nabla_{xy}^2 g^{(k)}(x, y)$ is Lipschitz continuous with the constant $\ell_{g_{xy}}>0$,  $\nabla_{yy}^2 g^{(k)}(x, y)$  is Lipschitz continuous with the constant $\ell_{g_{yy}}>0$, i.e.,
	\begin{equation}
	\begin{aligned}
		& \quad \mathbb{E}[\|\nabla_y g^{(k)}(z_1;\zeta) - \nabla_y g^{(k)}(z_2;\zeta)\|]  \leq  \ell_{g_y}\|z_1-z_2\| \ , \\
		& \quad \mathbb{E}[\|\nabla_{xy}^2 g^{(k)}(z_1;\zeta)-\nabla_{xy}^2 g^{(k)}(z_2;\zeta)\| ]\leq \ell_{g_{xy}}\|z_1-z_2\| \ ,  \\
		&\quad \mathbb{E}[ \|\nabla_{yy}^2 g^{(k)}(z_1;\zeta)-\nabla_{yy}^2 g^{(k)}(z_2;\zeta)\|] \leq \ell_{g_{yy}}\|z_1-z_2\| \ , \\
	\end{aligned}
\end{equation}
	hold for any $ (x_1, y_1), (x_2, y_2)\in\mathbb{R}^{d_x}\times \mathbb{R}^{d_y}$. 
	Moreover,  $\mathbb{E}[\|\nabla_{xy}^2 g^{(k)}(x, y; \zeta)\|]\leq c_{g_{xy}}$ with the constant $c_{g_{xy}}>0$ and $\mu \mathbf{1} \preceq \nabla_{yy}^2 g^{(k)}(x, y; \zeta)\preceq  \ell_{g_y} \mathbf{1} $ for  $(x, y)\in\mathbb{R}^{d_x}\times \mathbb{R}^{d_y} $. 
\end{assumption}

\begin{theorem} \label{theorem_vrdbo}
	Given Assumptions~\ref{assumption_graph}-\ref{assumption_variance},~\ref{assumption_upper_smooth_vr},~\ref{assumption_lower_smooth_vr}, if $\alpha_1>0$, $\alpha_2>0$, $\eta<\min \{1, \frac{1}{2\beta_1L_{F}^{*}}, \frac{1}{\sqrt{\alpha_1}}, \frac{1}{\sqrt{\alpha_2}}\}$, 
	$ \beta_1\leq \min\{\beta_{1,a},  \beta_{1,b},   \beta_{1,c}\}$ and $\beta_2 \leq \min\{\beta_{2,a},  \beta_{2,b}, \beta_{2, c} \}$, 
	\begin{equation} \label{beta_vr}
		\begin{aligned}
			& \beta_{1,a} = \frac{\beta_2\mu}{15L_{y} {L}_{F}} \ ,  \beta_{2,c} =\frac{1}{6\ell_{g_y}}  \ , \\
			& \beta_{1,b} = \frac{\mu}{8{\ell}_{g_{y}}\sqrt{(3+3/(\alpha_1K))L_{\tilde{F}}^2+(3+50/(\alpha_2K)){L}_{F}^2}} \  , \\
			& \beta_{1,c} = \frac{\mu(1-\lambda)^2/{\ell}_{g_{y}}  } {2\sqrt{(57+54/(\alpha_1K))L_{\tilde{F}}^2+(104+900/(\alpha_2K)){L}_{F}^2} }  \ ,  \\
			& \beta_{2,a} =\frac{(1-\lambda)^2 L_F/{\ell}_{g_{y}}}{2\sqrt{(57+54/(\alpha_1K))L_{\tilde{F}}^2+(104+900/(\alpha_2K)){L}_{F}^2}} \ ,  \\
			&  \beta_{2,b} = \frac{9\mu{L}_{F}^2}{8 {\ell}_{g_{y}}^2((6+6/(\alpha_1K))L_{\tilde{F}}^2+(6+100/(\alpha_2K)){L}_{F}^2)}  \ , \\
		\end{aligned}
	\end{equation}
	the convergence rate of Algorithm~\ref{alg_VRDBO} is 
	\begin{equation}
		\begin{aligned}
			&  \quad \frac{1}{T}\sum_{t=0}^{T-1}(\mathbb{E}[\|\nabla F(\bar{  {x}}_{t})\|^2  + L_F^2\|\bar{y}_{t} - y^*(\bar{x}_{t})\|^2] )\\
			& \leq \frac{2(F(x_0) - F(x_*) )}{\beta_1\eta T} +  \frac{12{L}_{F}^2}{\mu\beta_2\eta T}\|\bar{   {y}}_{0} -    {y}^{*}(\bar{   {x}}_{0})\| ^2  \\
			& \quad +  \frac{6C_{g_{xy}}^2C_{f_y}^2}{\mu^2}(1-\frac{\mu}{\ell_{g_{y}}})^{2J} +\frac{8 \sigma_{\tilde{F}}^2 }{\eta TB} + \frac{8L_F^2\sigma^2}{\eta TB\mu^2} \\
			& \quad  + \frac{6\sigma_{\tilde{F}}^2}{\alpha_1\eta^2 TBK}  + \frac{100{L}_{F}^2\sigma^2}{\alpha_2\eta^2 TBK\mu^2} + 2\alpha_1^2\eta^3\sigma_{\tilde{F}}^2 \\
			& \quad +   \frac{2\alpha_2^2\eta^3\sigma^2 L_F^2}{\mu^2}  +8\alpha_1^2 \eta^2 \sigma_{\tilde{F}}^2  +  \frac{8\alpha_2^2\eta^2 \sigma^2L_F^2}{\mu^2}  \\
			& \quad  + \frac{12\alpha_1 \eta^2 \sigma_{\tilde{F}}^2 }{K} + \frac{200\alpha_2\eta^2 \sigma^2{L}_{F}^2}{\mu^2K} \ ,  \\
		\end{aligned}
	\end{equation}
	where the definition of $L_y$, $L_F^*$, $L_F$, $L_{\tilde{F}}$, $\sigma_{\tilde{F}}$ is shown in Lemmas~\ref{lemma_hypergrad_smooth_optimal_var},~\ref{lemma_hypergrad_var_var},~\ref{lemma_hypergrad_smooth_var}.
\end{theorem}


\begin{corollary} \label{corollary_vrdbo}
	Given the same condition with Theorem~\ref{theorem_vrdbo}, by choosing $T=O(\frac{1}{K\epsilon^{3/2}(1-\lambda)^2})$, $\eta=O(K\epsilon^{1/2})$, $J=O(\log \frac{1}{\epsilon})$, $B=O(\frac{1}{\epsilon^{1/2}})$, $\beta_1=O((1-\lambda)^2)$,  $\beta_2=O((1-\lambda)^2)$, $\alpha_1=O(1/K)$, and $\alpha_2=O(1/K)$, Algorithm~\ref{alg_VRDBO} can achieve the $\epsilon$-accuracy solution.
	Then, the communication complexity  is $O(\frac{1}{K\epsilon^{3/2}(1-\lambda)^2})$, which is better than $O(\frac{1}{\epsilon^{2}(1-\lambda)^2})$ of Algorithm~\ref{alg_MDBO}.   Additionally,  the gradient complexity and Jacobian-vector product complexity of Algorithm~\ref{alg_VRDBO}  are $O(\frac{1}{K\epsilon^{3/2}(1-\lambda)^2})$ and the Hessian-vector product complexity is $\tilde{O}(\frac{1}{K\epsilon^{3/2}(1-\lambda)^2})$,  indicating the linear speedup with respect to the number of participants $K$. 
\end{corollary}

 \begin{figure*}[ht]
	\centering 
	\hspace{-15pt}
	\subfigure[a9a]{
		\includegraphics[scale=0.382]{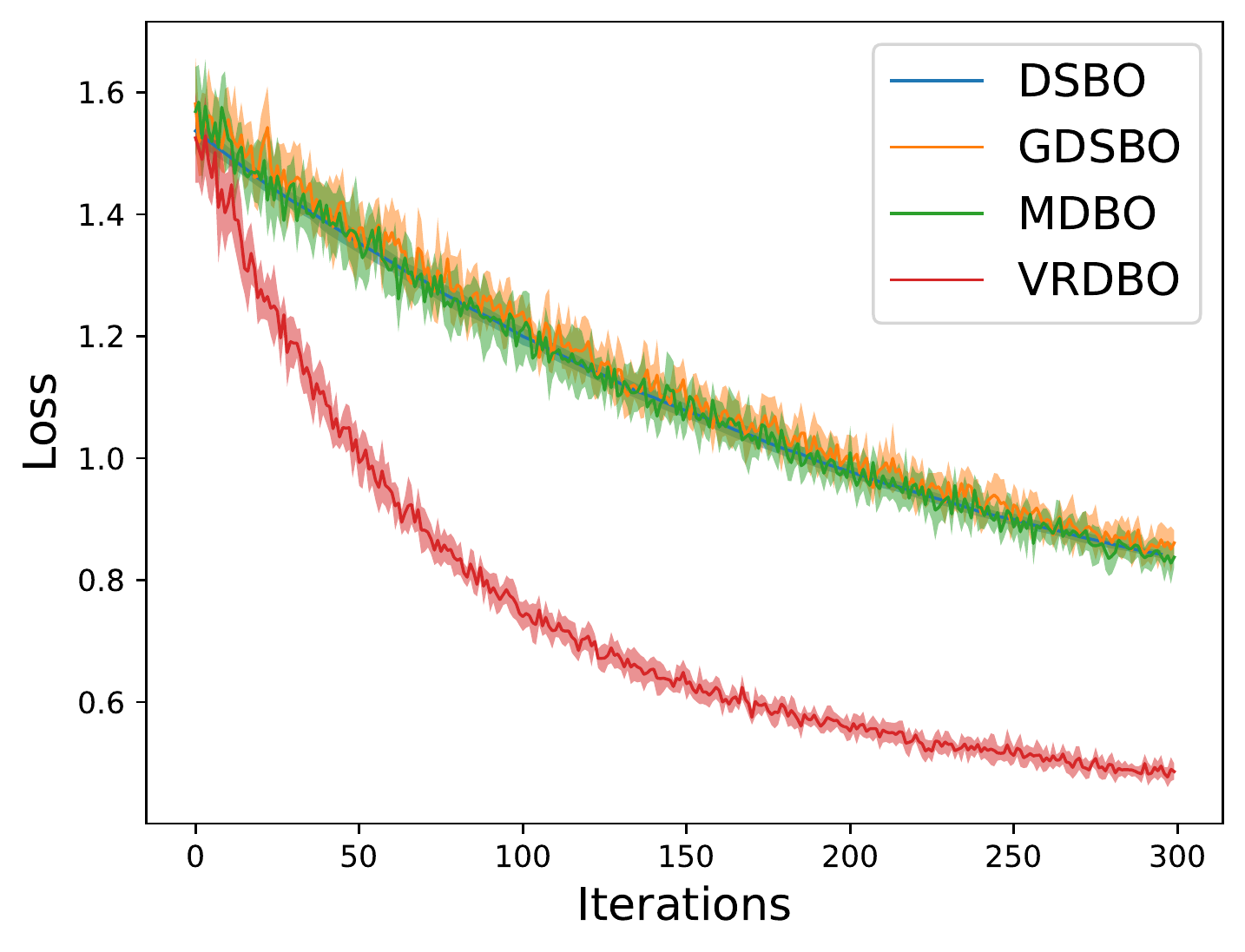}
	}
	\hspace{-12pt}
	\subfigure[ijcnn1]{
		\includegraphics[scale=0.3852]{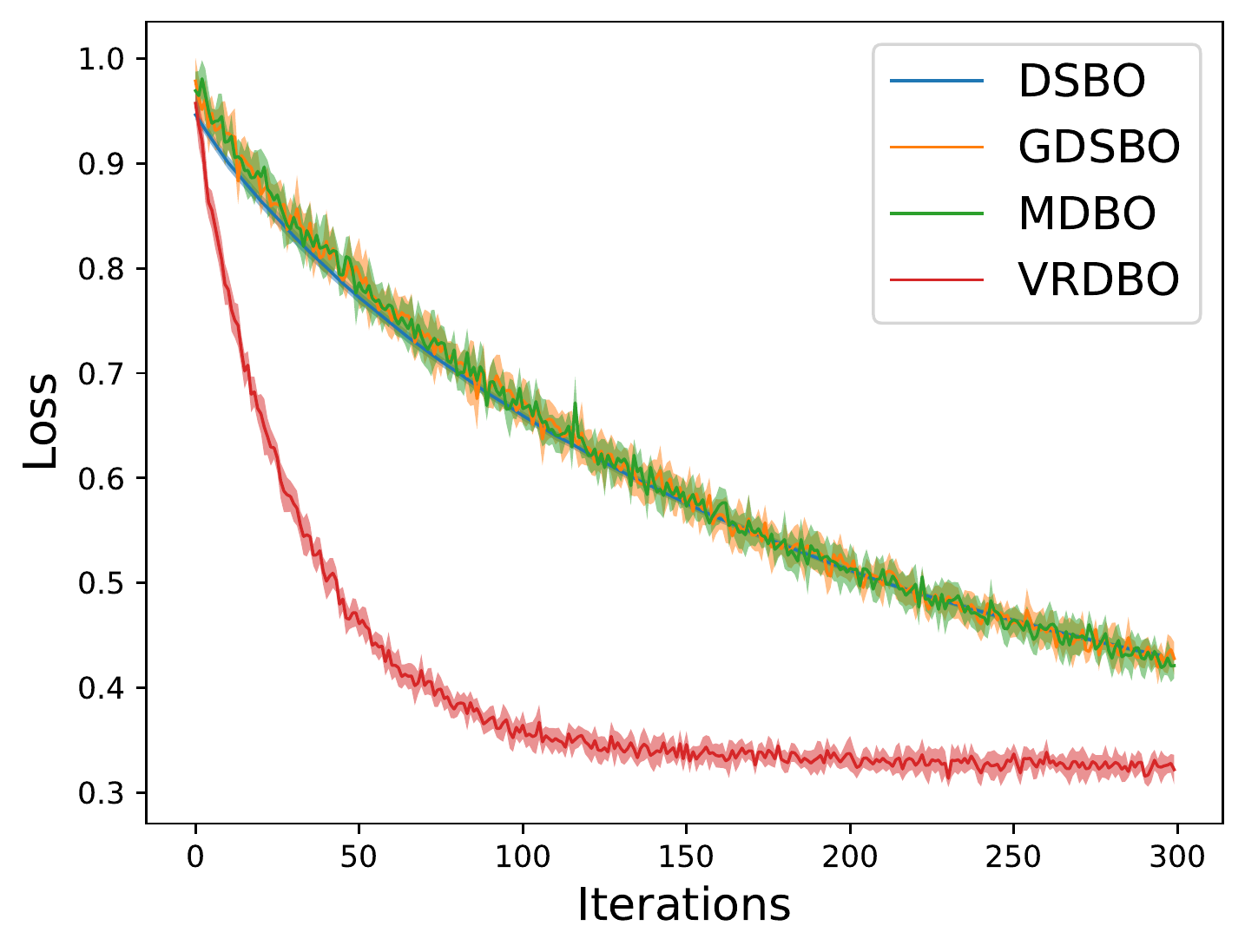}
	}
	\hspace{-12pt}
	\subfigure[covtype]{
		\includegraphics[scale=0.3852]{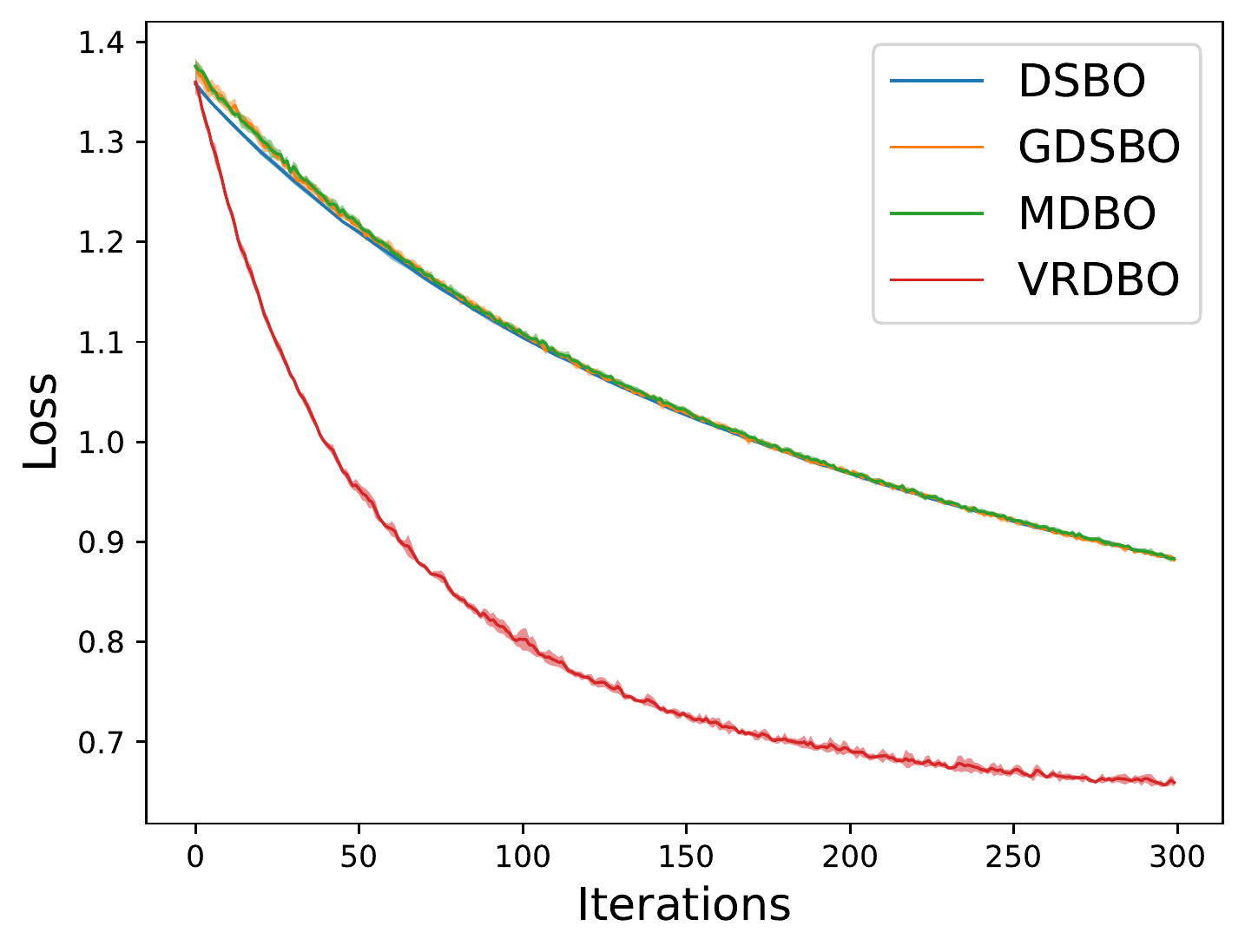}
	}
	\caption{The upper-level training loss function value versus the update of variables. }
	\label{loss_vs_iteration}
\end{figure*}

\begin{figure*}[h]
	\centering 
	\hspace{-15pt}
	\subfigure[a9a]{
		\includegraphics[scale=0.382]{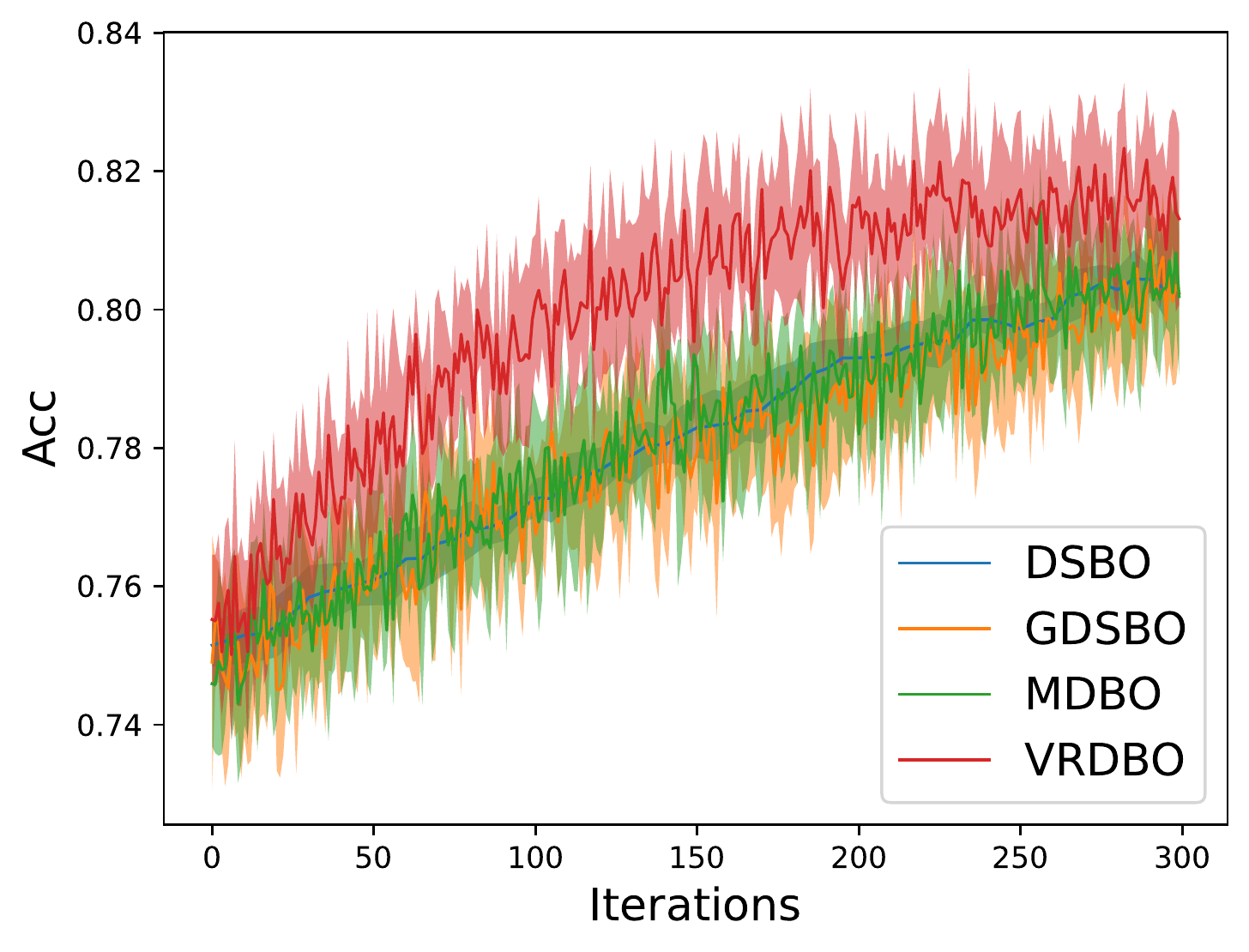}
	}
	\hspace{-12pt}
	\subfigure[ijcnn1]{
		\includegraphics[scale=0.3852]{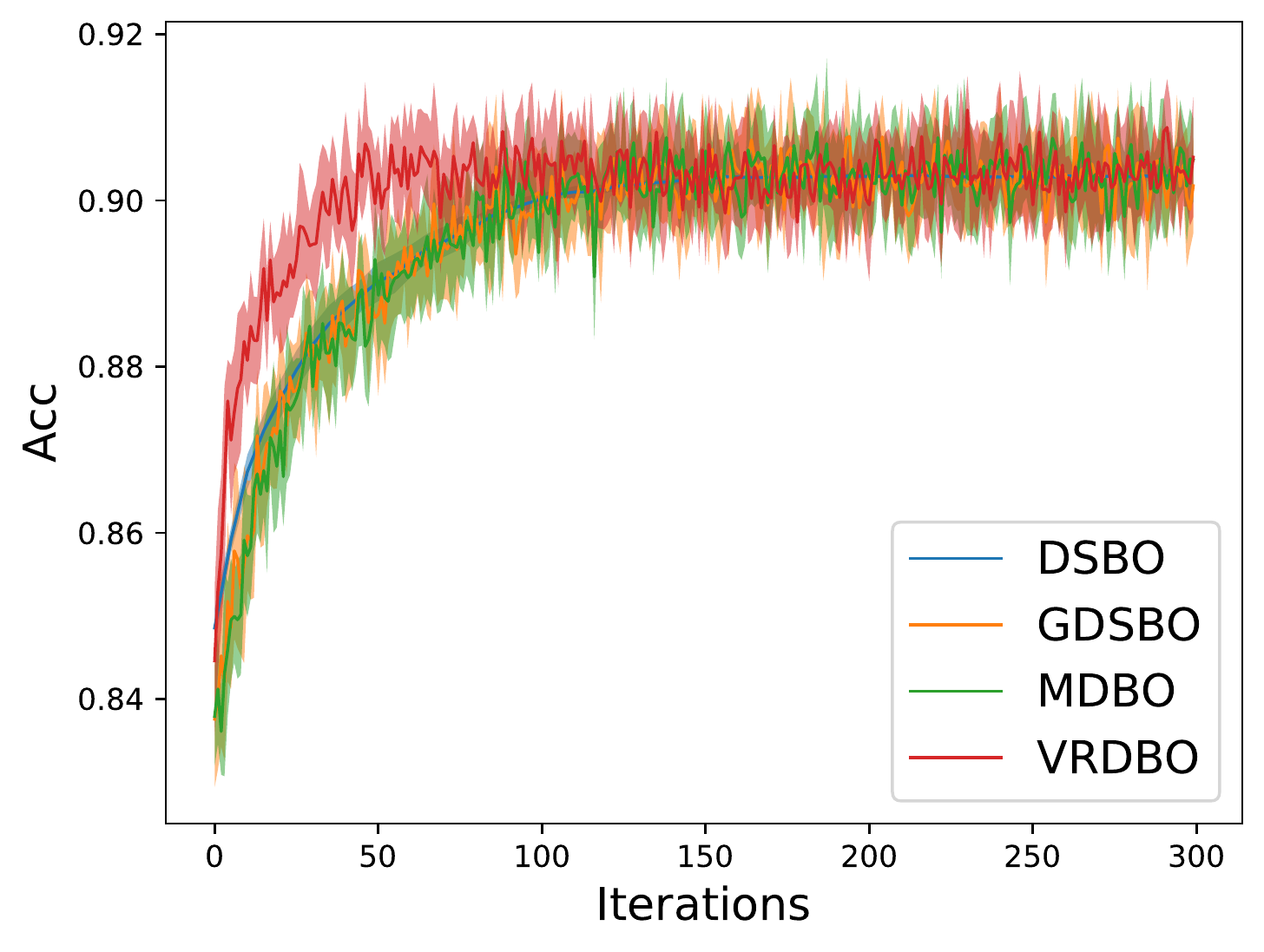}
	}
	\hspace{-12pt}
	\subfigure[covtype]{
		\includegraphics[scale=0.3852]{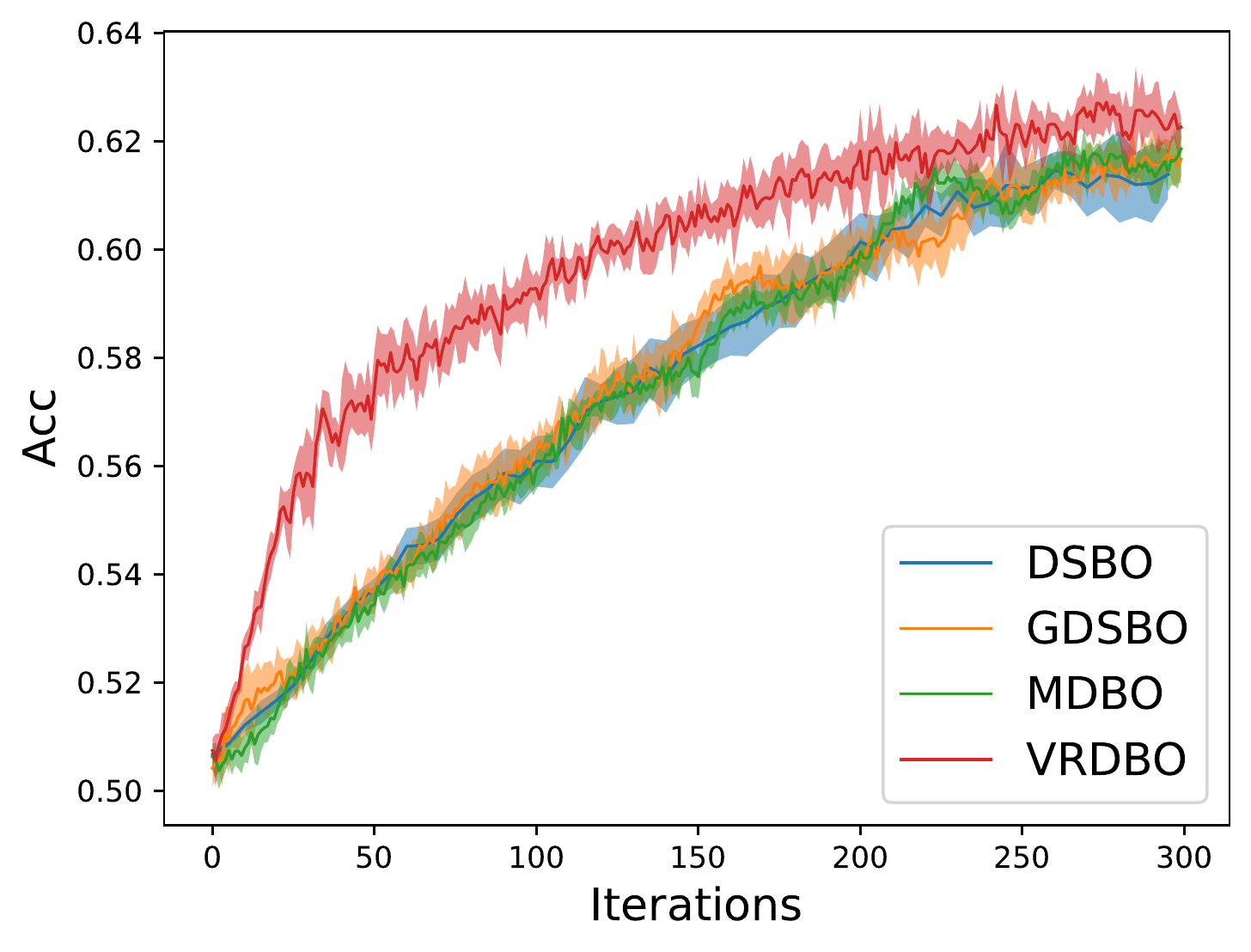}
	}
	\caption{The prediction accuracy of validation set versus the update of variables. }
	\label{acc_vs_iteration}
\end{figure*}

\begin{remark}
It is worth noting that we are able to set $\alpha_i=O(1/K)$ in Theorem~\ref{theorem_vrdbo}.  As a result, $\beta_i$ and $\alpha_i$ are decoupled (See Eq.~(\ref{beta_vr})).  
In this way, Algorithm~\ref{alg_VRDBO} can achieve linear speedup with respect to the number of participants.
In comparison, $\alpha_i$ has to be $O(1)$ in Theorem~\ref{theorem_mdbo}.  
\end{remark}

\section{Experiments}
In this section, we conduct experiments to verify the performance of our proposed algorithms.  In particular, we apply our algorithms to the hyperparameter optimization of the logistic regression model \citep{grazzi2020iteration}, which is defined as follows:
\begin{equation}
	\begin{aligned}
		& \min_{x\in \mathbb{R}^d} \frac{1}{K} \sum_{k=1}^{K}\frac{1}{n_{val}^{(k)}}\sum_{i=1}^{n_{val}^{(k)}}\ell_{CE}(y^*(x)^Ta_{val, i}^{(k)}, b_{val,i}^{(k)}) \\
		& s.t. \  y^*(x) =  \arg\min_{y\in \mathbb{R}^{d\times c}}\frac{1}{K} \sum_{k=1}^{K} \frac{1}{n_{tr}^{(k)}}\sum_{i=1}^{n_{tr}^{(k)}}  \ell_{CE}(y^Ta_{tr, i}^{(k)}, b_{tr, i}^{(k)}) \\
		& \quad  \quad \quad \quad  \quad \quad   \quad \quad  \quad  + \frac{1}{cd}\sum_{p=1}^{c}\sum_{q=1}^{d} \exp(x_q)y_{pq}^2 \ ,
	\end{aligned}
\end{equation}
where $(a_{val,i}^{(k)}, b_{val,i}^{(k)})\in \mathbb{R}^{d}\times\mathbb{R}^c$ denotes the $i$-th validation sample's feature and label of the $k$-th participant, $(a_{tr,i}^{(k)}, b_{tr,i}^{(k)})$ represents the training sample, $n_{val}^{(k)}$ is the number of validation samples in the $k$-th participant, $n_{tr}^{(k)}$ is the number of training samples, $\ell_{CE}$ is the cross-entropy loss function, $x\in \mathbb{R}^d$ represents the hyperparameter, $y\in \mathbb{R}^{d\times c}$ denotes the model parameter. 

\begin{figure*}[ht]
	\centering 
	\hspace{-15pt}
	\subfigure[a9a: MDBO]{
		\includegraphics[scale=0.382]{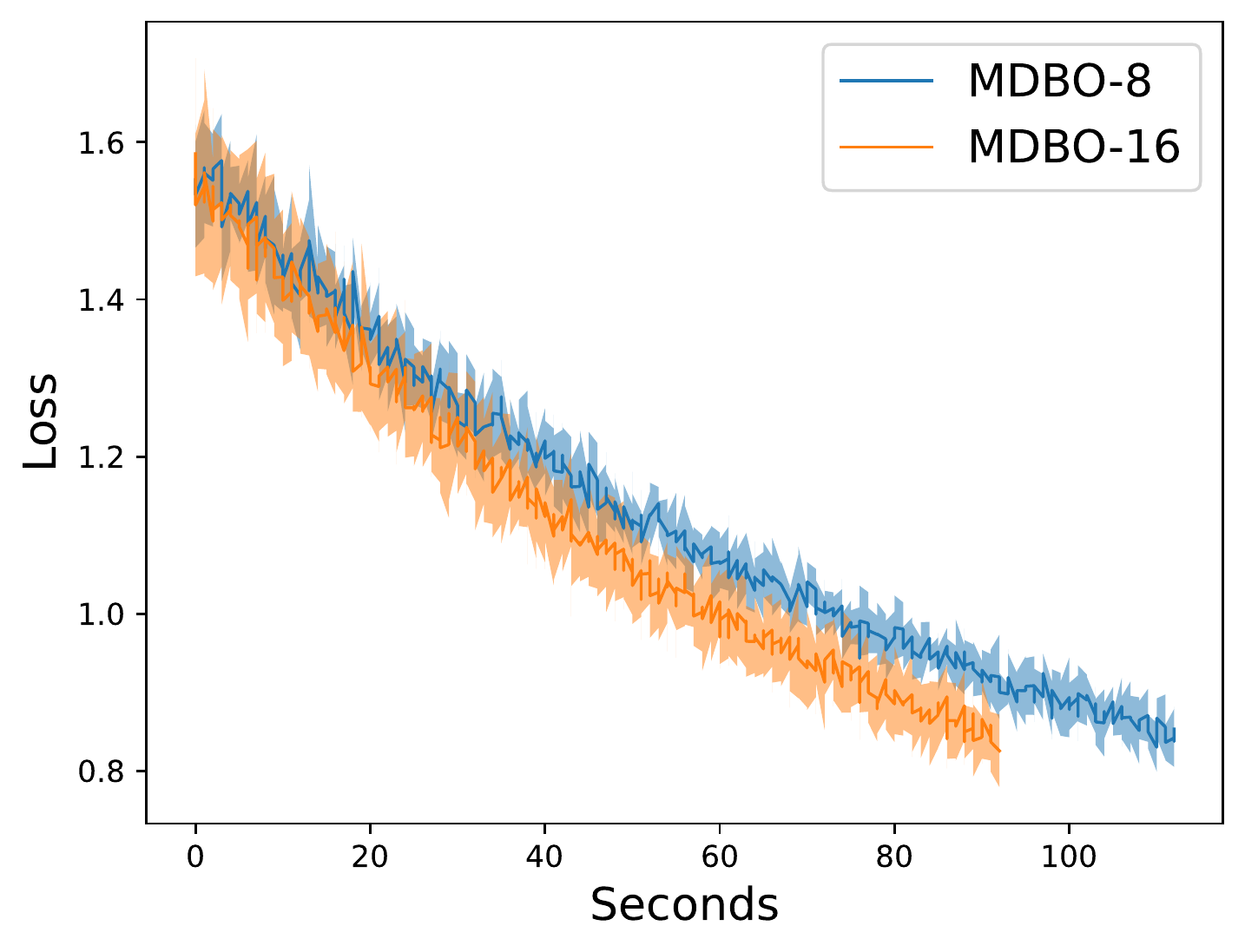}
	}
	\hspace{-12pt}
	\subfigure[ijcnn1: MDBO]{
		\includegraphics[scale=0.3852]{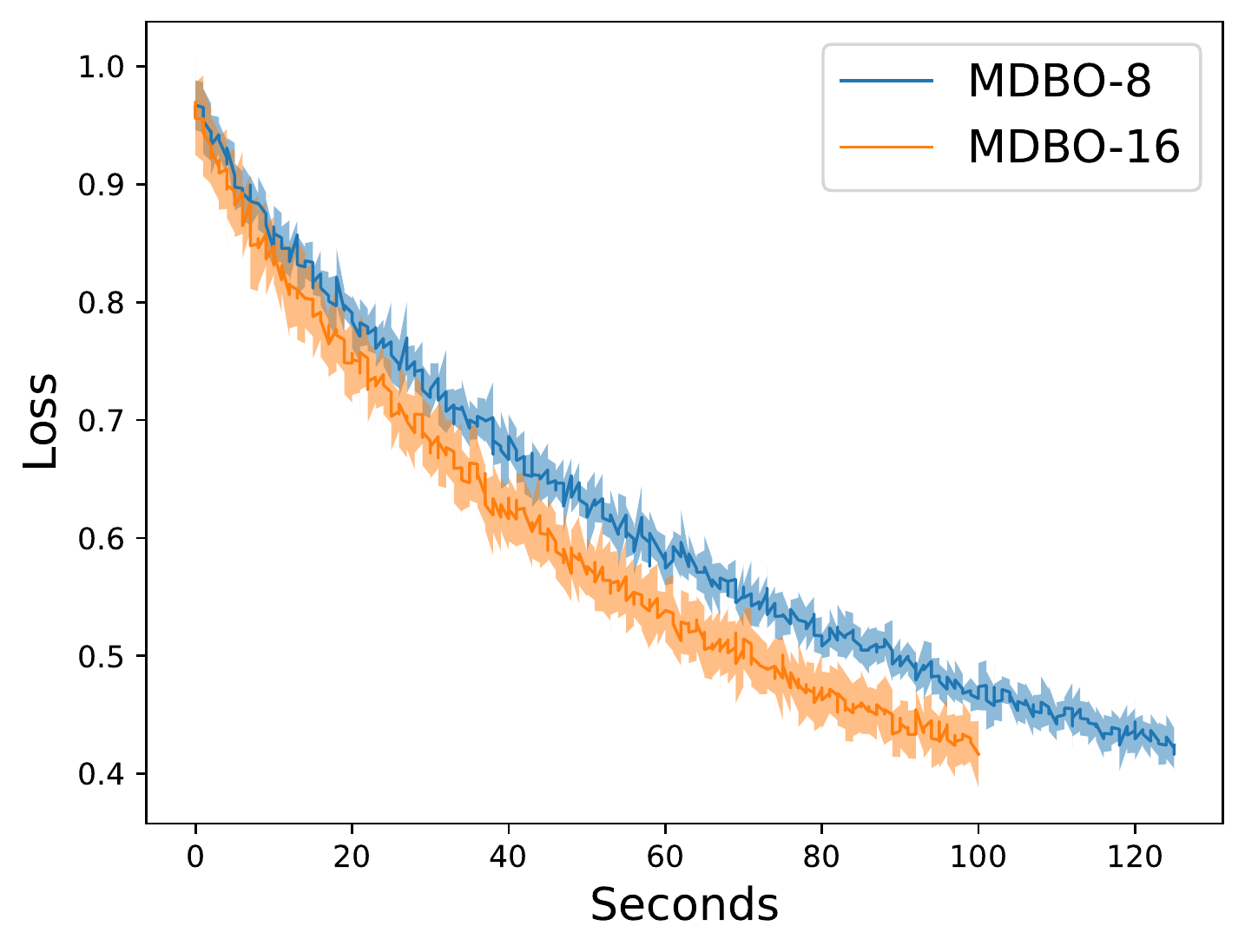}
	}
	\hspace{-12pt}
	\subfigure[covtype: MDBO]{
		\includegraphics[scale=0.3852]{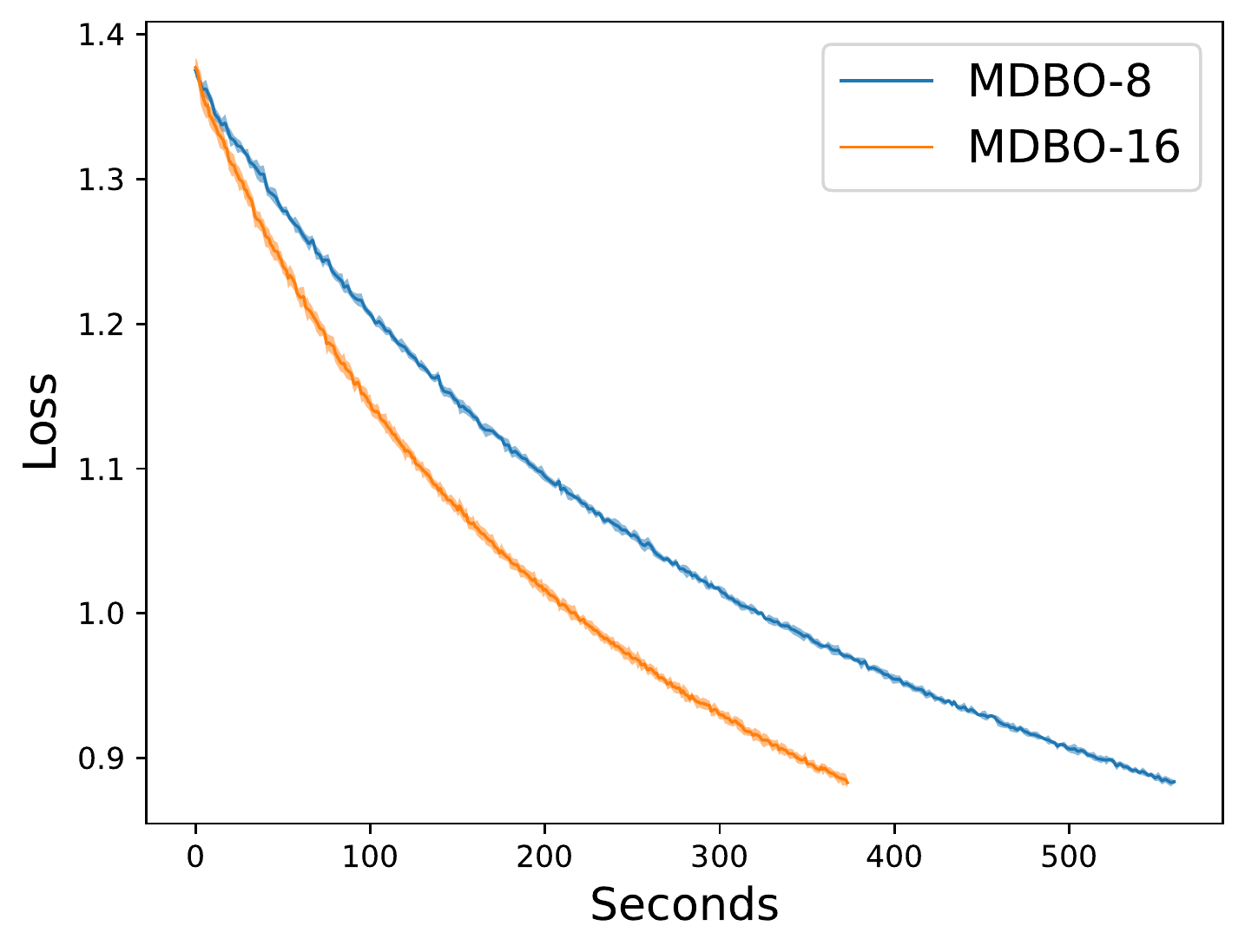}
	}
	\hspace{-15pt}
	\subfigure[a9a: VRDBO]{
		\includegraphics[scale=0.382]{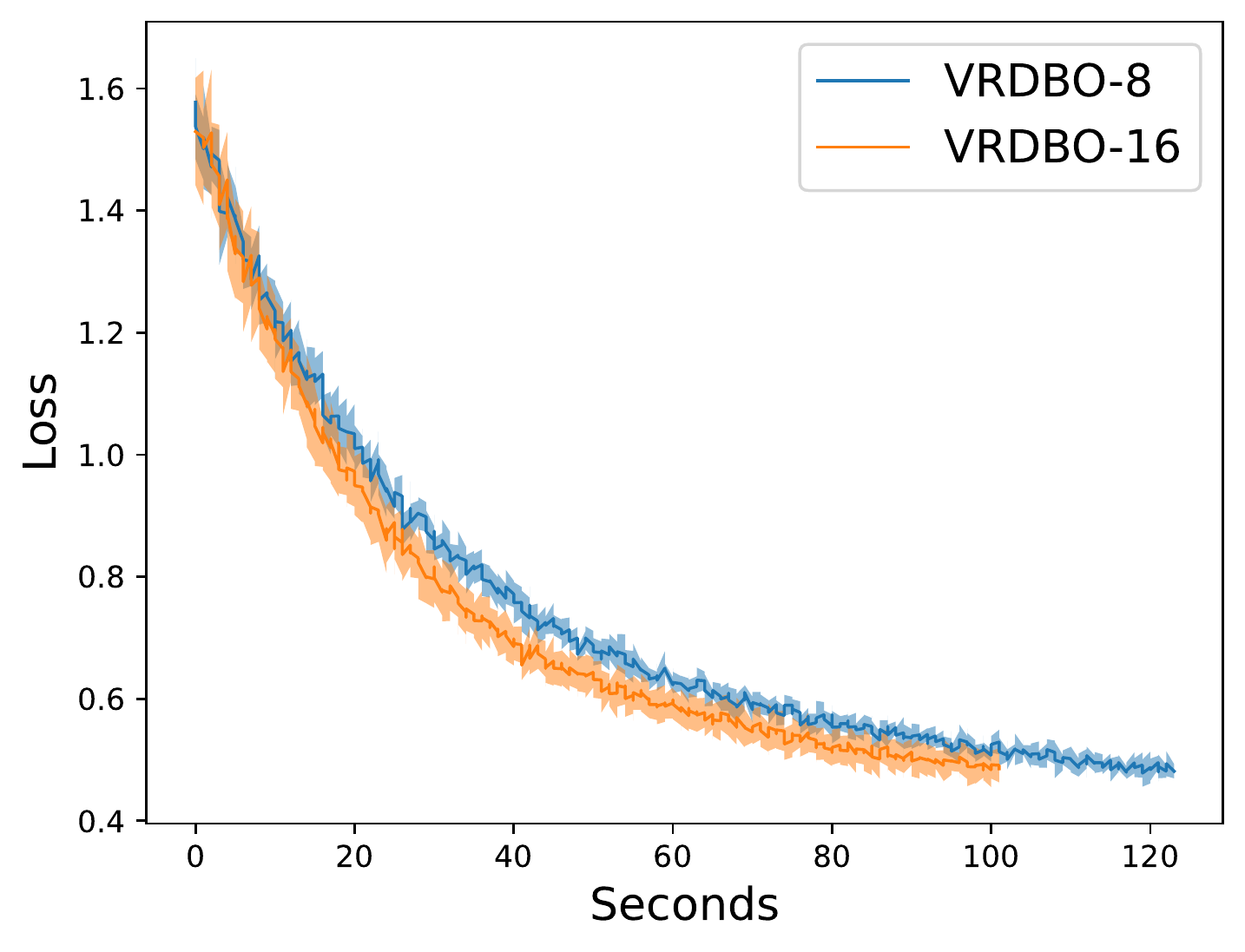}
	}
	\hspace{-12pt}
	\subfigure[ijcnn1: VRDBO]{
		\includegraphics[scale=0.3852]{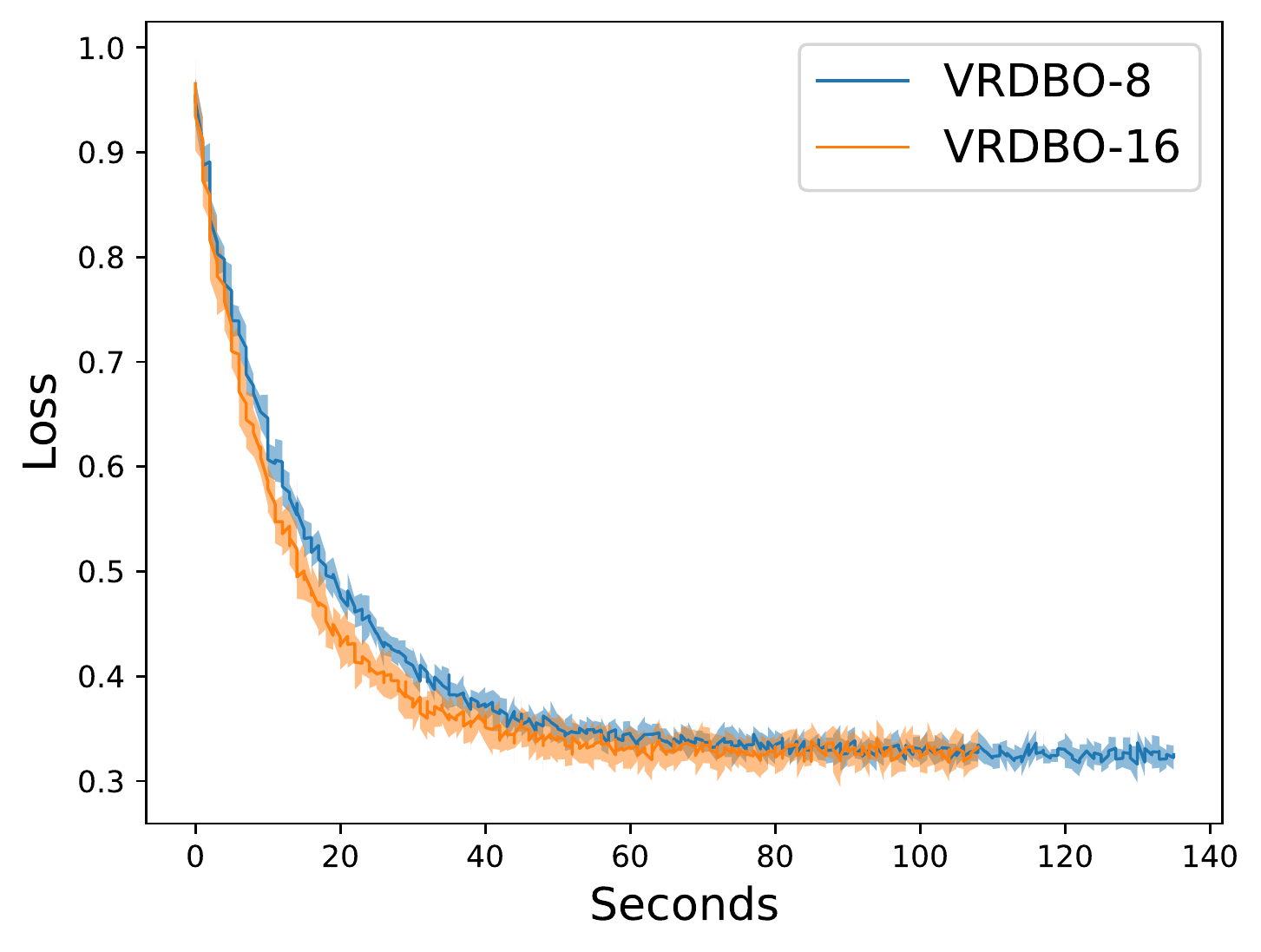}
	}
	\hspace{-12pt}
	\subfigure[covtype: VRDBO]{
		\includegraphics[scale=0.3852]{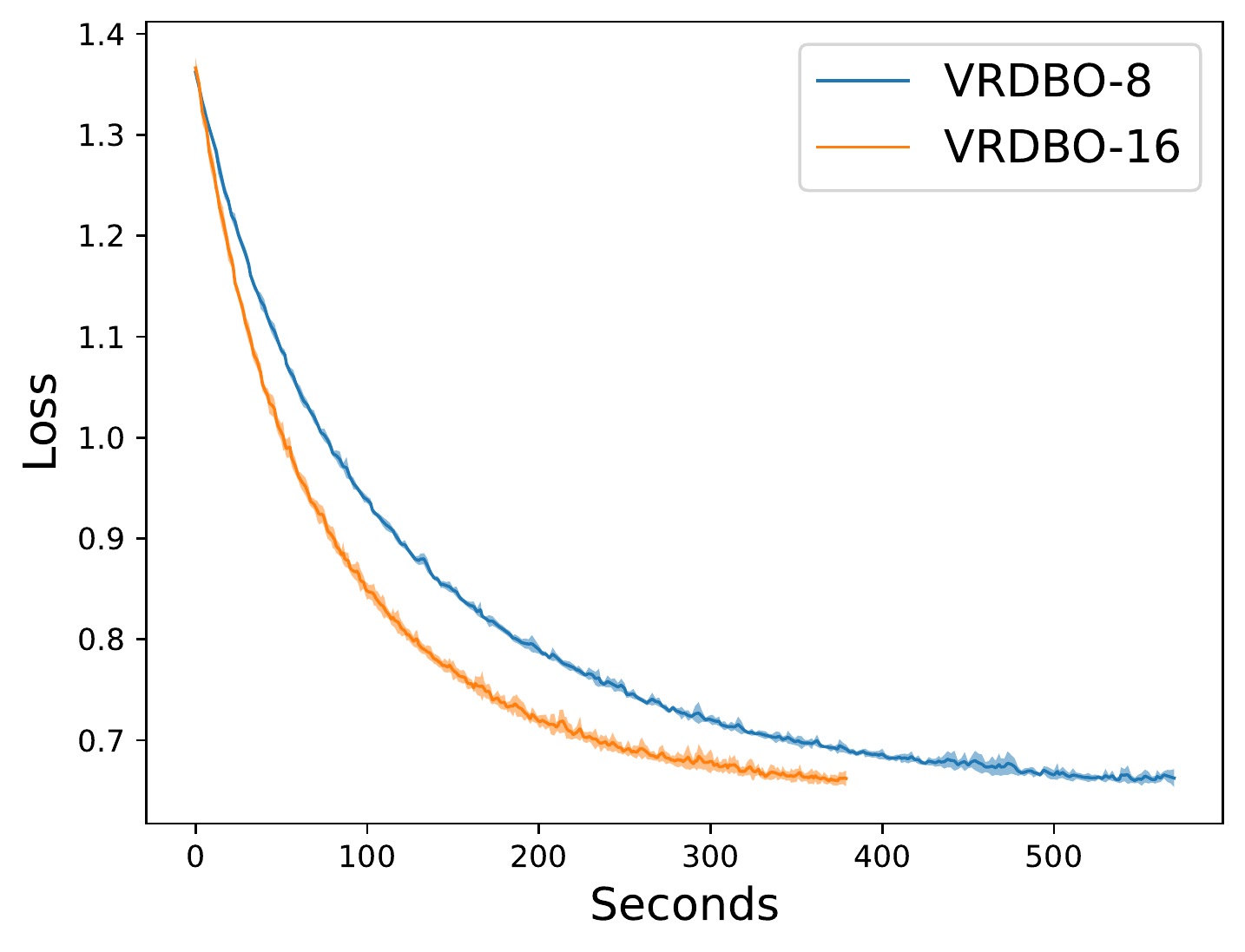}
	}
	\caption{The upper-level loss function value  versus the consumed time. }
	\label{linear}
\end{figure*}

In our experiments, we use three binary  classification  datasets \footnote{\url{https://www.csie.ntu.edu.tw/~cjlin/libsvmtools/datasets/}}: a9a, ijcnn1, and covtype. 
In particular, a9a has 32,561 sample, ijcnn1 has 49,990 samples, and covtype has 581,012 samples. 
We randomly select $30\%$ samples as the validation set and the remaining samples as the training set. Then, they are randomly and evenly put to each participant so that the data distribution is i.i.d. in all participants. 
To demonstrate the performance of our algorithms, we compare them with two baseline algorithms: DSBO \citep{chen2022decentralized} and GDSBO \citep{yang2022decentralized}. Specifically, DSBO employs stochastic (hyper) gradient and gossip communication strategy, and GDSBO takes the momentum technique and gossip communication strategy.   Note that They require to explicitly  communicate Hessian or Jacobian matrices, which is prohibitive for practical applications. Thus, we implement a simplified version, where Hessian and Jacobian matrices are implicitly computed as \citep{yang2021provably} and only model parameters (and gradient estimators) are communicated via the gossip communication strategy. 
In our experiments, we use the same batch size for all algorithms. In particular, the batch size on each participant is $400/K$ where $K$ is the total number of participants. When estimating the stochastic hypergradient, $J$ is set to $10$ for all algorithms.  Moreover, since the learning rate of DSBO, GDSBO, MDBO is in the order of $O(\epsilon)$ and that of VRDBO is $O(\epsilon^{1/2})$, we set the learning rate $\eta=0.1$ for DSBO, GDSBO, MDBO and $\eta=0.33$ for VRDBO. Additionally,   we set $\beta_1=\beta_2=1.0$ for both MDBO and VRDBO, $\alpha_1=\alpha_2=1.0$ for MDBO,  and $\alpha_1=\alpha_2=5.0$ for VRDBO.

In Figure~\ref{loss_vs_iteration},  we employ 8 workers and the network is a ring network. Here, we show the upper-level loss function value  with respect to the update of  variables. There are two observations. First, our MDBO has similar convergence performance as GDSBO based on both loss function value and prediction accuracy, which uses the same momentum technique as ours.  However, they don't show significant improvement over DSBO. This is consistent with the theoretical convergence rate.   Second, VRDBO converges much faster than MDBO and baseline methods, since it employs a variance-reduced gradient estimator. This observation  confirms the correctness of our theoretical convergence rate and the effectiveness of our algorithm.

In Figure~\ref{acc_vs_iteration}, we plot the prediction accuracy of validation set with respect to the  update of variables.  The experimental settings are the same with those of Figure~\ref{loss_vs_iteration}. From Figure~\ref{acc_vs_iteration}, we have two observations.  First, our  algorithms, MDBO and VRDBO, can achieve almost the same prediction accuracy with baseline algorithms, which confirms the correctness of our algorithms. Second,  our VRDBO converges faster than MDBO in terms of the prediction accuracy, which is consistent with our theoretical convergence rates.  In summary, Figure~\ref{loss_vs_iteration} and Figure~\ref{acc_vs_iteration} confirm the correctness and effectiveness of our two algorithms.

In Figure~\ref{linear}, we plot the upper-level loss function value with respect to the consumed time (seconds) for our MDBO and VRDBO algorithms. Here, to demonstrate the speedup effect, we use 8, 16 workers, respectively.  The batch size of each worker is set to $400/K$ where $K$ is the number of workers. Other experimental     settings are the same with those of Figure~\ref{loss_vs_iteration}. From Figure~\ref{linear}, we can find that using more workers is able to accelerate the practical convergence speed of our two algorithms.

\section{Conclusions}
In this paper, we studied how to facilitate bilevel optimization to the decentralized setting. In particular, we developed two decentralized bilevel optimization algorithms, which demonstrate how to update variables on each participant and communicate variables across participants. 
In addition, we  established the convergence rate, demonstrating how the  network topology, number of participants, and other hyperparameters affect the convergence rate. 
To our knowledge, this is the first work achieving these favorable  results. Moreover, extensive experimental results confirm the correctness and effectiveness of our algorithms. 

\subsubsection*{Acknowledgements}
This material is partially based upon work supported by the National Science Foundation under grants CNS-1814614, CNS-2140477, and IIS-1908594.

\bibliographystyle{abbrvnat}
\bibliography{bib}

\onecolumn
 \appendix
\input{supplement}

\end{document}

%% file: supplement.tex
\aistatstitle{Supplementary Materials}

\section{ Proof}
To investigate the convergence rate of our algorithms, other than the notations in Section 3,  we introduce the following additional notations:
\begin{equation}
	\begin{aligned}
		&  \bar{u}_t =  \frac{1}{K}\sum_{k=1}^{K} u_{t}^{(k)},  \bar{v}_t =  \frac{1}{K}\sum_{k=1}^{K} v_{t}^{(k)},  \bar{z}_t =  \frac{1}{K}\sum_{k=1}^{K} z_{t}^{(k)},   \\
		& \bar{X}_t = \frac{1}{K}X_t \mathbf{1}\mathbf{1}^T  \ , \bar{Y}_t = \frac{1}{K}Y_t \mathbf{1}\mathbf{1}^T  \ , \bar{U}_t = \frac{1}{K}U_t \mathbf{1}\mathbf{1}^T  \ , 
		\bar{V}_t = \frac{1}{K}V_t \mathbf{1}\mathbf{1}^T  \ , \bar{Z}_t^{\tilde{F}} = \frac{1}{K}Z_t^{\tilde{F}} \mathbf{1}\mathbf{1}^T  \ , \bar{Z}_t^{g} = \frac{1}{K}Z_t^{g} \mathbf{1}\mathbf{1}^T  \ ,\\
		& \Delta_t^{\tilde{F}} = [{\nabla} \tilde{F}^{(1)}({x}_{t}^{(1)}, {y}_{t}^{(1)}), {\nabla} \tilde{F}^{(2)}({x}_{t}^{(2)}, {y}_{t}^{(2)}), \cdots, {\nabla} \tilde{F}^{(K)}({x}_{t}^{(K)}, {y}_{t}^{(K)})]  \ , \\
		&  \Delta_t^{g} = [{\nabla_y} g^{(1)}({x}_{t}^{(1)}, {y}_{t}^{(1)}), {\nabla_y} g^{(2)}({x}_{t}^{(2)}, {y}_{t}^{(2)}), \cdots, {\nabla_y} g^{(K)}({x}_{t}^{(K)}, {y}_{t}^{(K)})]  \ , \\
		& \underline{\Delta}_t^{\tilde{F}} = [{\nabla} \tilde{F}^{(1)}(\bar{x}_{t}, \bar{y}_{t}), {\nabla} \tilde{F}^{(2)}(\bar{x}_{t}, \bar{y}_{t}), \cdots, {\nabla} \tilde{F}^{(K)}(\bar{x}_{t}, \bar{y}_{t})]  \ ,\\
		&  \underline{\Delta}_t^{g} = [{\nabla_y} g^{(1)}(\bar{x}_{t}, \bar{y}_{t}), {\nabla_y} g^{(2)}(\bar{x}_{t}, \bar{y}_{t}), \cdots, {\nabla_y} g^{(K)}(\bar{x}_{t}, \bar{y}_{t})]  \ . \\
	\end{aligned}
\end{equation}

\subsection{Proof Sketch}

\paragraph{Proof Sketch of Theorem~\ref{theorem_mdbo}.}
Since there exists inter-dependence between different consensus errors,  we proposed a novel potential function to  establish the convergence rate of Algorithm~\ref{alg_MDBO}, which is shown as follows:
\begin{equation} \label{potential_1}
	\begin{aligned}
		& \mathcal{L}_{t+1} ={ \mathbb{E}}[F(x_{t+1})] +\frac{6\beta_1{L}_{F}^2}{\beta_2\mu}\mathbb{E}[\|\bar{   {y}}_{t+1} -    {y}^{*}(\bar{   {x}}_{t+1})\| ^2 ]  + w_2 \frac{1}{K} \mathbb{E}[\|X_{t+1} - \bar{X}_{t+1}\|_F^2 ] + w_3 \frac{1}{K}\mathbb{E}[\|Y_{t+1} - \bar{Y}_{t+1}\|_F^2 ]  \\
		& \quad +  \frac{\beta_1(1-\lambda)}{2\alpha_1} \frac{1}{K}\mathbb{E}[\| Z^{\tilde{F}}_{t+1}-\bar{Z}^{\tilde{F}}_{t+1} \|_F^2] + \frac{25 (1-\lambda)\beta_1{L}_{F}^2}{\alpha_2\mu^2} \frac{1}{K}\mathbb{E}[\| Z^{g}_{t+1}-\bar{Z}^{g}_{t+1} \|_F^2] \\
		& \quad  +\frac{4\beta_1}{\alpha_1}\frac{1}{K}\mathbb{E}[\|\Delta_{t+1}^{\tilde{F}}-U_{t+1}\|_F^2] + \frac{100 \beta_1{L}_{F}^2}{\alpha_2\mu^2}\frac{1}{K}\mathbb{E}[\|\Delta_{t+1}^{g}-V_{t+1}\|_F^2] \ , \\
	\end{aligned}
\end{equation}
where $w_2=w_3= \frac{ 2\beta_1L_{g_y}^2((11+32/\alpha_1^2){L}_{\tilde{F}}^2+(450+800/\alpha_2^2){L}_{F}^2)}{\mu^2(1-\lambda^2)}$. With this potential function, we established the upper bound for each item. Then, with appropriate hyperparameters as shown in Theorem~\ref{theorem_mdbo}, we are able to get
\begin{equation}
	\begin{aligned}
		&    \mathcal{L}_{t+1} - \mathcal{L}_{t}  \leq  - \frac{\eta\beta_1}{2} \mathbb{E}[\|\nabla F(\bar{  {x}}_{t})\|^2]  - \frac{\eta\beta_1L_F^2}{2} \mathbb{E}[\|\bar{y}_{t} - y^*(\bar{x}_{t})\|^2]   \\
		& \quad +\frac{9\alpha_1\beta_1\eta^2\sigma_{\tilde{F}}^2}{2} +\frac{125\beta_1\alpha_2\eta^2 {L}_{F}^2\sigma^2}{\mu^2} +  \frac{3\eta\beta_1C_{g_{xy}}^2C_{f_y}^2}{\mu^2}(1-\frac{\mu}{L_{g_{y}}})^{2J}  \  . \\
	\end{aligned}
\end{equation}
With the help of this inequality, we can establish the convergence rate of Algorithm~\ref{alg_MDBO}. 

\paragraph{Proof Sketch of Theorem~\ref{theorem_mdbo_bounded_gradient_norm}.}
With the additional Assumption~\ref{assumption_continuous},  the consensus errors can be decoupled from each other. Then, we developed the following potential function for establishing the convergence rate of Algorithm~\ref{alg_MDBO}.
\begin{equation}
	\begin{aligned}
		& \mathcal{L}_{t+1} = F(x_{t+1}) +\frac{6\beta_1L_F^2}{\beta_2\mu}\mathbb{E}[\|\bar{   {y}}_{t+1} -    {y}^{*}(\bar{   {x}}_{t+1})\| ^2 ]   +\frac{3\beta_1}{\alpha_1}\frac{1}{K}\mathbb{E}[\|\Delta_{t+1}^{\tilde{F}}-U_{t+1}\|_F^2]  + \frac{50\beta_1L_F^2}{\alpha_2\mu^2}\frac{1}{K}\mathbb{E}[\|\Delta_{t+1}^{g}-V_{t+1}\|_F^2] \ . \\
	\end{aligned}
\end{equation}
Similarly, we can know how the potential function evolves in each iteration and then establish the convergence rate. 

\paragraph{Proof Sketch of Theorem~\ref{theorem_vrdbo}.}
Establishing the convergence rate of Algorithm~\ref{alg_VRDBO} is much more challenging due to the complicated variance-reduced gradient estimator. Directly employing the potential function in Eq.~(\ref{potential_1}) cannot give us the desired result.  To address this challenging problem, we developed a novel potential function, which is shown below. 
\begin{equation}
	\begin{aligned}
		& \mathcal{L}_{t+1} = {\mathbb{E}}[F(x_{t+1})] +  \frac{6\beta_1{L}_{F}^2}{\beta_2\mu}\mathbb{E}[\|\bar{   {y}}_{t+1} -    {y}^{*}(\bar{   {x}}_{t+1})\| ^2 ]   \\
		& \quad + w_2\frac{1}{K} \mathbb{E}[\|X_{t+1} - \bar{X}_{t+1}\|_F^2 ] + w_3 \frac{1}{K}\mathbb{E}[\|Y_{t+1} - \bar{Y}_{t+1}\|_F^2 ] \\
		& \quad +\beta_1(1-\lambda) \frac{1}{K}\mathbb{E}[\| Z^{\tilde{F}}_{t+1}-\bar{Z}^{\tilde{F}}_{t+1} \|_F^2]  +\frac{\beta_1(1-\lambda) L_F^2}{\mu^2}\frac{1}{K}\mathbb{E}[\| Z^{g}_{t+1}-\bar{Z}^{g}_{t+1} \|_F^2]  \\
		& \quad + 2\beta_1 \frac{1}{K}\mathbb{E}[\|\Delta_{t+1}^{\tilde{F}}-U_{t+1}\|_F^2] + \frac{2\beta_1L_F^2}{\mu^2}\frac{1}{K}\mathbb{E}[\|\Delta_{t+1}^{g}-V_{t+1}\|_F^2]  \\
		& \quad + \frac{3\beta_1}{\alpha_1\eta} \mathbb{E}[\|(\Delta_{t+1}^{\tilde{F}}-U_{t+1})\frac{1}{K}\mathbf{1}\|^2]  + \frac{50\beta_1{L}_{F}^2}{\alpha_2\eta\mu^2}\mathbb{E}[\|(\Delta_{t+1}^{g}-V_{t+1})\frac{1}{K}\mathbf{1}\|^2] \ , \\
	\end{aligned}
\end{equation}
where $w_2=w_3=\frac{2\beta_1\ell_{g_y}^2((51+48/\alpha_1K){L}_{\tilde{F}}^2+(98+800/\alpha_2K){L}_{F}^2)}{\mu^2(1-\lambda^2)}$.  Compared with Eq.~(\ref{potential_1}), there are  two additional terms, which are critical to achieve linear speedup. In particular, with this novel design, we are able to set $\alpha_i=1/K$. As a result, $\beta_i$ and $\alpha_i$ are decoupled (See Eq.~(\ref{beta_vr})).  
In this way, Algorithm~\ref{alg_VRDBO} can achieve linear speedup with respect to the number of participants (See Corollary~\ref{corollary_vrdbo}).
In comparison, $\alpha_i$ has to be $O(1)$ in Theorem~\ref{theorem_mdbo}. 

In summary, we proposed three novel potential functions for establishing the convergence rate in Theorems~\ref{theorem_mdbo}-~\ref{theorem_vrdbo}, disclosing how  hyperparameters affect  convergence rates. 

\subsection{Proof of  Theorem~\ref{theorem_mdbo}}

\subsubsection{Characterization of  $F^{(k)}(x)$}
\begin{lemma}  \cite{ghadimi2018approximation} \label{lemma_hypergrad_smooth_optimal}
	Given Assumptions~\ref{assumption_bi_strong}-\ref{assumption_lower_smooth}, the following inequalities hold. 
	\begin{equation}
		\begin{aligned}
			& \| \nabla F^{(k)}(x) - {\nabla} F^{(k)}(x, y)    \| \leq  {L}_F \|y - y^*(x)\|  \ ,  \\
			& \| \nabla F^{(k)}(x_1) - {\nabla} F^{(k)}(x_2)    \| \leq  {L}_{F}^{*} \|x_1 - x_2\|  \ ,  \\
			& \|y^*(x_1)-y^*(x_2)\| \leq L_y \|x_1 - x_2\| \ , 
		\end{aligned}
	\end{equation}
	where ${L}_F=L_{f_x}+\frac{L_{f_y}C_{g_{xy}}}{\mu}+\frac{C_{f_y}L_{g_{xy}}}{\mu}+\frac{L_{g_{yy}}C_{f_{y}}C_{g_{xy}}}{\mu^2}$ $L_F^*=L_F+\frac{L_FC_{g_{xy}}}{\mu}$, $L_y=\frac{C_{g_{xy}}}{\mu}$.
\end{lemma}

\subsubsection{Characterization of  $\nabla \tilde{F}^{(k)}(x, y)$}
\begin{lemma} \cite{ghadimi2018approximation} \label{lemma_hessian_bound}
	Given Assumptions~\ref{assumption_bi_strong}-\ref{assumption_lower_smooth}, the following inequalities hold. 
	\begin{equation}
		\begin{aligned}
			& \mathbb{E}\Big[\frac{J}{L_{g_{y}}}\prod_{j=1}^{\tilde{J}}(I-\frac{1}{L_{g_{y}}}\nabla_{yy}^2g^{(k)}(x, y; \zeta_j))\Big] = \frac{1}{L_{g_{y}}} \sum_{j=0}^{J-1}\Big(I-\frac{1}{L_{g_{y}}}\nabla_{yy}^2g^{(k)}(x, y)\Big)^{j}  \ , \\
			& \Big\|\Big(\nabla_{yy}^2g^{(k)}(x, y)\Big)^{-1}- \mathbb{E}\Big[\frac{J}{L_{g_{y}}}\prod_{j=1}^{\tilde{J}}(I-\frac{1}{L_{g_{y}}}\nabla_{yy}^2g^{(k)}(x, y; \zeta_j))\Big]\Big\| \leq \frac{1}{\mu}(1-\frac{\mu}{L_{g_{y}}})^{J}  \ , \\
			&  \mathbb{E}\Big [\Big\|\Big(\nabla_{yy}^2g^{(k)}(x, y)\Big)^{-1}- \frac{J}{L_{g_{y}}}\prod_{j=1}^{\tilde{J}}(I-\frac{1}{L_{g_{y}}}\nabla_{yy}^2g^{(k)}(x, y; \zeta_j))\Big\| \Big]\leq \frac{2}{\mu}  \ , \\
			& \mathbb{E}\Big[\Big\| \frac{J}{L_{g_{y}}}\prod_{j=1}^{\tilde{J}}(I-\frac{1}{L_{g_{y}}}\nabla_{yy}^2g^{(k)}(x, y; \zeta_j))\Big\| \Big] \leq \frac{1}{\mu}  \ . \\
		\end{aligned}
	\end{equation}
\end{lemma}

\begin{lemma} (Bias) \cite{ghadimi2018approximation} \label{lemma_hypergrad_bias}
	Given Assumptions~\ref{assumption_bi_strong}-\ref{assumption_lower_smooth}, the approximation error of $\nabla \tilde{F}^{(k)}(x, y)$ for $\nabla F^{(k)}(x, y)$ can be bounded as follows:
	\begin{equation}
		\begin{aligned}
			& \|\nabla F^{(k)}(x, y)   -\nabla \tilde{F}^{(k)}(x, y)  \| \leq  \frac{C_{g_{xy}}C_{f_y}}{\mu}(1-\frac{\mu}{L_{g_{y}}})^{J} \ . 
		\end{aligned}
	\end{equation}
\end{lemma}

\begin{lemma} (Variance) \label{lemma_hypergrad_var}
	Given Assumptions~\ref{assumption_bi_strong}-\ref{assumption_lower_smooth},  the variance of the stochastic hypergradient can be bounded as follows: 
	\begin{equation}
		\begin{aligned}
			& \mathbb{E}[\|{\nabla} \tilde{F} ^{(k)}(x, y) -  {\nabla} \tilde{F} ^{(k)}(x, y; \tilde{\xi})  \|^2] \leq \sigma_{\tilde{F}}^2 \ ,
		\end{aligned}
	\end{equation}
	where $\sigma_{\tilde{F}}^2=4\sigma^2 +  \frac{4C_{f_y}^2\sigma^2}{\mu^2}  + \frac{4\sigma^2(\sigma^2+C_{g_{xy}}^2)}{\mu^2} + \frac{16(\sigma^2+C_{g_{xy}}^2)(\sigma^2+C_{f_y}^2)}{\mu^2}$. 
\end{lemma}

\begin{proof}
	\begin{equation}
		\begin{aligned}
			&  \quad \mathbb{E}[\|{\nabla} \tilde{F} ^{(k)}(x, y) -  {\nabla} \tilde{F} ^{(k)}(x, y; \tilde{\xi})  \|^2] \\
			& =  \mathbb{E}\Big[\Big\|\nabla_x f^{(k)}(x, y) -  \nabla_{xy}^2g^{(k)}(x, y) \mathbb{E}\Big[\frac{J}{L_{g_{y}}}\prod_{j=1}^{\tilde{J}}(I-\frac{1}{L_{g_{y}}}\nabla_{yy}^2g^{(k)}(x, y; \zeta_j))\Big]\nabla_y f^{(k)}(x, y) \\
			& \quad -  \nabla_x f^{(k)}(x, y; \xi) +  \nabla_{xy}^2g^{(k)}(x, y; \zeta_0) \Big(\frac{J}{L_{g_{y}}}\prod_{j=1}^{\tilde{J}}(I-\frac{1}{L_{g_{y}}}\nabla_{yy}^2g^{(k)}(x, y; \zeta_j))\Big)\nabla_y f^{(k)}(x, y; \xi)  \Big\|^2\Big] \\
			& = \mathbb{E}\Big[\Big\|\nabla_x f^{(k)}(x, y) - \nabla_x f^{(k)}(x, y; \xi) \\
			& \quad + \nabla_{xy}^2g^{(k)}(x, y) \mathbb{E}\Big[\frac{J}{L_{g_{y}}}\prod_{j=1}^{\tilde{J}}(I-\frac{1}{L_{g_{y}}}\nabla_{yy}^2g^{(k)}(x, y; \zeta_j))\Big]\nabla_y f^{(k)}(x, y)  \\
			& \quad - \nabla_{xy}^2g^{(k)}(x, y; \zeta_0) \mathbb{E}\Big[\frac{J}{L_{g_{y}}}\prod_{j=1}^{\tilde{J}}(I-\frac{1}{L_{g_{y}}}\nabla_{yy}^2g^{(k)}(x, y; \zeta_j))\Big]\nabla_y f^{(k)}(x, y)  \\
			& \quad + \nabla_{xy}^2g^{(k)}(x, y; \zeta_0) \mathbb{E}\Big[\frac{J}{L_{g_{y}}}\prod_{j=1}^{\tilde{J}}(I-\frac{1}{L_{g_{y}}}\nabla_{yy}^2g^{(k)}(x, y; \zeta_j))\Big]\nabla_y f^{(k)}(x, y)  \\
			& \quad - \nabla_{xy}^2g^{(k)}(x, y; \zeta_0) \mathbb{E}\Big[\frac{J}{L_{g_{y}}}\prod_{j=1}^{\tilde{J}}(I-\frac{1}{L_{g_{y}}}\nabla_{yy}^2g^{(k)}(x, y; \zeta_j))\Big]\nabla_y f^{(k)}(x, y; \xi)  \\
			& \quad + \nabla_{xy}^2g^{(k)}(x, y; \zeta_0) \mathbb{E}\Big[\frac{J}{L_{g_{y}}}\prod_{j=1}^{\tilde{J}}(I-\frac{1}{L_{g_{y}}}\nabla_{yy}^2g^{(k)}(x, y; \zeta_j))\Big]\nabla_y f^{(k)}(x, y; \xi)  \\
			& \quad  - \nabla_{xy}^2g^{(k)}(x, y; \zeta_0) \Big(\frac{J}{L_{g_{y}}}\prod_{j=1}^{\tilde{J}}(I-\frac{1}{L_{g_{y}}}\nabla_{yy}^2g^{(k)}(x, y; \zeta_j))\Big)\nabla_y f^{(k)}(x, y; \xi)  \Big\|^2 \Big]\\
			& \leq 4\sigma^2 + 4 \mathbb{E}\Big[\Big\|\Big(\nabla_{xy}^2g^{(k)}(x, y) - \nabla_{xy}^2g^{(k)}(x, y; \zeta_0)\Big) \mathbb{E}\Big[\frac{J}{L_{g_{y}}}\prod_{j=1}^{\tilde{J}}(I-\frac{1}{L_{g_{y}}}\nabla_{yy}^2g^{(k)}(x, y; \zeta_j))\Big]\nabla_y f^{(k)}(x, y)  \Big\|^2 \Big]\\
			& \quad + 4\mathbb{E}\Big[\Big\|\nabla_{xy}^2g^{(k)}(x, y; \zeta_0) \mathbb{E}\Big[\frac{J}{L_{g_{y}}}\prod_{j=1}^{\tilde{J}}(I-\frac{1}{L_{g_{y}}}\nabla_{yy}^2g^{(k)}(x, y; \zeta_j))\Big]\Big(\nabla_y f^{(k)}(x, y)- \nabla_y f^{(k)}(x, y; \xi)\Big)\Big\|^2 \Big]\\
			& \quad + 4\mathbb{E}\Big[\Big\|\nabla_{xy}^2g^{(k)}(x, y; \zeta_0) \Big(\mathbb{E}\Big[\frac{J}{L_{g_{y}}}\prod_{j=1}^{\tilde{J}}(I-\frac{1}{L_{g_{y}}}\nabla_{yy}^2g^{(k)}(x, y; \zeta_j))\Big]\\
			& \quad \quad -\frac{J}{L_{g_{y}}}\prod_{j=1}^{\tilde{J}}(I-\frac{1}{L_{g_{y}}}\nabla_{yy}^2g^{(k)}(x, y; \zeta_j))\Big)\nabla_y f^{(k)}(x, y; \xi)  \Big \|^2  \Big]\\
			&\overset{(s_1)} \leq 4\sigma^2 +  \frac{4C_{f_y}^2\sigma^2}{\mu^2}  + \frac{4\sigma^2(\sigma^2+C_{g_{xy}}^2)}{\mu^2} + \frac{16(\sigma^2+C_{g_{xy}}^2)(\sigma^2+C_{f_y}^2)}{\mu^2} \ , 
		\end{aligned}
	\end{equation}
	where $(s_1)$ holds due to Lemma~\ref{lemma_hessian_bound},  Assumptions~\ref{assumption_upper_smooth}-\ref{assumption_lower_smooth}, and the following inequality.
	\begin{equation}
		\begin{aligned}
			& \quad \Big\|\mathbb{E}\Big[\frac{J}{L_{g_{y}}}\prod_{j=1}^{\tilde{J}}(I-\frac{1}{L_{g_{y}}}\nabla_{yy}^2g^{(k)}(x, y; \zeta_j))\Big]\Big\| = \Big\|\frac{1}{L_{g_{y}}} \sum_{j=0}^{J-1}\Big(I-\frac{1}{L_{g_{y}}}\nabla_{yy}^2g^{(k)}(x, y)\Big)^{j} \Big\| \\
			& \leq \frac{1}{L_{g_{y}}} \sum_{j=0}^{J-1} \Big\|\Big(I-\frac{1}{L_{g_{y}}}\nabla_{yy}^2g^{(k)}(x, y)\Big)^{j} \Big\| \\
			& \overset{(s_1)}\leq \frac{1}{L_{g_{y}}} \sum_{j=0}^{J-1} \Big(1-\frac{\mu}{L_{g_{y}}}\Big)^{j}  \\
			& \leq \frac{1}{\mu} \ , 
		\end{aligned}
	\end{equation}
	where $(s_1)$ holds due to Assumption~\ref{assumption_bi_strong}. 
\end{proof}

\begin{lemma} (Smoothness) \label{lemma_hypergrad_smooth}
	Given Assumptions~\ref{assumption_bi_strong}-\ref{assumption_lower_smooth},  the approximated hypergradient ${\nabla}\tilde{F}^{(k)}(x, y)$ is ${L}_{\tilde{F}}$-Lipschitz continuous: 
	\begin{equation}
		\begin{aligned}
			&  \|{\nabla}\tilde{F}^{(k)}(x_1, y_1)-{\nabla}\tilde{F}^{(k)}(x_{2}, y_{2})\|^2 \leq {L}_{\tilde{F}}^2(\|x_1-x_{2}\|^2+\|y_1-y_{2}\|^2),
		\end{aligned}
	\end{equation}
	where ${L}_{\tilde{F}}^2=4L_{f_x}^2+4 C_{g_{xy}}^2C_{f_{y}}^2\frac{J^2L_{g_{yy}}^2}{\mu^2 L_{g_{y}}^2} + \frac{4C_{g_{xy}}^2L_{f_y}^2}{\mu^2} + \frac{L_{g_{xy}}^2C_{f_y}^2}{\mu^2}$, $(x_1, y_1)\in\mathbb{R}^{d_x}\times \mathbb{R}^{d_y} $, $(x_2, y_2)\in\mathbb{R}^{d_x}\times \mathbb{R}^{d_y}$. 
\end{lemma}

\begin{proof}
	\begin{equation}
		\begin{aligned}
			&  \quad \|{\nabla}\tilde{F}^{(k)}(x_1, y_1)-{\nabla}\tilde{F}^{(k)}(x_{2}, y_{2})\|^2  \\
			& = \Big\|\nabla_x f^{(k)}(x_1, y_1) -  \nabla_{xy}^2g^{(k)}(x_1, y_1) \mathbb{E}\Big[\frac{J}{L_{g_{y}}}\prod_{j=1}^{\tilde{J}}(I-\frac{1}{L_{g_{y}}}\nabla_{yy}^2g^{(k)}(x_1, y_1; \zeta_j))\Big]\nabla_y f^{(k)}(x_1, y_1) \\
			& \quad - \nabla_x f^{(k)}(x_{2}, y_{2}) +   \nabla_{xy}^2g^{(k)}(x_{2}, y_{2}) \mathbb{E}\Big[\frac{J}{L_{g_{y}}}\prod_{j=1}^{\tilde{J}}(I-\frac{1}{L_{g_{y}}}\nabla_{yy}^2g^{(k)}(x_{2}, y_{2}; \zeta_j))\Big]\nabla_y f^{(k)}(x_{2}, y_{2}) \Big\|^2  \\
			& \leq \Big\|\nabla_x f^{(k)}(x_1, y_1)  - \nabla_x f^{(k)}(x_{2}, y_{2})\\
			& \quad  -  \nabla_{xy}^2g^{(k)}(x_1, y_1) \mathbb{E}\Big[\frac{J}{L_{g_{y}}}\prod_{j=1}^{\tilde{J}}(I-\frac{1}{L_{g_{y}}}\nabla_{yy}^2g^{(k)}(x_1, y_1; \zeta_j))\Big]\nabla_y f^{(k)}(x_1, y_1) \\
			& \quad  +  \nabla_{xy}^2g^{(k)}(x_1, y_1) \mathbb{E}\Big[\frac{J}{L_{g_{y}}}\prod_{j=1}^{\tilde{J}}(I-\frac{1}{L_{g_{y}}}\nabla_{yy}^2g^{(k)}(x_2, y_2; \zeta_j))\Big]\nabla_y f^{(k)}(x_1, y_1) \\
			& \quad  -  \nabla_{xy}^2g^{(k)}(x_1, y_1) \mathbb{E}\Big[\frac{J}{L_{g_{y}}}\prod_{j=1}^{\tilde{J}}(I-\frac{1}{L_{g_{y}}}\nabla_{yy}^2g^{(k)}(x_2, y_2; \zeta_j))\Big]\nabla_y f^{(k)}(x_1, y_1) \\
			& \quad  +  \nabla_{xy}^2g^{(k)}(x_1, y_1) \mathbb{E}\Big[\frac{J}{L_{g_{y}}}\prod_{j=1}^{\tilde{J}}(I-\frac{1}{L_{g_{y}}}\nabla_{yy}^2g^{(k)}(x_2, y_2; \zeta_j))\Big]\nabla_y f^{(k)}(x_2, y_2) \\
			& \quad  -  \nabla_{xy}^2g^{(k)}(x_1, y_1) \mathbb{E}\Big[\frac{J}{L_{g_{y}}}\prod_{j=1}^{\tilde{J}}(I-\frac{1}{L_{g_{y}}}\nabla_{yy}^2g^{(k)}(x_2, y_2; \zeta_j))\Big]\nabla_y f^{(k)}(x_2, y_2) \\
			& \quad +   \nabla_{xy}^2g^{(k)}(x_{2}, y_{2}) \mathbb{E}\Big[\frac{J}{L_{g_{y}}}\prod_{j=1}^{\tilde{J}}(I-\frac{1}{L_{g_{y}}}\nabla_{yy}^2g^{(k)}(x_{2}, y_{2}; \zeta_j))\Big]\nabla_y f^{(k)}(x_{2}, y_{2}) \Big\|^2  \\
			&\overset{(s_1)}\leq \Big(4L_{f_x}^2+4 C_{g_{xy}}^2C_{f_{y}}^2\frac{J^2L_{g_{yy}}^2}{\mu^2 L_{g_{y}}^2} + \frac{4C_{g_{xy}}^2L_{f_y}^2}{\mu^2} + \frac{L_{g_{xy}}^2C_{f_y}^2}{\mu^2}\Big)(\|x_1-x_{2}\|^2+\|y_1-y_{2}\|^2)  \ , \\
		\end{aligned}
	\end{equation} 
	where $(s_1)$ holds due to Assumptions~\ref{assumption_upper_smooth}-\ref{assumption_lower_smooth} and the following inequality. 
	\newpage
	\begin{equation}
		\begin{aligned}
			& \quad \Big\|\mathbb{E}\Big[\frac{J}{L_{g_{y}}}\prod_{j=1}^{\tilde{J}}(I-\frac{1}{L_{g_{y}}}\nabla_{yy}^2g^{(k)}(x_1, y_1; \zeta_j))\Big] - \mathbb{E}\Big[\frac{J}{L_{g_{y}}}\prod_{j=1}^{\tilde{J}}(I-\frac{1}{L_{g_{y}}}\nabla_{yy}^2g^{(k)}(x_2, y_2; \zeta_j))\Big]\Big\| \\
			&  = \Big\|\frac{1}{L_{g_{y}}} \sum_{j=0}^{J-1}\Big(I-\frac{1}{L_{g_{y}}}\nabla_{yy}^2g^{(k)}(x_1, y_1)\Big)^{j}  - \frac{1}{L_{g_{y}}} \sum_{j=0}^{J-1}\Big(I-\frac{1}{L_{g_{y}}}\nabla_{yy}^2g^{(k)}(x_2, y_2)\Big)^{j} \Big\| \\
			& \leq  \frac{1}{L_{g_{y}}}\sum_{j=1}^{J-1}\Big\| \Big(I-\frac{1}{L_{g_{y}}}\nabla_{yy}^2g^{(k)}(x_1, y_1)\Big)^{j}  - \Big(I-\frac{1}{L_{g_{y}}}\nabla_{yy}^2g^{(k)}(x_2, y_2)\Big)^{j} \Big\| \\
			& \overset{(s_1)}=  \frac{1}{L_{g_{y}}}\sum_{j=1}^{J-1}\Bigg(\Big\|  \Big(\frac{1}{L_{g_{y}}}\nabla_{yy}^2g^{(k)}(x_1, y_1)-\frac{1}{L_{g_{y}}}\nabla_{yy}^2g^{(k)}(x_2, y_2)\Big)\\
			& \quad \times \Big(\sum_{i=0}^{j-1}\Big(I-\frac{1}{L_{g_{y}}}\nabla_{yy}^2g^{(k)}(x_1, y_1)\Big)^{i}\Big(I-\frac{1}{L_{g_{y}}}\nabla_{yy}^2g^{(k)}(x_2, y_2)\Big)^{j-1-i}\Big)  \Big\|\Bigg) \\
			& \leq  \frac{1}{L_{g_{y}}}\sum_{j=1}^{J-1}\Bigg(\Big\|  \frac{1}{L_{g_{y}}}\nabla_{yy}^2g^{(k)}(x_1, y_1)-\frac{1}{L_{g_{y}}}\nabla_{yy}^2g^{(k)}(x_2, y_2)\Big\|\\
			& \quad \quad \times \Big \|\sum_{i=0}^{j-1}\Big(I-\frac{1}{L_{g_{y}}}\nabla_{yy}^2g^{(k)}(x_2, y_2)\Big)^{i}\Big(I-\frac{1}{L_{g_{y}}}\nabla_{yy}^2g^{(k)}(x_2, y_2)\Big)^{j-1-i}\Big\| \Bigg)\\
			& \leq  \frac{1}{L_{g_{y}}}\sum_{j=1}^{J-1}\Big(\frac{L_{g_{yy}}}{L_{g_{y}}}\|  (x_1, y_1)-(x_2, y_2)\| \times \sum_{i=0}^{j-1}(1-\frac{\mu}{L_{g_y}})^{j-1}\Big)\\
			& \leq  \frac{J}{L_{g_{y}}}\sum_{j=1}^{J-1}\Big(\frac{L_{g_{yy}}}{L_{g_{y}}}\|  (x_1, y_1)-(x_2, y_2)\| \times (1-\frac{\mu}{L_{g_y}})^{j-1}\Big)\\
			& \leq  \frac{JL_{g_{yy}}}{\mu L_{g_{y}}}\|  (x_1, y_1)-(x_2, y_2)\| \ , \\
		\end{aligned}
	\end{equation}
	where $(s_1)$ holds due to $a^n-b^n=(a-b)(\sum_{i=0}^{n-1}a^ib^{n-1-i})$.
	
	
\end{proof}

\subsubsection{Characterization of Gradient Estimators}
\begin{lemma} \label{lemma_hyper_momentum_var}
	Given Assumptions~\ref{assumption_graph}-\ref{assumption_lower_smooth},   the following inequality holds. 
	\begin{equation}
		\begin{aligned}
			&  \frac{1}{K}\mathbb{E}[\|\Delta_t^{\tilde{F}}-   U_t\|^2]\leq (1-\alpha_1\eta)\frac{1}{K}\mathbb{E}[\|\Delta_{t-1}^{\tilde{F}}-  U_{t-1} \|^2] +\frac{{L}_{\tilde{F}}^2}{\alpha_1\eta}\frac{1}{K} \mathbb{E}[\|X_{t}-X_{t-1}\|^2]\\
			& \quad +\frac{{L}_{\tilde{F}}^2}{\alpha_1\eta}\frac{1}{K} \mathbb{E}[\|Y_{t}-Y_{t-1}\|^2] + \alpha_1^2\eta^2 \sigma_{\tilde{F}}^2 \ . \\
		\end{aligned}
	\end{equation}
	
\end{lemma}

\begin{proof}
	\begin{equation}
		\begin{aligned}
			& \quad \frac{1}{K}\mathbb{E}[\|\Delta_t^{\tilde{F}}-   U_t\|^2]\\
			& = \frac{1}{K}\mathbb{E}[\|\Delta_t^{\tilde{F}}-   (1-\alpha_1\eta)U_{t-1} -  \alpha_1\eta\Delta_t^{\tilde{F}_{\tilde{\xi}_t}}\|^2]\\
			& = \frac{1}{K}\mathbb{E} [ \|  (1-\alpha_1\eta)(\Delta_{t-1}^{\tilde{F}}-  U_{t-1}) +(1-\alpha_1\eta)(\Delta_{t}^{\tilde{F}}-\Delta_{t-1}^{\tilde{F}})+ \alpha_1\eta(\Delta_{t}^{\tilde{F}}- \Delta_{t}^{\tilde{F}_{\tilde{\xi}_t}})  \|^2 ]\\
			&\overset{(s_1)}= (1-\alpha_1\eta)^2\frac{1}{K}\mathbb{E} [ \|  (\Delta_{t-1}^{\tilde{F}}-  U_{t-1}) +(\Delta_{t}^{\tilde{F}}-\Delta_{t-1}^{\tilde{F}})  \|^2 ]+\alpha_1^2\eta^2\frac{1}{K}\mathbb{E} [ \|\Delta_{t}^{\tilde{F}}- \Delta_{t}^{\tilde{F}_{\tilde{\xi}_t}} \|^2 ]\\
			&\overset{(s_2)}\leq (1-\alpha_1\eta) \frac{1}{K}\mathbb{E} [ \|\Delta_{t-1}^{\tilde{F}}-  U_{t-1}\|^2 ]+\frac{1}{\alpha_1\eta}\frac{1}{K}\mathbb{E} [ \|\Delta_{t}^{\tilde{F}}-\Delta_{t-1}^{\tilde{F}} \|^2 ]  +  \alpha_1^2\eta^2 \sigma_{\tilde{F}}^2\\
			&\overset{(s_3)}\leq (1-\alpha_1\eta)\frac{1}{K}\mathbb{E} [ \|\Delta_{t-1}^{\tilde{F}}-  U_{t-1} \|^2 ] +\frac{{L}_{\tilde{F}}^2}{\alpha_1\eta}\frac{1}{K} \mathbb{E}[\|X_{t}-X_{t-1}\|^2]\\
			& \quad +\frac{{L}_{\tilde{F}}^2}{\alpha_1\eta}\frac{1}{K} \mathbb{E}[\|Y_{t}-Y_{t-1}\|^2] + \alpha_1^2\eta^2 \sigma_{\tilde{F}}^2\ , \\
		\end{aligned}
	\end{equation}
	where $(s_1)$ holds due to  $\Delta_{t}^{\tilde{F}}= \mathbb{E}[\Delta_{t}^{\tilde{F}_{\tilde{\xi}_t}}]$,  
	$(s_2)$ holds due to Lemma~\ref{lemma_hypergrad_var} and Lemma~\ref{lemma_ineqality} with  $a=\frac{\alpha_1\eta}{1-\alpha_1\eta}$, $(s_3)$ holds due to Lemma~\ref{lemma_hypergrad_smooth}. 
\end{proof}

\begin{lemma} \label{lemma_lower_momentum_var}
	Given Assumptions~\ref{assumption_graph}-\ref{assumption_lower_smooth},   the following inequality holds. 
	\begin{equation}
		\begin{aligned}
			&  \frac{1}{K}\mathbb{E} [ \|\Delta_t^{g}-V_t \|^2 ]\leq (1-\alpha_2\eta)\frac{1}{K}\mathbb{E} [ \| \Delta_{t-1}^{g}-  V_{t-1} \|^2 ]+\frac{{L}_{g_{y}}^2}{\alpha_2\eta}\frac{1}{K} \mathbb{E}[\|X_t - X_{t-1}\|_F^2]\\
			& \quad +\frac{{L}_{g_y}^2}{\alpha_2\eta}\frac{1}{K} \mathbb{E}[\|Y_{t}-Y_{t-1}\|_F^2] + \alpha_2^2\eta^2 \sigma^2 \ . \\
		\end{aligned}
	\end{equation}

\end{lemma}

\begin{proof}
	\begin{equation}
		\begin{aligned}
			& \quad \frac{1}{K}\mathbb{E} [ \|\Delta_t^{g}-V_t \|^2 ]\\
			& = \frac{1}{K}\mathbb{E} [ \| \Delta_t^{g} -  (1-\alpha_2\eta)V_{t-1} -  \alpha_2\eta\Delta_t^{g_{\zeta_t}}  \|^2 ]\\
			& = \frac{1}{K}\mathbb{E} [ \|  (1-\alpha_2\eta)(\Delta_{t-1}^{g}-  V_{t-1}) +(1-\alpha_2\eta)(\Delta_{t}^{g}-\Delta_{t-1}^{g} )+ \alpha_2\eta(\Delta_t^{g}- \Delta_t^{g_{\zeta_t}})  \|^2 ]\\
			&\overset{(s_1)}= (1-\alpha_2\eta)^2\frac{1}{K}\mathbb{E} [ \|  (\Delta_{t-1}^{g}-  V_{t-1}) +(\Delta_{t}^{g}-\Delta_{t-1}^{g} )  \|^2 ] +\alpha_2^2\eta^2\frac{1}{K}\mathbb{E} [ \|\Delta_t^{g}- \Delta_t^{g_{\zeta_t}} \|^2 ]\\
			&\overset{(s_2)}\leq (1-\alpha_2\eta)\frac{1}{K}\mathbb{E} [ \| \Delta_{t-1}^{g}-  V_{t-1} \|^2 ] +\frac{1}{\alpha_2\eta}\frac{1}{K}\mathbb{E} [ \|\Delta_{t}^{g}-\Delta_{t-1}^{g}  \|^2 ]  + \alpha_2^2\eta^2 \sigma^2 \\
			&\overset{(s_3)}\leq (1-\alpha_2\eta)\frac{1}{K} \mathbb{E} [ \| \Delta_{t-1}^{g}-  V_{t-1}  \|^2 ]+\frac{{L}_{g_{y}}^2}{\alpha_2\eta}\frac{1}{K} \mathbb{E}[\|X_t - X_{t-1}\|_F^2]\\
			& \quad +\frac{{L}_{g_y}^2}{\alpha_2\eta}\frac{1}{K} \mathbb{E}[\|Y_{t}-Y_{t-1}\|_F^2] + \alpha_2^2\eta^2 \sigma^2 \ , \\
		\end{aligned}
	\end{equation}
	where $(s_1)$ holds due to  $\Delta_t^{g} = \mathbb{E}[\Delta_t^{g_{\zeta_t}}]$, 
	$(s_2)$ holds due to Assumption~\ref{assumption_variance} and Lemma~\ref{lemma_ineqality} with $a=\frac{\alpha_2\eta}{1-\alpha_2\eta}$, $(s_3)$ holds due to Assumption~\ref{assumption_lower_smooth}. 
\end{proof}

\subsubsection{Characterization of Consensus Errors}
\begin{lemma} \label{lemma_consensus_z_f}
	Given Assumptions~\ref{assumption_graph}-\ref{assumption_lower_smooth},   the following inequality holds. 
	\begin{equation}
		\begin{aligned}
			&  \frac{1}{K}\mathbb{E}[\|Z_{t}^{\tilde{F}} - \bar{Z}_{t}^{\tilde{F}}\|_F^2] 	\leq  \lambda\frac{1}{K}\mathbb{E}[\|Z_{t-1}^{\tilde{F}} - \bar{Z}_{t-1}^{\tilde{F}} \|_F^2]  +\frac{2\alpha_1^2\eta^2}{1-\lambda}\frac{1}{K}\mathbb{E}[\| U_{t-1}- \Delta_{t-1}^{\tilde{F}}\|_F^2] +\frac{\alpha_1^2\eta^2\sigma_{\tilde{F}}^2}{1-\lambda}  \\
			& \quad +\frac{2\alpha_1^2\eta^2L_{\tilde{F}}^2}{1-\lambda}\frac{1}{K}\mathbb{E}[\| X_{t}- X_{t-1}\|_F^2] +\frac{2\alpha_1^2\eta^2L_{\tilde{F}}^2}{1-\lambda}\frac{1}{K}\mathbb{E}[\| Y_{t}- Y_{t-1}\|_F^2]  \ .\\
		\end{aligned}
	\end{equation}
\end{lemma}

\begin{proof}
	\begin{equation}
		\begin{aligned}
			& \quad \frac{1}{K}\mathbb{E}[\|Z_{t}^{\tilde{F}} - \bar{Z}_{t}^{\tilde{F}}\|_F^2] \\
			& = \frac{1}{K}\mathbb{E}[\|Z_{t-1}^{\tilde{F}}W +U_{t} - U_{t-1} - \bar{Z}_{t-1}^{\tilde{F}}- \bar{U}_{t} +\bar{U}_{t-1}\|_F^2] \\
			&\overset{(s_1)}\leq \frac{1}{K}\lambda\mathbb{E}[\|Z_{t-1}^{\tilde{F}} - \bar{Z}_{t-1}^{\tilde{F}} \|_F^2]  + \frac{1}{1-\lambda}\frac{1}{K}\mathbb{E}[\|U_{t} - U_{t-1} - \bar{U}_{t} +\bar{U}_{t-1}\|_F^2]\\
			& \leq \lambda\frac{1}{K}\mathbb{E}[\|Z_{t-1}^{\tilde{F}} - \bar{Z}_{t-1}^{\tilde{F}} \|_F^2]  + \frac{1}{1-\lambda}\frac{1}{K}\mathbb{E}[\|U_{t} - U_{t-1} \|_F^2]\\
			& = \lambda\frac{1}{K}\mathbb{E}[\|Z_{t-1}^{\tilde{F}} - \bar{Z}_{t-1}^{\tilde{F}} \|_F^2]  + \frac{1}{1-\lambda}\frac{1}{K}\mathbb{E}[\|(1-\alpha_1\eta)U_{t-1}+ \alpha_1\eta \Delta_{t}^{\tilde{F}_{\tilde{\xi}_t}}- U_{t-1} \|_F^2]\\
			& = \lambda\frac{1}{K}\mathbb{E}[\|Z_{t-1}^{\tilde{F}} - \bar{Z}_{t-1}^{\tilde{F}} \|_F^2]  + \frac{1}{1-\lambda}\frac{1}{K}\mathbb{E}[\|-\alpha_1\eta U_{t-1}+ \alpha_1\eta \Delta_{t}^{\tilde{F}} - \alpha_1\eta \Delta_{t}^{\tilde{F}}+ \alpha_1\eta \Delta_{t}^{\tilde{F}_{\tilde{\xi}_t}}\|_F^2]\\
			& 	\overset{(s_2)}	\leq \lambda\frac{1}{K}\mathbb{E}[\|Z_{t-1}^{\tilde{F}} - \bar{Z}_{t-1}^{\tilde{F}} \|_F^2]  +\frac{\alpha_1^2\eta^2}{1-\lambda}\frac{1}{K}\mathbb{E}[\| U_{t-1}- \Delta_{t}^{\tilde{F}}\|_F^2]+\frac{\alpha_1^2\eta^2\sigma_{\tilde{F}}^2}{1-\lambda} \\
			& \leq \lambda\frac{1}{K}\mathbb{E}[\|Z_{t-1}^{\tilde{F}} - \bar{Z}_{t-1}^{\tilde{F}} \|_F^2]  +\frac{2\alpha_1^2\eta^2}{1-\lambda}\frac{1}{K}\mathbb{E}[\| U_{t-1}- \Delta_{t-1}^{\tilde{F}}\|_F^2] +\frac{2\alpha_1^2\eta^2}{1-\lambda}\frac{1}{K}\mathbb{E}[\| \Delta_{t-1}^{\tilde{F}}- \Delta_{t}^{\tilde{F}}\|_F^2]+\frac{\alpha_1^2\eta^2\sigma_{\tilde{F}}^2}{1-\lambda} \\
			& \overset{(s_3)}	\leq \lambda\frac{1}{K}\mathbb{E}[\|Z_{t-1}^{\tilde{F}} - \bar{Z}_{t-1}^{\tilde{F}} \|_F^2]  +\frac{2\alpha_1^2\eta^2}{1-\lambda}\frac{1}{K}\mathbb{E}[\| U_{t-1}- \Delta_{t-1}^{\tilde{F}}\|_F^2] +\frac{\alpha_1^2\eta^2\sigma_{\tilde{F}}^2}{1-\lambda}  \\
			& \quad +\frac{2\alpha_1^2\eta^2L_{\tilde{F}}^2}{1-\lambda}\frac{1}{K}\mathbb{E}[\| X_{t}- X_{t-1}\|_F^2] +\frac{2\alpha_1^2\eta^2L_{\tilde{F}}^2}{1-\lambda}\frac{1}{K}\mathbb{E}[\| Y_{t}- Y_{t-1}\|_F^2] \ , 
		\end{aligned}
	\end{equation}
	where $(s_1)$ holds due to Lemma~\ref{lemma_ineqality} with $a=\frac{1-\lambda}{\lambda}$, $(s_2)$ holds due to Lemma~\ref{lemma_hypergrad_var},  $(s_3)$  holds due to Lemma~\ref{lemma_hypergrad_smooth}.

\end{proof}

\begin{lemma} \label{lemma_consensus_z_g}
	Given Assumptions~\ref{assumption_graph}-\ref{assumption_lower_smooth},   the following inequality holds. 
	\begin{equation}
		\begin{aligned}
			&  \frac{1}{K}\mathbb{E}[\|Z_{t}^{g} - \bar{Z}_{t}^{g}\|_F^2] \leq \lambda\frac{1}{K}\mathbb{E}[\|Z_{t-1}^{g} - \bar{Z}_{t-1}^{g} \|_F^2]  +\frac{2\alpha_2^2\eta^2}{1-\lambda}\frac{1}{K}\mathbb{E}[\| {V}_{t-1}- \Delta_{t-1}^{g}\|_F^2]+\frac{\alpha_2^2\eta^2\sigma^2}{1-\lambda}  \\
			& \quad +\frac{2\alpha_2^2\eta^2L_{g_y}^2}{1-\lambda}\frac{1}{K}\mathbb{E}[\| X_{t}- X_{t-1}\|_F^2]+\frac{2\alpha_2^2\eta^2L_{g_y}^2}{1-\lambda}\frac{1}{K}\mathbb{E}[\| Y_{t}- Y_{t-1}\|_F^2] \ . \\
		\end{aligned}
	\end{equation}
\end{lemma}
\begin{proof}
	\begin{equation}
		\begin{aligned}
			& \quad \frac{1}{K}\mathbb{E}[\|Z_{t}^{g} - \bar{Z}_{t}^{g}\|_F^2] \\
			& = \frac{1}{K}\mathbb{E}[\|Z_{t-1}^{g}W +{V}_{t} - {V}_{t-1} - \bar{Z}_{t-1}^{g}- \bar{V}_{t} +\bar{V}_{t-1}\|_F^2] \\
			& \overset{(s_1)} \leq \frac{1}{K}\lambda\mathbb{E}[\|Z_{t-1}^{g} - \bar{Z}_{t-1}^{g} \|_F^2]  + \frac{1}{1-\lambda}\frac{1}{K}\mathbb{E}[\|{V}_{t} - {V}_{t-1} - \bar{V}_{t} +\bar{V}_{t-1}\|_F^2]\\
			& \leq \lambda\frac{1}{K}\mathbb{E}[\|Z_{t-1}^{g} - \bar{Z}_{t-1}^{g} \|_F^2]  + \frac{1}{1-\lambda}\frac{1}{K}\mathbb{E}[\|{V}_{t} - {V}_{t-1} \|_F^2]\\
			& = \lambda\frac{1}{K}\mathbb{E}[\|Z_{t-1}^{g} - \bar{Z}_{t-1}^{g} \|_F^2]  + \frac{1}{1-\lambda}\frac{1}{K}\mathbb{E}[\|(1-\alpha_2\eta){V}_{t-1}+ \alpha_2\eta \Delta_{t}^{g_{\zeta_t}}- {V}_{t-1} \|_F^2]\\
			& = \lambda\frac{1}{K}\mathbb{E}[\|Z_{t-1}^{g} - \bar{Z}_{t-1}^{g} \|_F^2]  + \frac{1}{1-\lambda}\frac{1}{K}\mathbb{E}[\|-\alpha_2\eta {V}_{t-1}+ \alpha_2\eta \Delta_{t}^{g} - \alpha_2\eta \Delta_{t}^{g}+ \alpha_2\eta \Delta_{t}^{g_{\zeta_t}}\|_F^2]\\
			& 		\overset{(s_2)}		\leq \lambda\frac{1}{K}\mathbb{E}[\|Z_{t-1}^{g} - \bar{Z}_{t-1}^{g} \|_F^2]  +\frac{\alpha_2^2\eta^2}{1-\lambda}\frac{1}{K}\mathbb{E}[\| {V}_{t-1}- \Delta_{t}^{g}\|_F^2]+\frac{\alpha_2^2\eta^2\sigma^2}{1-\lambda}  \\
			& 		\leq \lambda\frac{1}{K}\mathbb{E}[\|Z_{t-1}^{g} - \bar{Z}_{t-1}^{g} \|_F^2]  +\frac{2\alpha_2^2\eta^2}{1-\lambda}\frac{1}{K}\mathbb{E}[\| {V}_{t-1}- \Delta_{t-1}^{g}\|_F^2]+\frac{2\alpha_2^2\eta^2}{1-\lambda}\frac{1}{K}\mathbb{E}[\| \Delta_{t-1}^{g}- \Delta_{t}^{g}\|_F^2]+\frac{\alpha_2^2\eta^2\sigma^2}{1-\lambda}  \\
			& 		\overset{(s_3)}		\leq \lambda\frac{1}{K}\mathbb{E}[\|Z_{t-1}^{g} - \bar{Z}_{t-1}^{g} \|_F^2]  +\frac{2\alpha_2^2\eta^2}{1-\lambda}\frac{1}{K}\mathbb{E}[\| {V}_{t-1}- \Delta_{t-1}^{g}\|_F^2]+\frac{\alpha_2^2\eta^2\sigma^2}{1-\lambda}  \\
			& \quad +\frac{2\alpha_2^2\eta^2L_{g_y}^2}{1-\lambda}\frac{1}{K}\mathbb{E}[\| X_{t}- X_{t-1}\|_F^2]+\frac{2\alpha_2^2\eta^2L_{g_y}^2}{1-\lambda}\frac{1}{K}\mathbb{E}[\| Y_{t}- Y_{t-1}\|_F^2] \ , 
		\end{aligned}
	\end{equation}
	where $(s_1)$ holds due to Lemma~\ref{lemma_ineqality} with $a=\frac{1-\lambda}{\lambda}$, $(s_2)$ holds due to Assumption~\ref{assumption_variance}, $(s_3)$ holds due to  Assumption~\ref{assumption_lower_smooth}. 
	
\end{proof}

\begin{lemma} \label{lemma_consensus_x}
	Given Assumptions~\ref{assumption_graph}-\ref{assumption_lower_smooth},   the following inequality holds. 
	\begin{equation}
		\begin{aligned}
			& \quad\mathbb{E}[\|X_{t+1} - \bar{X}_{t+1}\|_F^2] \leq   \Big(1-\frac{\eta(1-\lambda^2)}{2}\Big)\mathbb{E}[\|X_t -\bar{X}_t\|_F^2]+ \frac{2\eta \beta_1^2}{1-\lambda^2}\mathbb{E}[\|Z_t^{\tilde{F}} - \bar{ Z}_t^{\tilde{F}}\|_F^2]\ .\\ 
		\end{aligned}
	\end{equation}
\end{lemma}

\begin{proof}
	\begin{equation}
		\begin{aligned}
			&  \quad \mathbb{E}[\|X_{t+1} - \bar{X}_{t+1}\|_F^2 ]\\
			& \overset{(s_0)}= \mathbb{E}[\|X_{t}-\eta X_{t} (I-W) - \beta_1\eta Z_t^{\tilde{F}} - \bar{X}_{t}+\beta_1\eta \bar{ Z}_t^{\tilde{F}} \|_F^2]\\
			& = \mathbb{E}[\|(1-\eta)(X_{t}- \bar{X}_{t})+ \eta (X_{t} W -  \bar{X}_{t} - \beta_1 Z_t^{\tilde{F}}+\beta_1 \bar{ Z}_t^{\tilde{F}} )\|_F^2]\\
			& \overset{(s_1)}\leq (1-\eta)\mathbb{E}[\|X_t -\bar{X}_t\|_F^2]+ \eta\mathbb{E}[\|X_{t}W + \beta_1 Z_t^{\tilde{F}} -\bar{X}_{t} - \beta_1\bar{ Z}_t^{\tilde{F}} \|_F^2]\\ 
			& \overset{(s_2)}\leq (1-\eta)\mathbb{E}[\|X_t -\bar{X}_t\|_F^2]+ \frac{\eta(1+\lambda^2)}{2\lambda^2}\mathbb{E}[\|X_{t}W  -\bar{X}_{t} \|_F^2]+ \frac{2\eta \beta_1^2}{1-\lambda^2}\mathbb{E}[\|Z_t^{\tilde{F}} - \bar{ Z}_t^{\tilde{F}}\|_F^2]\\ 
			& \overset{(s_3)}\leq \Big(1-\frac{\eta(1-\lambda^2)}{2}\Big)\mathbb{E}[\|X_t -\bar{X}_t\|_F^2]+ \frac{2\eta \beta_1^2}{1-\lambda^2}\mathbb{E}[\|Z_t^{\tilde{F}} - \bar{ Z}_t^{\tilde{F}}\|_F^2] \ ,\\ 
		\end{aligned}
	\end{equation}
	where $(s_0)$ holds due to $X_{t} (I-W)\mathbf{1}=0$,  $(s_1)$ holds due to Lemma~\ref{lemma_ineqality} with $a=\frac{\eta}{1-\eta}$,  $(s_2)$ holds due to Lemma~\ref{lemma_ineqality} with $a=\frac{1-\lambda^2}{2\lambda^2}$, 
	$(s_3)$ holds due to $\|X_{t}W  -\bar{X}_{t} \|_F^2\leq \lambda^2 \|X_{t}  -\bar{X}_{t} \|_F^2$. 
\end{proof}

\begin{lemma} \label{lemma_consensus_y}
	Given Assumptions~\ref{assumption_graph}-\ref{assumption_lower_smooth},   the following inequality holds. 
	\begin{equation}
		\begin{aligned}
			&  \mathbb{E}[\|Y_{t+1} - \bar{Y}_{t+1}\|_F^2 ]\leq \Big(1-\frac{\eta(1-\lambda^2)}{2}\Big) \mathbb{E}[\|Y_t -\bar{Y}_t\|_F^2]+ \frac{2\eta \beta_2^2}{1-\lambda^2}\mathbb{E}[ \|Z_t^{g} - \bar{Z}_t^{g}\|_F^2]\ .\\ 
		\end{aligned}
	\end{equation}
\end{lemma}

\begin{proof}
	In terms of the definition of  ${x}_{t+1}^{(k)}$,  we have
	\begin{equation}
		\begin{aligned}
			&  \quad \mathbb{E}[\|Y_{t+1} - \bar{Y}_{t+1}\|_F^2 ]\\
			& \overset{(s_0)}=\mathbb{E}[ \|Y_{t}-\eta Y_{t} (I-W) - \beta_2\eta Z_t^{g}- \bar{Y}_{t}+\beta_2\eta \bar{Z}_t^{g} \|_F^2]\\
			& = \mathbb{E}[\|(1-\eta)(Y_{t}- \bar{Y}_{t})+ \eta (Y_{t} W -  \bar{Y}_{t} - \beta_2 Z_t^{g}+\beta_2 \bar{Z}_t^{g} )\|_F^2]\\
			& \overset{(s_1)}\leq(1-\eta)\mathbb{E}[\|Y_t -\bar{Y}_t\|_F^2]+ \eta\mathbb{E}[\|Y_{t}W + \beta_2 Z_t^{g} -\bar{Y}_{t} - \beta_2\bar{Z}_t^{g} \|_F^2]\\ 
			& \overset{(s_2)}\leq (1-\eta)\mathbb{E}[\|Y_t -\bar{Y}_t\|_F^2]+ \frac{\eta(1+\lambda^2)}{2\lambda^2}\mathbb{E}[\|Y_{t}W  -\bar{Y}_{t} \|_F^2]+ \frac{2\eta \beta_2^2}{1-\lambda^2}\mathbb{E}[\|Z_t^{g} - \bar{Z}_t^{g}\|_F^2]\\ 
			& \overset{(s_3)}\leq \Big(1-\frac{\eta(1-\lambda^2)}{2}\Big)\mathbb{E}[\|Y_t -\bar{Y}_t\|_F^2]+ \frac{2\eta \beta_2^2}{1-\lambda^2}\mathbb{E}[\|Z_t^{g} - \bar{Z}_t^{g}\|_F^2] \ ,\\ 
		\end{aligned}
	\end{equation}
	where $(s_0)$ holds due to $Y_{t} (I-W)\mathbf{1}=0$,  $(s_1)$ holds due to Lemma~\ref{lemma_ineqality} with $a=\frac{\eta}{1-\eta}$,  $(s_2)$ holds due to Lemma~\ref{lemma_ineqality} with $a=\frac{1-\lambda^2}{2\lambda^2}$,  $(s_3)$ holds due to  $\|Y_{t}W  -\bar{Y}_{t} \|_F^2\leq \lambda^2 \|Y_{t}  -\bar{Y}_{t} \|_F^2$. 
	
\end{proof}

\begin{lemma} \label{lemma_incremental_x}
	Given Assumptions~\ref{assumption_graph}-\ref{assumption_lower_smooth},   the following inequality holds. 
	\begin{equation}
		\begin{aligned}
			& \quad \mathbb{E}[\|X_{t+1}-X_{t}\|_F^2]\leq  8\eta^2\mathbb{E}[\|X_{t}-\bar{X}_t\|_F^2] + 4\eta^2\beta_1^2\mathbb{E}[\|Z^{\tilde{F}}_t-\bar{Z}^{\tilde{F}}_t\|_F^2 ]+4\eta^2\beta_1^2\mathbb{E}[\|\bar{Z}^{\tilde{F}}_t\|_F^2] \ . \\
		\end{aligned}
	\end{equation}
\end{lemma}

\begin{proof}
	\begin{equation}
		\begin{aligned}
			& \quad  \mathbb{E}[\|X_{t+1}-X_{t}\|_F^2]\\
			& = \mathbb{E}[\|X_{t}-\eta X_{t} (I-W) - \beta_1\eta Z^{\tilde{F}}_t-X_{t}\|_F^2]\\
			& = \mathbb{E}[\|\eta X_{t}(W-I) -  \beta_1\eta Z^{\tilde{F}}_t \|_F^2 ]\\
			& \leq  2\eta^2\mathbb{E}[\|X_{t}(W-I)\|_F^2] + 2\eta^2\beta_1^2\mathbb{E}[\|Z^{\tilde{F}}_t\|_F^2 ]\\
			& =  2\eta^2\mathbb{E}[\|(X_{t}-\bar{X}_t)(W-I)\|_F^2]+ 2\eta^2\beta_1^2\mathbb{E}[\|Z^{\tilde{F}}_t-\bar{Z}^{\tilde{F}}_t+\bar{Z}^{\tilde{F}}_t\|_F^2]\\
			& \overset{(s_1)}\leq 8\eta^2\mathbb{E}[\|X_{t}-\bar{X}_t\|_F^2] + 4\eta^2\beta_1^2\mathbb{E}[\|Z^{\tilde{F}}_t-\bar{Z}^{\tilde{F}}_t\|_F^2]+4\eta^2\beta_1^2\mathbb{E}[\|\bar{Z}^{\tilde{F}}_t\|_F^2] \ , \\
		\end{aligned}
	\end{equation}
	where $(s_1)$ holds due to $\|AB\|_F\leq \|A\|_2\|B\|_F$ and  $\|I-W\|_2\leq 2$.

\end{proof}

\begin{lemma} \label{lemma_incremental_y}
	Given Assumptions~\ref{assumption_graph}-\ref{assumption_lower_smooth},   the following inequality holds. 
	\begin{equation}
		\begin{aligned}
			& \quad \mathbb{E}[\|Y_{t+1}-Y_{t}\|_F^2]\leq 8\eta^2\mathbb{E}[\|Y_{t}-\bar{Y}_t\|_F^2] + 4\eta^2\beta_2^2\mathbb{E}[\|Z^{g}_t-\bar{Z}^{g}_t\|_F^2] +4\eta^2\beta_2^2\mathbb{E}[\|\bar{Z}^{g}_t\|_F^2] \ . \\
		\end{aligned}
	\end{equation}
\end{lemma}

\begin{proof}
	In terms of the definition of $\tilde{{x}}_{t+1}^{(k)}$, we can get 
	\begin{equation}
		\begin{aligned}
			& \quad\mathbb{E}[\|Y_{t+1}-Y_{t}\|_F^2]\\
			& =  \mathbb{E}[\|Y_{t}-\eta Y_{t} (I-W) - \beta_2\eta Z^{g}_t-Y_{t}\|_F^2]\\
			& = \mathbb{E}[\|\eta Y_{t}(W-I) -  \beta_2\eta Z^{g}_t \|_F^2] \\
			& \leq  2\eta^2\mathbb{E}[\|Y_{t}(W-I)\|_F^2] + 2\eta^2\beta_2^2\mathbb{E}[\|Z^{g}_t\|_F^2] \\
			& =  2\eta^2\mathbb{E}[\|(Y_{t}-\bar{Y}_t)(W-I)\|_F^2 ]+ 2\eta^2\beta_2^2\mathbb{E}[\|Z^{g}_t-\bar{Z}^{g}_t+\bar{Z}^{g}_t\|_F^2 ]\\
			& \overset{(s_1)}\leq 8\eta^2\mathbb{E}[\|Y_{t}-\bar{Y}_t\|_F^2 ]+ 4\eta^2\beta_2^2\mathbb{E}[\|Z^{g}_t-\bar{Z}^{g}_t\|_F^2] +4\eta^2\beta_2^2\mathbb{E}[\|\bar{Z}^{g}_t\|_F^2] \ , \\
		\end{aligned}
	\end{equation}
	where $(s_1)$ holds due to $\|AB\|_F\leq \|A\|_2\|B\|_F$ and  $\|I-W\|_2\leq 2$. 
\end{proof}

\subsubsection{Proof of Theorem~\ref{theorem_mdbo}}

\begin{proof}

	According to the smoothness of $F(x)$, we can get
	\begin{equation} \label{lemma_F_incremental}
		\begin{aligned}
			& \mathbb{E}[F(\bar{  {x}}_{t+1})]\leq \mathbb{E}[F(\bar{  {x}}_{t}) ]+ \mathbb{E}[\langle \nabla F(\bar{  {x}}_{t}), \bar{  {x}}_{t+1} - \bar{  {x}}_{t}\rangle ] + \frac{L_{F}^{*}}{2} \mathbb{E}[\|\bar{  {x}}_{t+1} - \bar{  {x}}_{t}\|^2 ]\\
			& = \mathbb{E}[F(\bar{  {x}}_{t})] - \eta\beta_1\mathbb{E}[\langle \nabla F(\bar{  {x}}_{t}), \bar{  {u}}_t\rangle  ]+ \frac{\eta^2\beta_1^2L_{F}^{*}}{2} \mathbb{E}[\|\bar{  {u}}_t\|^2] \\
			& = \mathbb{E}[F(\bar{  {x}}_{t})] - \frac{\eta\beta_1}{2} \mathbb{E}[\|\nabla F(\bar{  {x}}_{t})\|^2]+ \Big(\frac{\eta^2\beta_1^2L_{F}^{*}}{2}-  \frac{\eta\beta_1}{2}\Big) \mathbb{E}[\|\bar{  {u}}_t\|^2] + \frac{\eta\beta_1}{2} \mathbb{E}[\|\nabla F(\bar{  {x}}_{t}) -\bar{  {u}}_t\|^2] \\
			& \overset{(s_1)}\leq \mathbb{E}[F(\bar{  {x}}_{t}) ]- \frac{\eta\beta_1}{2} \mathbb{E}[\|\nabla F(\bar{  {x}}_{t})\|^2]-  \frac{\eta\beta_1}{4} \mathbb{E}[\|\bar{  {u}}_t\|^2]  + \eta\beta_1{L}_{F}^2\mathbb{E}[\|\bar{y}_{t} - y^*(\bar{x}_{t})\|^2]\\
			& \quad  +  \frac{3\eta\beta_1C_{g_{xy}}^2C_{f_y}^2}{\mu^2}(1-\frac{\mu}{L_{g_{y}}})^{2J}+\frac{3\eta\beta_1{L}_{\tilde{F}}^2}{K}\mathbb{E}[\|X_{t}-\bar{X}_t\|_F^2]+ \frac{3\eta\beta_1{L}_{\tilde{F}}^2}{K}\mathbb{E}[\|Y_{t}-\bar{Y}_t\|_F^2]\\
			& \quad + 3\eta\beta_1\frac{1}{K}\mathbb{E}[\| \Delta_t^{\tilde{F}}-  U_t \|_F^2 ]\ , \\
		\end{aligned}
	\end{equation}
	where  $(s_1)$ holds due to $ \eta\leq \frac{1}{2\beta_1L_{F}^{*}}$ and the following inequality: 
	\begin{equation}
		\begin{aligned}
			& \quad \mathbb{E}[\|\nabla F(\bar{  {x}}_{t}) -\bar{  {u}}_t\|^2]\\
			& \leq 2\mathbb{E}[\|\nabla F(\bar{  {x}}_{t}) - {\nabla} F(\bar{x}_t, \bar{y}_t) \|^2 ] + 2\mathbb{E}[\|{\nabla} F(\bar{x}_t, \bar{y}_t)   - {\nabla} \tilde{F}(\bar{  {x}}_{t}, \bar{  {y}}_{t})+ (\underline{\Delta}_t^{\tilde{F}}- \Delta_t^{\tilde{F}}+ \Delta_t^{\tilde{F}}-  U_t) \frac{1}{K}\mathbf{1}\|^2]\\
			& \leq 2\mathbb{E}[\|\nabla F(\bar{  {x}}_{t})  - {\nabla} F(\bar{  {x}}_{t}, \bar{  {y}}_{t})\|^2] + 6\mathbb{E}[\|{\nabla} F(\bar{  {x}}_{t}, \bar{  {y}}_{t})  - {\nabla} \tilde{F}(\bar{  {x}}_{t}, \bar{  {y}}_{t})\|^2]  + 6\mathbb{E}\Big[\Big\|(\underline{\Delta}_t^{\tilde{F}}- \Delta_t^{\tilde{F}})\frac{1}{K}\mathbf{1}\Big\|^2\Big]\\
			& \quad +6\mathbb{E}\Big[\Big\| (\Delta_t^{\tilde{F}}-  U_t )\frac{1}{K}\mathbf{1}\Big\|^2\Big]\\
			& \overset{(s_1)}\leq 2{L}_{F}^2\mathbb{E}[\|\bar{y}_{t} - y^*(\bar{x}_{t})\|^2] + \frac{6C_{g_{xy}}^2C_{f_y}^2}{\mu^2}(1-\frac{\mu}{L_{g_{y}}})^{2J} +\frac{6{L}_{\tilde{F}}^2}{K}\mathbb{E}[ \|X_{t}-\bar{X}_t\|_F^2]+\frac{6{L}_{\tilde{F}}^2}{K} \mathbb{E}[\|Y_{t}-\bar{Y}_t\|_F^2]\\
			& \quad +6\frac{1}{K}\mathbb{E}[\| \Delta_t^{\tilde{F}}-  U_t \|_F^2] \ , \\
		\end{aligned}
	\end{equation}
	where $(s_1)$  holds due to  Lemma~\ref{lemma_hypergrad_bias} and Lemma~\ref{lemma_hypergrad_smooth}.  
	Additionally, we can bound $\mathbb{E}[\|\bar{y}_{t} - y^*(\bar{x}_{t})\|^2]$ as follows:
	\begin{equation} \label{eq_y}
		\begin{aligned}
			& \quad  \mathbb{E}[\|\bar{   {y}}_{t+1} -    {y}^{*}(\bar{   {x}}_{t+1})\| ^2 ] \\
			& \overset{(s_1)} \leq  (1-\frac{\beta_2\eta\mu}{4}) \mathbb{E}[\|\bar{   {y}}_{t}   -    {y}^{*}(\bar{   {x}}_t)\| ^2] - \frac{3\eta\beta_2^2}{4} \mathbb{E}[\|\bar{v}_{t}  \|^2]   +\frac{25\eta\beta_1^2L_{y}^2 }{6\beta_2\mu} \mathbb{E}[\|\bar{u}_{t}\| ^2 ]+  \frac{25\beta_2 \eta}{6\mu}  \mathbb{E}[\|(\underline{\Delta}_{t}^{g}- V_t)\frac{1}{K}\mathbf{1}\|^2]   \\
			& \overset{(s_2)}\leq  (1-\frac{\beta_2\eta\mu}{4}) \mathbb{E}[\|\bar{   {y}}_{t}   -    {y}^{*}(\bar{   {x}}_t)\| ^2] - \frac{3\eta\beta_2^2}{4} \mathbb{E}[\|\bar{v}_{t}  \|^2]  +\frac{25\eta\beta_1^2L_{y}^2 }{6\beta_2\mu} \mathbb{E}[\|\bar{u}_{t}\| ^2]\\
			& \quad +  \frac{25\beta_2 \eta L_{g_y}^2}{3\mu} \frac{1}{K}\mathbb{E}[\|X_t - \bar{X}_t\|_F^2] +  \frac{25\beta_2 \eta L_{g_y}^2}{3\mu} \frac{1}{K}\mathbb{E}[\|Y_t - \bar{Y}_t\|_F^2 ]+  \frac{25\beta_2 \eta}{3\mu}\frac{1}{K} \mathbb{E}[\|\Delta_{t}^{g} - V_t\|_F^2]  \ ,  \\
		\end{aligned}
	\end{equation}
	where $(s_1)$ holds due to Lemma 9 in \cite{yang2021provably} with $\beta_2<\frac{1}{6L_{g_y}}$ and $0<\eta<1$, $(s_2)$  holds due to the following inequality. 
	\begin{equation}
		\begin{aligned}
			&\quad  \mathbb{E}[\|(\underline{\Delta}_{t}^{g}- V_t)\frac{1}{K}\mathbf{1}\|^2] \\
			& = \mathbb{E}[\|(\underline{\Delta}_{t}^{g}-\Delta_{t}^{g} + \Delta_{t}^{g} - V_t)\frac{1}{K}\mathbf{1}\|^2 ]\\
			& \leq 2\mathbb{E}[\|(\underline{\Delta}_{t}^{g}-\Delta_{t}^{g})\frac{1}{K}\mathbf{1}\|^2 ]+ 2\mathbb{E}[\|(\Delta_{t}^{g} - V_t)\frac{1}{K}\mathbf{1}\|^2 ]\\
			& \leq \frac{2L_{g_y}^2}{K}\mathbb{E}[ \|X_t - \bar{X}_t\|_F^2 ]+ \frac{2L_{g_y}^2}{K} \mathbb{E}[\|Y_t - \bar{Y}_t\|_F^2] + 2\frac{1}{K}\mathbb{E}[\|\Delta_{t}^{g} - V_t\|_F^2]  \ . \\
		\end{aligned}
	\end{equation}
	
	Then, to investigate the convergence rate of Algorithm~\ref{alg_MDBO}, we introduce the following potential function:
\begin{equation}
	\begin{aligned}
		& \mathcal{L}_{t+1} ={ \mathbb{E}}[F(x_{t+1})] +\frac{6\beta_1{L}_{F}^2}{\beta_2\mu}\mathbb{E}[\|\bar{   {y}}_{t+1} -    {y}^{*}(\bar{   {x}}_{t+1})\| ^2 ]  \\
		& \quad  +   \frac{ 2\beta_1L_{g_y}^2((11+32/\alpha_1^2){L}_{\tilde{F}}^2+(450+800/\alpha_2^2){L}_{F}^2)}{\mu^2(1-\lambda^2)}\frac{1}{K} \mathbb{E}[\|X_{t+1} - \bar{X}_{t+1}\|_F^2 ] \\
		& \quad + \frac{ 2\beta_1L_{g_y}^2((11+32/\alpha_1^2){L}_{\tilde{F}}^2+(450+800/\alpha_2^2){L}_{F}^2)}{\mu^2(1-\lambda^2)}  \frac{1}{K}\mathbb{E}[\|Y_{t+1} - \bar{Y}_{t+1}\|_F^2 ]  \\
		& \quad +  \frac{\beta_1(1-\lambda)}{2\alpha_1} \frac{1}{K}\mathbb{E}[\| Z^{\tilde{F}}_{t+1}-\bar{Z}^{\tilde{F}}_{t+1} \|_F^2]  + \frac{25 (1-\lambda)\beta_1{L}_{F}^2}{\alpha_2\mu^2} \frac{1}{K}\mathbb{E}[\| Z^{g}_{t+1}-\bar{Z}^{g}_{t+1} \|_F^2] \\
		& \quad  +\frac{4\beta_1}{\alpha_1}\frac{1}{K}\mathbb{E}[\|\Delta_{t+1}^{\tilde{F}}-U_{t+1}\|_F^2] + \frac{100 \beta_1{L}_{F}^2}{\alpha_2\mu^2}\frac{1}{K}\mathbb{E}[\|\Delta_{t+1}^{g}-V_{t+1}\|_F^2] \ . \\
	\end{aligned}
\end{equation}

Based on Lemmas~\ref{lemma_hyper_momentum_var},~\ref{lemma_lower_momentum_var},~\ref{lemma_consensus_z_f},~\ref{lemma_consensus_z_g},~\ref{lemma_consensus_x},~\ref{lemma_consensus_y},~\ref{lemma_incremental_x},~\ref{lemma_incremental_y}, we can get 
\begin{equation}
	\begin{aligned}
		& \quad   \mathcal{L}_{t+1} - \mathcal{L}_{t}  \\
&\leq  - \frac{\eta\beta_1}{2} \mathbb{E}[\|\nabla F(\bar{  {x}}_{t})\|^2]  - \frac{\eta\beta_1L_F^2}{2}\mathbb{E}[\|\bar{y}_{t} - y^*(\bar{x}_{t})\|^2] +  \frac{3\eta\beta_1C_{g_{xy}}^2C_{f_y}^2}{\mu^2}(1-\frac{\mu}{L_{g_{y}}})^{2J}\\
& \quad + (\frac{25\eta\beta_1^3{L}_{F}^2L_{g_y}^2 }{\beta_2^2\mu^2}  +4\alpha_1\beta_1^3\eta^4L_{\tilde{F}}^2+\frac{200\alpha_2\beta_1^3\eta^4L_{g_y}^2{L}_{F}^2}{\mu^2}+\frac{16\eta\beta_1^3{L}_{\tilde{F}}^2}{\alpha_1^2}+\frac{400\eta\beta_1^3{L}_{g_{y}}^2{L}_{F}^2}{\alpha_2^2\mu^2}-  \frac{\eta\beta_1}{4} ) \mathbb{E}[\|\bar{  {u}}_t\|^2]  \\
& \quad +( 4\alpha_1\beta_1\beta_2^2\eta^4L_{\tilde{F}}^2+\frac{200\alpha_2\beta_1\beta_2^2\eta^4L_{g_y}^2{L}_{F}^2}{\mu^2} +\frac{16\eta\beta_1\beta_2^2{L}_{\tilde{F}}^2}{\alpha_1^2} +\frac{400 \eta\beta_1\beta_2^2{L}_{g_y}^2{L}_{F}^2}{\alpha_2^2\mu^2}-  \frac{9\eta\beta_1\beta_2{L}_{F}^2}{2\mu})\mathbb{E}[\|\bar{v}_{t}  \|^2]  \\
& \quad +(\frac{ 4\eta \beta_1^3L_{g_y}^2((11+32/\alpha_1^2){L}_{\tilde{F}}^2+(450+800/\alpha_2^2){L}_{F}^2)}{\mu^2(1-\lambda^2)^2} -  \frac{\beta_1(1-\lambda)^2}{2\alpha_1}+ 4\alpha_1\beta_1^3\eta^4L_{\tilde{F}}^2 \\
& \quad \quad + \frac{200\alpha_2\beta_1^3\eta^4L_{g_y}^2{L}_{F}^2}{\mu^2}+ \frac{16\eta\beta_1^3{L}_{\tilde{F}}^2}{\alpha_1^2}+ \frac{400\eta\beta_1^3{L}_{g_{y}}^2 {L}_{F}^2}{\alpha_2^2\mu^2})\frac{1}{K}\mathbb{E}[\|Z^{\tilde{F}}_t - \bar{ Z}^{\tilde{F}}_t\|_F^2] \\ 
& \quad+ (\frac{ 4\eta \beta_1\beta_2^2L_{g_y}^2((11+32/\alpha_1^2){L}_{\tilde{F}}^2+(450+800/\alpha_2^2){L}_{F}^2)}{\mu^2(1-\lambda^2)^2} -  \frac{25 (1-\lambda)^2\beta_1{L}_{F}^2}{\alpha_2\mu^2}+ 4\alpha_1\beta_1\beta_2^2\eta^4L_{\tilde{F}}^2\\
& \quad \quad +\frac{200 \alpha_2\beta_1\beta_2^2\eta^4L_{g_y}^2{L}_{F}^2}{\mu^2} + \frac{16\eta\beta_1\beta_2^2{L}_{\tilde{F}}^2}{\alpha_1^2}+ \frac{400\eta\beta_1\beta_2^2{L}_{g_y}^2 {L}_{F}^2}{\alpha_2^2\mu^2})\frac{1}{K}\mathbb{E}[ \|Z^{g}_t - \bar{Z}^{g}_t\|_F^2] \\ 
& \quad +\frac{\alpha_1\beta_1\eta^2\sigma_{\tilde{F}}^2}{2}+\frac{25\beta_1\alpha_2\eta^2 {L}_{F}^2\sigma^2}{\mu^2}+ 4\beta_1\alpha_1\eta^2 \sigma_{\tilde{F}}^2  + \frac{100 \beta_1\alpha_2\eta^2{L}_{F}^2 \sigma^2}{\mu^2} \ . \\
	\end{aligned}
\end{equation}

By setting the coefficient of $\mathbb{E}[\|\bar{  {u}}_t\|^2]$ to be non-positive, we can get
\begin{equation}
	\begin{aligned}
		& \frac{25\eta\beta_1^3{L}_{F}^2L_{g_y}^2 }{\beta_2^2\mu^2}  +4\alpha_1\beta_1^3\eta^4L_{\tilde{F}}^2+\frac{200\alpha_2\beta_1^3\eta^4L_{g_y}^2{L}_{F}^2}{\mu^2}+\frac{16\eta\beta_1^3{L}_{\tilde{F}}^2}{\alpha_1^2}+\frac{400\eta\beta_1^3{L}_{g_{y}}^2{L}_{F}^2}{\alpha_2^2\mu^2}-  \frac{\eta\beta_1}{4} \leq  0  \ , \\
		& \frac{25\beta_1^2L_{y}^2 {L}_{F}^2}{\beta_2^2\mu^2} + 4\alpha_1\beta_1^2\eta^3L_{\tilde{F}}^2+\frac{200\alpha_2\beta_1^2\eta^3L_{g_y}^2{L}_{F}^2}{\mu^2}+\frac{16\beta_1^2{L}_{\tilde{F}}^2}{\alpha_1^2}+\frac{400 \beta_1^2{L}_{g_{y}}^2{L}_{F}^2}{\alpha_2^2\mu^2}-  \frac{1}{4}  \leq 0  \ . \\
	\end{aligned}
\end{equation}
By setting $\beta_1  \leq \frac{\beta_2\mu} {15L_{y} {L}_{F}} $, we can get $\frac{25\beta_1^2L_{y}^2 {L}_{F}^2}{\beta_2^2\mu^2} -  \frac{1}{4}  \leq -  \frac{1}{8}$.  Because $\alpha_1\eta<1, \alpha_2\eta<1, \eta<1, \frac{L_{g_y}^2}{\mu^2}>1$,  we can get
\begin{equation}
	\begin{aligned}
		& \quad \frac{25\beta_1^2L_{y}^2 {L}_{F}^2}{\beta_2^2\mu^2} + 4\alpha_1\beta_1^2\eta^3L_{\tilde{F}}^2+\frac{200\alpha_2\beta_1^2\eta^3L_{g_y}^2{L}_{F}^2}{\mu^2}+\frac{16\beta_1^2{L}_{\tilde{F}}^2}{\alpha_1^2}+\frac{400 \beta_1^2{L}_{g_{y}}^2{L}_{F}^2}{\alpha_2^2\mu^2}-  \frac{1}{4}  \\
		& \leq  \frac{25\beta_1^2L_{y}^2 {L}_{F}^2}{\beta_2^2\mu^2} + \frac{4\beta_1^2L_{\tilde{F}}^2L_{g_y}^2}{\mu^2}+\frac{200\beta_1^2L_{g_y}^2{L}_{F}^2}{\mu^2}+\frac{16\beta_1^2{L}_{\tilde{F}}^2}{\alpha_1^2}\frac{L_{g_y}^2}{\mu^2}+\frac{400 \beta_1^2{L}_{g_{y}}^2{L}_{F}^2}{\alpha_2^2\mu^2}-  \frac{1}{4} \\
		&\leq  \frac{4\beta_1^2L_{\tilde{F}}^2L_{g_y}^2}{\mu^2}+\frac{200\beta_1^2L_{g_y}^2{L}_{F}^2}{\mu^2}+\frac{16\beta_1^2{L}_{\tilde{F}}^2}{\alpha_1^2}\frac{L_{g_y}^2}{\mu^2}+\frac{400 \beta_1^2{L}_{g_{y}}^2{L}_{F}^2}{\alpha_2^2\mu^2}-  \frac{1}{8}  \ .  \\
	\end{aligned}
\end{equation}
By letting this upper bound non-positive, we can get
\begin{equation}
	\begin{aligned}
		&  \frac{4\beta_1^2L_{\tilde{F}}^2L_{g_y}^2}{\mu^2}+\frac{200\beta_1^2L_{g_y}^2{L}_{F}^2}{\mu^2}+\frac{16\beta_1^2{L}_{\tilde{F}}^2}{\alpha_1^2}\frac{L_{g_y}^2}{\mu^2}+\frac{400 \beta_1^2{L}_{g_{y}}^2{L}_{F}^2}{\alpha_2^2\mu^2}-  \frac{1}{8}  \leq 0  \ , \\
		&  \frac{(4+16/\alpha_1^2)L_{\tilde{F}}^2L_{g_y}^2}{\mu^2}\beta_1^2+\frac{(200+400/\alpha_2^2){L}_{F}^2L_{g_y}^2}{\mu^2}\beta_1^2-  \frac{1}{8}  \leq 0  \ , \\
		& \beta_1^2 \leq  \frac{1}{8}\frac{\mu^2}{((4+16/\alpha_1^2)L_{\tilde{F}}^2+(200+400/\alpha_2^2){L}_{F}^2)L_{g_y}^2}  \ , \\
		& \beta_1 \leq  \frac{\mu}{4L_{g_y}\sqrt{((2+8/\alpha_1^2)L_{\tilde{F}}^2+(100+200/\alpha_2^2){L}_{F}^2)}} \ . 
	\end{aligned}
\end{equation}
Therefore, by setting $\beta_1\leq \min\Big\{\frac{\beta_2\mu} {15L_{y} {L}_{F}}, \frac{\mu}{4L_{g_y}\sqrt{((2+8/\alpha_1^2)L_{\tilde{F}}^2+(100+200/\alpha_2^2){L}_{F}^2)}}\Big\}$, the coefficient of $\mathbb{E}[\|\bar{  {u}}_t\|^2]$ is non-positive. 

By setting the coefficient of $\mathbb{E}[\|\bar{  {v}}_t\|^2]$ to be non-positive, we can get
\begin{equation}
	\begin{aligned}
& 4\alpha_1\beta_1\beta_2^2\eta^4L_{\tilde{F}}^2+\frac{200\alpha_2\beta_1\beta_2^2\eta^4L_{g_y}^2{L}_{F}^2}{\mu^2} +\frac{16\eta\beta_1\beta_2^2{L}_{\tilde{F}}^2}{\alpha_1^2} +\frac{400 \eta\beta_1\beta_2^2{L}_{g_y}^2{L}_{F}^2}{\alpha_2^2\mu^2}-  \frac{9\eta\beta_1\beta_2{L}_{F}^2}{2\mu} \leq  0 \ ,  \\
		& 4\alpha_1\eta^3L_{\tilde{F}}^2+\frac{200 \alpha_2\eta^3L_{g_y}^2{L}_{F}^2}{\mu^2} +\frac{16{L}_{\tilde{F}}^2}{\alpha_1^2}  + \frac{400{L}_{g_y}^2 {L}_{F}^2}{\alpha_2^2\mu^2} - \frac{9{L}_{F}^2}{2\beta_2\mu} \leq  0  \ . \\
	\end{aligned}
\end{equation}
Because $\alpha_1\eta<1, \alpha_2\eta<1, \eta<1, \frac{L_{g_y}^2}{\mu^2}>1$,  we can get
\begin{equation}
	\begin{aligned}
		& \quad 4\alpha_1\eta^3L_{\tilde{F}}^2+\frac{200 \alpha_2\eta^3L_{g_y}^2{L}_{F}^2}{\mu^2} +\frac{16{L}_{\tilde{F}}^2}{\alpha_1^2}  + \frac{400{L}_{g_y}^2 {L}_{F}^2}{\alpha_2^2\mu^2} - \frac{9{L}_{F}^2}{2\beta_2\mu}  \\
		&\leq  \frac{4L_{\tilde{F}}^2L_{g_y}^2}{\mu^2}+\frac{200L_{g_y}^2{L}_{F}^2}{\mu^2} +\frac{16{L}_{\tilde{F}}^2L_{g_y}^2}{\alpha_1^2\mu^2}  + \frac{400{L}_{g_y}^2 {L}_{F}^2}{\alpha_2^2\mu^2} - \frac{9{L}_{F}^2}{2\beta_2\mu} \\
		& =  \frac{(4+16/\alpha_1^2)L_{\tilde{F}}^2L_{g_y}^2}{\mu^2}+\frac{(200+400/\alpha_2^2)L_{g_y}^2{L}_{F}^2}{\mu^2}  - \frac{9{L}_{F}^2}{2\beta_2\mu}   \ . \\
	\end{aligned}
\end{equation}
By letting this upper bound non-positive, we can get
\begin{equation}
	\begin{aligned} 
		& \frac{(4+16/\alpha_1^2)L_{\tilde{F}}^2L_{g_y}^2}{\mu^2}+\frac{(200+400/\alpha_2^2)L_{g_y}^2{L}_{F}^2}{\mu^2}  - \frac{9{L}_{F}^2}{2\beta_2\mu} \leq  0 \ ,  \\
		& \beta_2 \leq \frac{9\mu{L}_{F}^2}{2((4+16/\alpha_1^2)L_{\tilde{F}}^2+(200+400/\alpha_2^2){L}_{F}^2)L_{g_y}^2} \ . \\
	\end{aligned}
\end{equation}

By setting the coefficient of $\mathbb{E}[\|Z^{\tilde{F}}_t - \bar{ Z}^{\tilde{F}}_t\|_F^2]$ to be non-positive, we can get
\begin{equation}
	\begin{aligned}
		& \frac{ 4\eta \beta_1^3L_{g_y}^2((11+32/\alpha_1^2){L}_{\tilde{F}}^2+(450+800/\alpha_2^2){L}_{F}^2)}{\mu^2(1-\lambda^2)^2} -  \frac{\beta_1(1-\lambda)^2}{2\alpha_1}+ 4\alpha_1\beta_1^3\eta^4L_{\tilde{F}}^2 \\
		& \quad \quad + \frac{200\alpha_2\beta_1^3\eta^4L_{g_y}^2{L}_{F}^2}{\mu^2}+ \frac{16\eta\beta_1^3{L}_{\tilde{F}}^2}{\alpha_1^2}+ \frac{400\eta\beta_1^3{L}_{g_{y}}^2 {L}_{F}^2}{\alpha_2^2\mu^2} \leq 0  \ , \\
		& \beta_1^2\frac{ 4L_{g_y}^2((11+32/\alpha_1^2){L}_{\tilde{F}}^2+(450+800/\alpha_2^2){L}_{F}^2)}{\mu^2(1-\lambda^2)^2} - \frac{(1-\lambda)^2}{2\eta\alpha_1}+ 4\beta_1^2\alpha_1\eta^3L_{\tilde{F}}^2 \\
		& \quad + \frac{200 \alpha_2\eta^3L_{g_y}^2{L}_{F}^2}{\mu^2}\beta_1^2+ \frac{16\beta_1^2{L}_{\tilde{F}}^2}{\alpha_1^2}+ \frac{400\beta_1^2{L}_{g_{y}}^2{L}_{F}^2}{\alpha_2^2\mu^2}\leq  0  \ . \\
	\end{aligned}
\end{equation}
Because $\alpha_1\eta<1, \alpha_2\eta<1, \eta<1, \frac{L_{g_y}^2}{\mu^2}>1, 1-\lambda^2<1$,  we can get
\begin{equation}
	\begin{aligned}
		& \quad \beta_1^2\frac{ 4L_{g_y}^2((11+32/\alpha_1^2){L}_{\tilde{F}}^2+(450+800/\alpha_2^2){L}_{F}^2)}{\mu^2(1-\lambda^2)^2} - \frac{(1-\lambda)^2}{2\eta\alpha_1}+ 4\beta_1^2\alpha_1\eta^3L_{\tilde{F}}^2 \\
		& \quad + \frac{200 \alpha_2\eta^3L_{g_y}^2{L}_{F}^2}{\mu^2}\beta_1^2+ \frac{16\beta_1^2{L}_{\tilde{F}}^2}{\alpha_1^2}+ \frac{400\beta_1^2{L}_{g_{y}}^2{L}_{F}^2}{\alpha_2^2\mu^2} \\
		& \leq \beta_1^2\frac{ 4L_{g_y}^2((11+32/\alpha_1^2){L}_{\tilde{F}}^2+(450+800/\alpha_2^2){L}_{F}^2)}{\mu^2(1-\lambda^2)^2} - \frac{(1-\lambda)^2}{2\eta\alpha_1}+ \frac{4\beta_1^2L_{\tilde{F}}^2{L}_{g_{y}}^2}{\mu^2(1-\lambda^2)^2} \\
		& \quad + \frac{200 L_{g_y}^2{L}_{F}^2}{\mu^2(1-\lambda^2)^2}\beta_1^2+ \frac{16\beta_1^2{L}_{\tilde{F}}^2{L}_{g_{y}}^2}{\alpha_1^2\mu^2(1-\lambda^2)^2}+ \frac{400\beta_1^2{L}_{g_{y}}^2{L}_{F}^2}{\alpha_2^2\mu^2(1-\lambda^2)^2} \ .  \\
	\end{aligned}
\end{equation}
By letting this upper bound non-positive, we can get
\begin{equation}
	\begin{aligned}
		&  \beta_1^2\frac{ 4L_{g_y}^2((11+32/\alpha_1^2){L}_{\tilde{F}}^2+(450+800/\alpha_2^2){L}_{F}^2)}{\mu^2(1-\lambda^2)^2} - \frac{(1-\lambda)^2}{2\eta\alpha_1}+ \frac{4\beta_1^2L_{\tilde{F}}^2{L}_{g_{y}}^2}{\mu^2(1-\lambda^2)^2} \\
		& \quad + \frac{200 L_{g_y}^2{L}_{F}^2}{\mu^2(1-\lambda^2)^2}\beta_1^2+ \frac{16\beta_1^2{L}_{\tilde{F}}^2{L}_{g_{y}}^2}{\alpha_1^2\mu^2(1-\lambda^2)^2}+ \frac{400\beta_1^2{L}_{g_{y}}^2{L}_{F}^2}{\alpha_2^2\mu^2(1-\lambda^2)^2}\leq  0 \ ,  \\
		& \beta_1^2\frac{ 4L_{g_y}^2((12+36/\alpha_1^2){L}_{\tilde{F}}^2+(500+900/\alpha_2^2){L}_{F}^2)}{\mu^2(1-\lambda^2)^2} \leq \frac{(1-\lambda)^2}{2\eta\alpha_1} \ , \\
		& \beta_1^2 \leq \frac{\mu^2(1-\lambda)^4}{ 8L_{g_y}^2((12+36/\alpha_1^2){L}_{\tilde{F}}^2+(500+900/\alpha_2^2){L}_{F}^2)} \ , \\
		& \beta_1 \leq \frac{\mu(1-\lambda)^2}{ 4L_{g_y}\sqrt{(6+18/\alpha_1^2){L}_{\tilde{F}}^2+(250+450/\alpha_2^2){L}_{F}^2}} \ , \\
	\end{aligned}
\end{equation}
where the second to last step holds due to $ \eta\alpha_1< 1$. 

By setting the coefficient of $\mathbb{E}[ \|Z^{g}_t - \bar{Z}^{g}_t\|_F^2]$ to be non-positive, we can get
\begin{equation}
	\begin{aligned}
		& \frac{ 4\eta \beta_1\beta_2^2L_{g_y}^2((11+32/\alpha_1^2){L}_{\tilde{F}}^2+(450+800/\alpha_2^2){L}_{F}^2)}{\mu^2(1-\lambda^2)^2} -  \frac{25 (1-\lambda)^2\beta_1{L}_{F}^2}{\alpha_2\mu^2}+ 4\alpha_1\beta_1\beta_2^2\eta^4L_{\tilde{F}}^2\\
		& \quad \quad +\frac{200 \alpha_2\beta_1\beta_2^2\eta^4L_{g_y}^2{L}_{F}^2}{\mu^2} + \frac{16\eta\beta_1\beta_2^2{L}_{\tilde{F}}^2}{\alpha_1^2}+ \frac{400\eta\beta_1\beta_2^2{L}_{g_y}^2 {L}_{F}^2}{\alpha_2^2\mu^2} \leq  0  \ , \\
		& \beta_2^2\frac{ 4L_{g_y}^2((11+32/\alpha_1^2){L}_{\tilde{F}}^2+(450+800/\alpha_2^2){L}_{F}^2)}{\mu^2(1-\lambda^2)^2} -   \frac{25 (1-\lambda)^2{L}_{F}^2}{\eta\alpha_2\mu^2} \\
		& \quad + 4\alpha_1\beta_2^2\eta^3L_{\tilde{F}}^2+ \frac{200\alpha_2\beta_2^2\eta^3L_{g_y}^2 {L}_{F}^2}{\mu^2}\frac{}{} + \frac{16\beta_2^2{L}_{\tilde{F}}^2}{\alpha_1^2} + \frac{400 \beta_2^2{L}_{g_y}^2{L}_{F}^2}{\alpha_2^2\mu^2}\leq  0 \ . \\
	\end{aligned}
\end{equation}
Because $\alpha_1\eta<1, \alpha_2\eta<1, \eta<1, \frac{L_{g_y}^2}{\mu^2}>1, 1-\lambda^2<1$,  we can get
\begin{equation}
	\begin{aligned}
		& \beta_2^2\frac{ 4L_{g_y}^2((11+32/\alpha_1^2){L}_{\tilde{F}}^2+(450+800/\alpha_2^2){L}_{F}^2)}{\mu^2(1-\lambda^2)^2} -   \frac{25 (1-\lambda)^2{L}_{F}^2}{\eta\alpha_2\mu^2} \\
		& \quad + 4\alpha_1\beta_2^2\eta^3L_{\tilde{F}}^2+ \frac{200\alpha_2\beta_2^2\eta^3L_{g_y}^2 {L}_{F}^2}{\mu^2}\frac{}{} + \frac{16\beta_2^2{L}_{\tilde{F}}^2}{\alpha_1^2} + \frac{400 \beta_2^2{L}_{g_y}^2{L}_{F}^2}{\alpha_2^2\mu^2}\\
		&\leq  \beta_2^2\frac{ 4L_{g_y}^2((11+32/\alpha_1^2){L}_{\tilde{F}}^2+(450+800/\alpha_2^2){L}_{F}^2)}{\mu^2(1-\lambda^2)^2} -   \frac{25 (1-\lambda)^2{L}_{F}^2}{\eta\alpha_2\mu^2} \\
		& \quad + \frac{4\beta_2^2L_{\tilde{F}}^2L_{g_y}^2}{\mu^2(1-\lambda^2)^2}+ \frac{200\beta_2^2L_{g_y}^2 {L}_{F}^2}{\mu^2(1-\lambda^2)^2} + \frac{16\beta_2^2{L}_{\tilde{F}}^2L_{g_y}^2}{\alpha_1^2\mu^2(1-\lambda^2)^2} + \frac{400 \beta_2^2{L}_{g_y}^2{L}_{F}^2}{\alpha_2^2\mu^2(1-\lambda^2)^2}  \ .  \\
	\end{aligned}
\end{equation}
By letting this upper bound non-positive, we can get
\begin{equation}
	\begin{aligned}
		&  \beta_2^2\frac{ 4L_{g_y}^2((11+32/\alpha_1^2){L}_{\tilde{F}}^2+(450+800/\alpha_2^2){L}_{F}^2)}{\mu^2(1-\lambda^2)^2} -   \frac{25 (1-\lambda)^2{L}_{F}^2}{\eta\alpha_2\mu^2} \\
		& \quad + \frac{4\beta_2^2L_{\tilde{F}}^2L_{g_y}^2}{\mu^2(1-\lambda^2)^2}+ \frac{200\beta_2^2L_{g_y}^2 {L}_{F}^2}{\mu^2(1-\lambda^2)^2} + \frac{16\beta_2^2{L}_{\tilde{F}}^2L_{g_y}^2}{\alpha_1^2\mu^2(1-\lambda^2)^2} + \frac{400 \beta_2^2{L}_{g_y}^2{L}_{F}^2}{\alpha_2^2\mu^2(1-\lambda^2)^2}\leq  0  \ ,\\
		& \beta_2^2\frac{ 4L_{g_y}^2((12+36/\alpha_1^2){L}_{\tilde{F}}^2+(500+900/\alpha_2^2){L}_{F}^2)}{\mu^2(1-\lambda^2)^2} \leq  \frac{25 (1-\lambda)^2{L}_{F}^2}{\eta\alpha_2\mu^2}   \ , \\
		& \beta_2^2 \leq  \frac{25 (1-\lambda)^4{L}_{F}^2}{ 4L_{g_y}^2((12+36/\alpha_1^2){L}_{\tilde{F}}^2+(500+900/\alpha_2^2){L}_{F}^2)}  \ ,  \\
		& \beta_2 \leq  \frac{5 (1-\lambda)^2{L}_{F}}{ 2L_{g_y}\sqrt{(12+36/\alpha_1^2){L}_{\tilde{F}}^2+(500+900/\alpha_2^2){L}_{F}^2}}   \ , \\
	\end{aligned}
\end{equation}
where the second to last step holds due to $ \eta\alpha_2< 1$. 
As a result, by setting $\alpha_1\eta<1, \alpha_2\eta<1, \eta<1$, and
\begin{equation}
	\begin{aligned}
		& \beta_1\leq \min\Big\{\frac{\beta_2\mu} {15L_{y} {L}_{F}}, \frac{\mu}{4L_{g_y}\sqrt{((2+8/\alpha_1^2)L_{\tilde{F}}^2+(100+200/\alpha_2^2){L}_{F}^2)}}, \frac{\mu(1-\lambda)^2}{ 4L_{g_y}\sqrt{(6+18/\alpha_1^2){L}_{\tilde{F}}^2+(250+450/\alpha_2^2){L}_{F}^2}} \Big\}  \ , \\
		& \beta_2 \leq \min\Big\{\frac{9\mu{L}_{F}^2}{2((4+16/\alpha_1^2)L_{\tilde{F}}^2+(200+400/\alpha_2^2){L}_{F}^2)L_{g_y}^2}  , \frac{5 (1-\lambda)^2{L}_{F}}{ 2L_{g_y}\sqrt{(12+36/\alpha_1^2){L}_{\tilde{F}}^2+(500+900/\alpha_2^2){L}_{F}^2}} \Big\} \ , 
	\end{aligned}
\end{equation}
we can get
\begin{equation}
	\begin{aligned}
		& \quad   \mathcal{L}_{t+1} - \mathcal{L}_{t}  \\
		&\leq  - \frac{\eta\beta_1}{2} \mathbb{E}[\|\nabla F(\bar{  {x}}_{t})\|^2]  - \frac{\eta\beta_1L_F^2}{2} \mathbb{E}[\|\bar{y}_{t} - y^*(\bar{x}_{t})\|^2] +  \frac{3\eta\beta_1C_{g_{xy}}^2C_{f_y}^2}{\mu^2}(1-\frac{\mu}{L_{g_{y}}})^{2J}\\
		& \quad +\frac{\alpha_1\beta_1\eta^2\sigma_{\tilde{F}}^2}{2}+\frac{25\beta_1\alpha_2\eta^2 {L}_{F}^2\sigma^2}{\mu^2}+ 4\beta_1\alpha_1\eta^2 \sigma_{\tilde{F}}^2  + \frac{100 \beta_1\alpha_2\eta^2{L}_{F}^2 \sigma^2}{\mu^2} \  . \\
	\end{aligned}
\end{equation}
By summing over $t$ from $0$ to $T-1$, we can get
\begin{equation}
	\begin{aligned}
		& \quad \frac{1}{T}\sum_{t=0}^{T-1}\mathbb{E}[\|\nabla F(\bar{  {x}}_{t})\|^2]  + L_F^2 \mathbb{E}[\|\bar{y}_{t} - y^*(\bar{x}_{t})\|^2]  \\
		& \leq \frac{2(\mathcal{L}_{0} - \mathcal{L}_{T} )}{\eta\beta_1T} +  \frac{6C_{g_{xy}}^2C_{f_y}^2}{\mu^2}(1-\frac{\mu}{L_{g_{y}}})^{2J} +9\alpha_1\eta\sigma_{\tilde{F}}^2+ \frac{250 \alpha_2\eta{L}_{F}^2\sigma^2}{\mu^2} \ .   \\
	\end{aligned}
\end{equation}
As for $\mathcal{L}_0$, we have
\begin{equation}
	\begin{aligned}
		& \quad \frac{1}{K}\mathbb{E}[\| Z^{\tilde{F}}_{0}-\bar{Z}^{\tilde{F}}_{0} \|_F^2]  \\
		& = \frac{1}{K}\mathbb{E}[\| \Delta_{0}^{\tilde{F}_{\tilde{\xi}_0}}-\bar{\Delta}_{0}^{\tilde{F}_{\tilde{\xi}_0}} \|_F^2]  \\
		& = \frac{1}{K}\sum_{k=1}^{K}  \mathbb{E}[\|\nabla\tilde{F}^{(k)}(x_0, y_0; \tilde{\xi}_0^{(k)}) - \frac{1}{K} \sum_{k'=1}^{K}\nabla\tilde{F}^{(k')}(x_0, y_0; \tilde{\xi}_0^{(k')})  \|_F^2]  \\
		& = \frac{1}{K}\sum_{k=1}^{K}  \mathbb{E}[\|\nabla\tilde{F}^{(k)}(x_0, y_0; \tilde{\xi}_0^{(k)}) - \nabla\tilde{F}^{(k)}(x_0, y_0) +  \nabla\tilde{F}^{(k)}(x_0, y_0) \\
		& \quad  -  \frac{1}{K} \sum_{k'=1}^{K}\nabla\tilde{F}^{(k')}(x_0, y_0)  + \frac{1}{K} \sum_{k'=1}^{K}\nabla\tilde{F}^{(k')}(x_0, y_0)  - \frac{1}{K} \sum_{k'=1}^{K}\nabla\tilde{F}^{(k')}(x_0, y_0; \tilde{\xi}_0^{(k')})  \|_F^2]  \\
		& = \frac{1}{K}\sum_{k=1}^{K}  \mathbb{E}[\|\nabla\tilde{F}^{(k)}(x_0, y_0; \tilde{\xi}_0^{(k)}) - \nabla\tilde{F}^{(k)}(x_0, y_0)  + \frac{1}{K} \sum_{k'=1}^{K}\nabla\tilde{F}^{(k')}(x_0, y_0)  - \frac{1}{K} \sum_{k'=1}^{K}\nabla\tilde{F}^{(k')}(x_0, y_0; \tilde{\xi}_0^{(k')})  \|_F^2]  \\
		& = 2\sigma_{\tilde{F}}^2 \   , 
	\end{aligned}
\end{equation}
and $\frac{1}{K}\mathbb{E}[\| Z^{g}_{0}-\bar{Z}^{g}_{0} \|_F^2]  \leq  2\sigma^2$, $\frac{1}{K}\mathbb{E}[\|\Delta_{0}^{\tilde{F}}-U_{0}\|_F^2]  = \frac{1}{K}\mathbb{E}[\|\Delta_{0}^{\tilde{F}}-\Delta_{0}^{\tilde{F}_{\tilde{\xi}_0}}\|_F^2] \leq \sigma_{\tilde{F}}^2$, $ \frac{1}{K}\mathbb{E}[\|\Delta_{0}^{g}-V_{0}\|_F^2]  \leq \sigma^2 $.
Then, we can get
\begin{equation}
	\begin{aligned}
		& \mathcal{L}_{0} ={ \mathbb{E}}[F(x_{0})] +\frac{6\beta_1{L}_{F}^2}{\beta_2\mu}\mathbb{E}[\|\bar{   {y}}_{0} -    {y}^{*}(\bar{   {x}}_{0})\| ^2 ]  \\
		& \quad +  \frac{\beta_1(1-\lambda)}{2\alpha_1} \frac{1}{K}\mathbb{E}[\| Z^{\tilde{F}}_{0}-\bar{Z}^{\tilde{F}}_{0} \|_F^2]  + \frac{25 (1-\lambda)\beta_1{L}_{F}^2}{\alpha_2\mu^2} \frac{1}{K}\mathbb{E}[\| Z^{g}_{0}-\bar{Z}^{g}_{0} \|_F^2] \\
		& \quad  +\frac{4\beta_1}{\alpha_1}\frac{1}{K}\mathbb{E}[\|\Delta_{0}^{\tilde{F}}-U_{0}\|_F^2] + \frac{100 \beta_1{L}_{F}^2}{\alpha_2\mu^2}\frac{1}{K}\mathbb{E}[\|\Delta_{0}^{g}-V_{0}\|_F^2]  \\
		& \leq { \mathbb{E}}[F(x_{0})] +\frac{6\beta_1{L}_{F}^2}{\beta_2\mu}\mathbb{E}[\|\bar{   {y}}_{0} -    {y}^{*}(\bar{   {x}}_{0})\| ^2 ] +  \frac{5\beta_1\sigma_{\tilde{F}}^2 }{\alpha_1} + \frac{150 \beta_1{L}_{F}^2\sigma^2}{\alpha_2\mu^2} \ . \\
	\end{aligned}
\end{equation}
Finally, we can get 
\begin{equation}
	\begin{aligned}
		& \quad \frac{1}{T}\sum_{t=0}^{T-1}\mathbb{E}[\|\nabla F(\bar{  {x}}_{t})\|^2]  + L_F^2 \mathbb{E}[\|\bar{y}_{t} - y^*(\bar{x}_{t})\|^2]  \\
		& \leq \frac{2(\mathcal{L}_{0} - \mathcal{L}_{T} )}{\eta\beta_1T} +  \frac{6C_{g_{xy}}^2C_{f_y}^2}{\mu^2}(1-\frac{\mu}{L_{g_{y}}})^{2J} +9\alpha_1\eta\sigma_{\tilde{F}}^2+ \frac{250 \alpha_2\eta{L}_{F}^2\sigma^2}{\mu^2}    \\
		& \leq \frac{2(F(x_0) -F(x_*))}{\eta\beta_1T} +  \frac{12{L}_{F}^2}{\beta_2\mu\eta T}\mathbb{E}[\|\bar{   {y}}_{0} -    {y}^{*}(\bar{   {x}}_{0})\| ^2 ] +  \frac{10\sigma_{\tilde{F}}^2 }{\alpha_1\eta T} +\frac{300{L}_{F}^2\sigma^2}{\alpha_2 \mu^2\eta T}\\
		& \quad +  \frac{6C_{g_{xy}}^2C_{f_y}^2}{\mu^2}(1-\frac{\mu}{L_{g_{y}}})^{2J} +9\alpha_1\eta\sigma_{\tilde{F}}^2+ \frac{250 \alpha_2\eta{L}_{F}^2\sigma^2}{\mu^2} \ .     \\
	\end{aligned}
\end{equation}


\end{proof}

\subsubsection{Proof of Corollary~\ref{corollary_mdbo}}
\begin{proof}
	By setting $T=O(\frac{1}{\epsilon^2(1-\lambda)^2})$, $\eta=O(\epsilon)$, $J=O(\log \frac{1}{\epsilon})$, $\beta_1=O((1-\lambda)^2)$,  $\beta_2=O((1-\lambda)^2)$, $\alpha_1=O(1)$, $\alpha_2=O(1)$, we can get
	\begin{equation}
		\begin{aligned}
			& \frac{2(F(x_0)- F(x_*))}{\eta\beta_1 T} = O(\epsilon),\quad \frac{12{L}_{F}^2}{\beta_2\eta T\mu}\|\bar{   {y}}_{0} -    {y}^{*}(\bar{   {x}}_{0})\| ^2  = O(\epsilon), \quad \frac{12\sigma_{\tilde{F}}^2}{\alpha_1\eta T} =O(\epsilon), \quad  \frac{400 {L}_{F}^2\sigma^2}{\alpha_2\eta T\mu^2}=O(\epsilon) , \\
			& \frac{6C_{g_{xy}}^2C_{f_y}^2}{\mu^2}(1-\frac{\mu}{L_{g_{y}}})^{2J} = O(\epsilon), \quad  10\alpha_1\eta \sigma_{\tilde{F}}^2 =O(\epsilon),\quad \frac{300 \alpha_2\eta {L}_{F}^2\sigma^2}{\mu^2}=O(\epsilon) .  \\
		\end{aligned}
	\end{equation}
	Additionally, since  the communication is conducted in each iteration, the communication complexity is equal to the number of iterations. Thus, it is $O(\frac{1}{\epsilon^2(1-\lambda)^2})$. When computing the stochastic gradient, the batch size is $O(1)$. Thus, the gradient complexity is $O(\frac{1}{\epsilon^2(1-\lambda)^2})$. When computing the stochastic hypergradient, the Jacobian-vector product is conducted for one time in each iteration and the Hessian-vector product is performed for $J$ times in each iteration. Thus, the Jacobian-vector product complexity is $O(\frac{1}{\epsilon^2(1-\lambda)^2})$, and the Hessian-vector product complexity is $J\times T=O(\frac{1}{\epsilon^2(1-\lambda)^2}\log \frac{1}{\epsilon})=\tilde{O}(\frac{1}{\epsilon^2(1-\lambda)^2})$.

\end{proof}

\subsection{Proof of  Theorem~\ref{theorem_mdbo_bounded_gradient_norm}}
To study the convergence of MDBO under Assumption~\ref{assumption_continuous}, we employ a little different initialization condition, which is demonstrated in Algorithm~\ref{alg_MDBO2}.

\begin{algorithm}[]
	\caption{MDBO}
	\label{alg_MDBO2}
	\begin{algorithmic}[1]
		\REQUIRE ${x}_{0}^{(k)}={x}_{0}$, ${y}_{0}^{(k)}={y}_{0}$, $\eta>0$, $\alpha_1>0$, $\alpha_2>0$, $\beta_1>0$, $\beta_2>0$.
		\STATE $U_{-1} =0$,  $V_{-1} = 0$, 
		$Z_{-1}^{\tilde{F}}=0$, $Z_{-1}^{g}= 0$, \\
		\FOR{$t=0,\cdots, T-1$} 
		\STATE $U_{t}  = (1-\alpha_1\eta)U_{t-1}+\alpha_1\eta  \Delta_t^{\tilde{F}_{\tilde{\xi}_t}}$,   \\
$V_{t} = (1-\alpha_2\eta)V_{t-1}+\alpha_2 \eta\Delta_t^{g_{\zeta_t}}$, \\
		\STATE$Z_t^{\tilde{F}} = Z_{t-1}^{\tilde{F}}W + U_{t} - U_{t-1}$ ,  \\
		$Z_t^{g} = Z_{t-1}^{g}W + V_{t} - V_{t-1}$ , \\
		\STATE 
		$X_{t+1}=X_{t}-\eta X_{t} (I-W) - \beta_1\eta Z_t^{\tilde{F}} $ ,  \\
		$Y_{t+1}=Y_{t}-\eta Y_{t} (I-W) - \beta_2\eta Z_t^{g} $,  \\
		\ENDFOR
	\end{algorithmic}
\end{algorithm}

\subsubsection{Characterization of  Gradient Norm}

\begin{lemma} \label{lemma_grad_norm}
	Given Assumptions~\ref{assumption_graph}-\ref{assumption_continuous},   the following inequality holds. 
	\begin{equation}
		\begin{aligned}
			&  \hat{C}_{f_x}^2\triangleq\mathbb{E}[\| \nabla_x f^{(k)}(x, y; \xi) \|^2] \leq \sigma^2 +  C_{f_{x}}^2  \ , \\
			&  \hat{C}_{g_{xy}}^2\triangleq\mathbb{E}[\| \nabla_{xy}^2g^{(k)}(x, y; \zeta)\|^2] \leq \sigma^2 + C_{g_{xy}}^2  \ , \\
			&  \hat{C}_{f_y}^2\triangleq\mathbb{E}[\| \nabla_y f^{(k)}(x, y; \xi)  \|^2] \leq \sigma^2 + C_{f_y}^2 \ ,  \\
			&  \hat{C}_{\tilde{F}}^2\triangleq \mathbb{E}[\|\nabla \tilde{F}^{(k)}(x, y; \tilde{\xi})  \|^2] \leq 2\hat{C}_{f_x}^2 + \frac{2 \hat{C}_{g_{xy}}^2\hat{C}_{f_y}^2}{\mu^2} \  ,\\
			&  \hat{C}_{g_y}^2\triangleq\mathbb{E}[\| \nabla_y g^{(k)}(x, y; \zeta)  \|^2]   \leq \sigma^2 + C_{g_y}^2  \ , \\
			&  \|U_{t}\|_F^2 \leq   K\hat{C}_{\tilde{F}}^2 ,  \quad \|V_{t}\|_F^2 \leq   K\hat{C}_{g_y}^2 \ . \\
		\end{aligned}
	\end{equation}
\end{lemma}
\begin{proof}
	Based on Assumptions~\ref{assumption_graph}-\ref{assumption_continuous}, it is easy to get
	\begin{equation}
		\begin{aligned}
			&  \hat{C}_{f_x}^2\triangleq\mathbb{E}[\| \nabla_x f^{(k)}(x, y; \xi) \|^2] = \mathbb{E}[\| \nabla_x f^{(k)}(x, y; \xi)  - \nabla_x f^{(k)}(x, y)  +  \nabla_x f^{(k)}(x, y) \|^2] \leq \sigma^2 +  C_{f_{x}}^2  \ , \\
			&  \hat{C}_{g_{xy}}^2\triangleq\mathbb{E}[\| \nabla_{xy}^2g^{(k)}(x, y; \zeta)\|^2] =  \mathbb{E}[\| \nabla_{xy}^2g^{(k)}(x, y; \zeta)-\nabla_{xy}^2g^{(k)}(x, y) + \nabla_{xy}^2g^{(k)}(x, y)\|^2] \leq \sigma^2 + C_{g_{xy}}^2 \ ,  \\
			&  \hat{C}_{f_y}^2\triangleq\mathbb{E}[\| \nabla_y f^{(k)}(x, y; \xi)  \|^2]  = \mathbb{E}[\| \nabla_y f^{(k)}(x, y; \xi) -\nabla_y f^{(k)}(x, y) + \nabla_y f^{(k)}(x, y)  \|^2] \leq \sigma^2 + C_{f_y}^2  \ , \\
			&  \hat{C}_{g_y}^2\triangleq\mathbb{E}[\| \nabla_y g^{(k)}(x, y; \zeta)  \|^2]  = \mathbb{E}[\| \nabla_y g^{(k)}(x, y; \zeta) -\nabla_y g^{(k)}(x, y) + \nabla_y g^{(k)}(x, y)  \|^2] \leq \sigma^2 + C_{g_y}^2  \ . \\
		\end{aligned}
	\end{equation}
Then, we can get
	\begin{equation}
		\begin{aligned}
			& \quad \hat{C}_{\tilde{F}}^2\triangleq \mathbb{E}[\|\nabla \tilde{F}^{(k)}(x, y; \tilde{\xi})  \|^2] \\
			& = \mathbb{E}[\| \nabla_x f^{(k)}(x, y; \xi)  -  \nabla_{xy}^2g^{(k)}(x, y; \zeta_0) \frac{J}{L_{g_{y}}}\prod_{j=1}^{\tilde{J}}(I-\frac{1}{L_{g_{y}}}\nabla_{yy}^2g^{(k)}(x, y; \zeta_j))\nabla_y f^{(k)}(x, y; \xi)   \|^2] \\
			& \leq  2\mathbb{E}[\| \nabla_x f^{(k)}(x, y; \xi) \|^2] +  2\mathbb{E}[\|\nabla_{xy}^2g^{(k)}(x, y; \zeta_0) \frac{J}{L_{g_{y}}}\prod_{j=1}^{\tilde{J}}(I-\frac{1}{L_{g_{y}}}\nabla_{yy}^2g^{(k)}(x, y; \zeta_j))\nabla_y f^{(k)}(x, y; \xi)   \|^2] \\
			& \leq 2\mathbb{E}[\| \nabla_x f^{(k)}(x, y; \xi) \|^2] +  2\mathbb{E}\Big[\|\nabla_{xy}^2g^{(k)}(x, y; \zeta_0) \|^2\times\|\frac{J}{L_{g_{y}}}\prod_{j=1}^{\tilde{J}}(I-\frac{1}{L_{g_{y}}}\nabla_{yy}^2g^{(k)}(x, y; \zeta_j)) \|^2\times\|\nabla_y f^{(k)}(x, y; \xi)   \|^2\Big] \\
			& \leq 2\hat{C}_{f_x}^2 + \frac{2 \hat{C}_{g_{xy}}^2\hat{C}_{f_y}^2}{\mu^2} \ , 
		\end{aligned}
	\end{equation}
	where we use Lemma~\ref{lemma_hessian_bound} in the last step. 
	
		Then, we use induction approach to prove $\|U_{t}\|_F^2 $. At first,  when $t=0$, $\|U_{0}\|_F^2 = \|\alpha_1\eta\Delta_0^{\tilde{F}_{\tilde{\xi}_0}}\|_F^2\leq K\hat{C}_{\tilde{F}}^2$ since $\alpha_1\eta\leq 1$. Then, assuming $ \|U_{t}\|_F^2 \leq   K\hat{C}_{\tilde{F}}^2$, we can get
	\begin{equation}
		\begin{aligned}
			&  \|U_{t+1}\|_F  = \|(1-\alpha_1\eta)U_{t} + \alpha_1\eta  \Delta_{t+1}^{\tilde{F}_{\tilde{\xi}_{t+1}}}\|_F\\
			& \leq   (1-\alpha_1\eta)\|U_{t} \|_F+ \alpha_1\eta \| \Delta_{t+1}^{\tilde{F}_{\tilde{\xi}_{t+1}}}\|_F\\
			& \leq (1-\alpha_1\eta)\sqrt{K } \hat{C}_{\tilde{F}}+ \alpha_1\eta\sqrt{K }\hat{C}_{\tilde{F}} \\
			& = \sqrt{K } \hat{C}_{\tilde{F}} \ , \\
		\end{aligned}
	\end{equation}
	where the second step holds due to the convexity of Frobenius norm.  This completes the proof for $\|U_{t}\|_F^2 $. Similarly, we can prove the claim for $\|V_{t}\|_F^2$.
\end{proof}


\newpage
\subsubsection{Characterization of Consensus Errors}
\begin{lemma} \label{lemma_consensus_z_f_bounded_gradient_norm}
	Given Assumptions~\ref{assumption_graph}-\ref{assumption_continuous},   the following inequality holds. 
	\begin{equation}
		\begin{aligned}
			&  \frac{1}{K}\mathbb{E}[\|Z_{t}^{\tilde{F}} - \bar{Z}_{t}^{\tilde{F}}\|_F^2] 	\leq \frac{2\alpha_1^2\eta^2\hat{C}_{\tilde{F}}^2}{(1-\lambda)^2} \ .\\
		\end{aligned}
	\end{equation}
\end{lemma}

\begin{proof}
	
	\begin{equation}
		\begin{aligned}
			& \quad \frac{1}{K}\mathbb{E}[\|Z_{t}^{\tilde{F}} - \bar{Z}_{t}^{\tilde{F}}\|_F^2] \\
			& = \frac{1}{K}\mathbb{E}[\|Z_{t-1}^{\tilde{F}}W +U_{t} - U_{t-1} - \bar{Z}_{t-1}^{\tilde{F}}- \bar{U}_{t} +\bar{U}_{t-1}\|_F^2] \\
			&\overset{(s_1)}\leq \frac{1}{K}\lambda\mathbb{E}[\|Z_{t-1}^{\tilde{F}} - \bar{Z}_{t-1}^{\tilde{F}} \|_F^2]  + \frac{1}{1-\lambda}\frac{1}{K}\mathbb{E}[\|U_{t} - U_{t-1} - \bar{U}_{t} +\bar{U}_{t-1}\|_F^2]\\
			& \leq \lambda\frac{1}{K}\mathbb{E}[\|Z_{t-1}^{\tilde{F}} - \bar{Z}_{t-1}^{\tilde{F}} \|_F^2]  + \frac{1}{1-\lambda}\frac{1}{K}\mathbb{E}[\|U_{t} - U_{t-1} \|_F^2]\\
			& = \lambda\frac{1}{K}\mathbb{E}[\|Z_{t-1}^{\tilde{F}} - \bar{Z}_{t-1}^{\tilde{F}} \|_F^2]  + \frac{1}{1-\lambda}\frac{1}{K}\mathbb{E}[\|(1-\alpha_1\eta)U_{t-1}+ \alpha_1\eta \Delta_{t}^{\tilde{F}_{\tilde{\xi}_t}}- U_{t-1} \|_F^2]\\
			& = \lambda\frac{1}{K}\mathbb{E}[\|Z_{t-1}^{\tilde{F}} - \bar{Z}_{t-1}^{\tilde{F}} \|_F^2]  + \frac{1}{1-\lambda}\frac{1}{K}\mathbb{E}[\|-\alpha_1\eta U_{t-1}+  \alpha_1\eta \Delta_{t}^{\tilde{F}_{\tilde{\xi}_t}}\|_F^2]\\
			& 	\overset{(s_2)}	\leq \lambda\frac{1}{K}\mathbb{E}[\|Z_{t-1}^{\tilde{F}} - \bar{Z}_{t-1}^{\tilde{F}} \|_F^2]  +\frac{2\alpha_1^2\eta^2\hat{C}_{\tilde{F}}^2}{1-\lambda}  \\
			& \overset{(s_3)}\leq \frac{2\alpha_1^2\eta^2\hat{C}_{\tilde{F}}^2}{1-\lambda}\sum_{i=-1}^{t-1} \lambda^{t-1-i} \\
			& \leq \frac{2\alpha_1^2\eta^2\hat{C}_{\tilde{F}}^2}{(1-\lambda)^2} \ , 
		\end{aligned}
	\end{equation}
	where $(s_1)$ holds due to Lemma~\ref{lemma_ineqality} with $a=\frac{1-\lambda}{\lambda}$, $(s_2)$ holds due to Lemma~\ref{lemma_grad_norm}, ${(s_3)}$ holds due to recursive expansion and the initialisation conditions: $U_{-1} =0, Z_{-1}^{\tilde{F}}=0$.

\end{proof}

\begin{lemma} \label{lemma_consensus_z_g_bounded_gradient_norm}
	Given Assumptions~\ref{assumption_graph}-\ref{assumption_lower_smooth},   the following inequality holds. 
	\begin{equation}
		\begin{aligned}
			&  \frac{1}{K}\mathbb{E}[\|Z_{t}^{g} - \bar{Z}_{t}^{g}\|_F^2] \leq \frac{2\alpha_2^2\eta^2\hat{C}_{g_y}^2}{(1-\lambda)^2} \ . \\
		\end{aligned}
	\end{equation}
\end{lemma}
This lemma can be proved by following Lemma~\ref{lemma_consensus_z_f_bounded_gradient_norm}.

\begin{lemma} \label{lemma_consensus_x_bounded_gradient_norm}
	Given Assumptions~\ref{assumption_graph}-\ref{assumption_lower_smooth},   the following inequality holds. 
	\begin{equation}
		\begin{aligned}
			& \frac{1}{K}\mathbb{E}[\|X_{t+1} - \bar{X}_{t+1}\|_F^2] \leq   \frac{8\alpha_1^2 \beta_1^2\eta^2\hat{C}_{\tilde{F}}^2}{(1-\lambda)^4}\ .\\ 
		\end{aligned}
	\end{equation}
\end{lemma}

\begin{proof}
	\begin{equation}
		\begin{aligned}
			&  \quad \frac{1}{K}\mathbb{E}[\|X_{t+1} - \bar{X}_{t+1}\|_F^2 ]\\
			& \overset{(s_0)}= \frac{1}{K}\mathbb{E}[\|X_{t}-\eta X_{t} (I-W) - \beta_1\eta Z_t^{\tilde{F}} - \bar{X}_{t}+\beta_1\eta \bar{ Z}_t^{\tilde{F}} \|_F^2]\\
			& = \frac{1}{K}\mathbb{E}[\|(1-\eta)(X_{t}- \bar{X}_{t})+ \eta (X_{t} W -  \bar{X}_{t} - \beta_1 Z_t^{\tilde{F}}+\beta_1 \bar{ Z}_t^{\tilde{F}} )\|_F^2]\\
			& \overset{(s_1)}\leq (1-\eta)\frac{1}{K}\mathbb{E}[\|X_t -\bar{X}_t\|_F^2]+ \eta\frac{1}{K}\mathbb{E}[\|X_{t}W + \beta_1 Z_t^{\tilde{F}} -\bar{X}_{t} - \beta_1\bar{ Z}_t^{\tilde{F}} \|_F^2]\\ 
			& \overset{(s_2)}\leq (1-\eta)\frac{1}{K}\mathbb{E}[\|X_t -\bar{X}_t\|_F^2]+ \frac{\eta(1+\lambda^2)}{2\lambda^2}\frac{1}{K}\mathbb{E}[\|X_{t}W  -\bar{X}_{t} \|_F^2]+ \frac{2\eta \beta_1^2}{1-\lambda^2}\frac{1}{K}\mathbb{E}[\|Z_t^{\tilde{F}} - \bar{ Z}_t^{\tilde{F}}\|_F^2]\\ 
			& \overset{(s_3)}\leq \Big(1-\frac{\eta(1-\lambda^2)}{2}\Big)\frac{1}{K}\mathbb{E}[\|X_t -\bar{X}_t\|_F^2]+ \frac{2\eta \beta_1^2}{1-\lambda^2}\frac{1}{K}\mathbb{E}[\|Z_t^{\tilde{F}} - \bar{ Z}_t^{\tilde{F}}\|_F^2] \\ 
			& \overset{(s_4)}\leq \Big(1-\frac{\eta(1-\lambda^2)}{2}\Big)\frac{1}{K}\mathbb{E}[\|X_t -\bar{X}_t\|_F^2]+ \frac{2\eta \beta_1^2}{1-\lambda^2}\frac{2\alpha_1^2\eta^2\hat{C}_{\tilde{F}}^2}{(1-\lambda)^2} \\
			& \leq  \frac{2\eta \beta_1^2}{1-\lambda^2}\frac{2\alpha_1^2\eta^2\hat{C}_{\tilde{F}}^2}{(1-\lambda)^2} \sum_{i=0}^{t} \Big(1-\frac{\eta(1-\lambda^2)}{2}\Big)^{t-i} \\
			& \leq \frac{8\alpha_1^2 \beta_1^2\eta^2\hat{C}_{\tilde{F}}^2}{(1-\lambda)^4} \ , \\
		\end{aligned}
	\end{equation}
	where $(s_0)$ holds due to $X_{t} (I-W)\mathbf{1}=0$,  $(s_1)$ holds due to Lemma~\ref{lemma_ineqality} with $a=\frac{\eta}{1-\eta}$,  $(s_2)$ holds due to Lemma~\ref{lemma_ineqality} with $a=\frac{1-\lambda^2}{2\lambda^2}$, 
	$(s_3)$ holds due to $\|X_{t}W  -\bar{X}_{t} \|_F^2\leq \lambda^2 \|X_{t}  -\bar{X}_{t} \|_F^2$, $(s_4)$ holds due to Lemma~\ref{lemma_consensus_z_f_bounded_gradient_norm}.
\end{proof}

\begin{lemma} \label{lemma_consensus_y_bounded_gradient_norm}
	Given Assumptions~\ref{assumption_graph}-\ref{assumption_lower_smooth},   the following inequality holds. 
	\begin{equation}
		\begin{aligned}
			&  \frac{1}{K}\mathbb{E}[\|Y_{t+1} - \bar{Y}_{t+1}\|_F^2 ]\leq  \frac{8\alpha_2^2\beta_2^2\eta^2\hat{C}_{g_y}^2}{(1-\lambda)^4} \ .\\ 
		\end{aligned}
	\end{equation}
\end{lemma}
This lemma can be proved by following Lemma~\ref{lemma_consensus_x_bounded_gradient_norm}.

\begin{lemma} \label{lemma_incremental_x_bounded_gradient_norm}
	Given Assumptions~\ref{assumption_graph}-\ref{assumption_lower_smooth},   the following inequality holds. 
	\begin{equation}
		\begin{aligned}
			& \quad \frac{1}{K}\mathbb{E}[\|X_{t+1}-X_{t}\|_F^2]\leq   \frac{72\alpha_1^2 \beta_1^2\eta^4\hat{C}_{\tilde{F}}^2}{(1-\lambda)^4}   + 4\eta^2\beta_1^2 \mathbb{E}[\|\bar{u}_t\|^2]  \ . \\
		\end{aligned}
	\end{equation}
\end{lemma}

\begin{proof}
	\begin{equation}
		\begin{aligned}
			& \quad  \frac{1}{K}\mathbb{E}[\|X_{t+1}-X_{t}\|_F^2]\\
			& = \frac{1}{K}\mathbb{E}[\|X_{t}-\eta X_{t} (I-W) - \beta_1\eta Z^{\tilde{F}}_t-X_{t}\|_F^2]\\
			& = \frac{1}{K}\mathbb{E}[\|\eta X_{t}(W-I) -  \beta_1\eta Z^{\tilde{F}}_t \|_F^2 ]\\
			& \leq  2\eta^2\frac{1}{K}\mathbb{E}[\|X_{t}(W-I)\|_F^2] + 2\eta^2\beta_1^2\frac{1}{K}\mathbb{E}[\|Z^{\tilde{F}}_t\|_F^2 ]\\
			& =  2\eta^2\frac{1}{K}\mathbb{E}[\|(X_{t}-\bar{X}_t)(W-I)\|_F^2]+ 2\eta^2\beta_1^2\frac{1}{K}\mathbb{E}[\|Z^{\tilde{F}}_t-\bar{Z}^{\tilde{F}}_t+\bar{Z}^{\tilde{F}}_t\|_F^2]\\
			& \overset{(s_1)}\leq 8\eta^2\frac{1}{K}\mathbb{E}[\|X_{t}-\bar{X}_t\|_F^2] + 4\eta^2\beta_1^2\frac{1}{K}\mathbb{E}[\|Z^{\tilde{F}}_t-\bar{Z}^{\tilde{F}}_t\|_F^2]+4\eta^2\beta_1^2\frac{1}{K}\mathbb{E}[\|\bar{Z}^{\tilde{F}}_t\|_F^2] \\
			& \overset{(s_1)}\leq \frac{64\alpha_1^2 \beta_1^2\eta^4\hat{C}_{\tilde{F}}^2}{(1-\lambda)^4} + \frac{8\alpha_1^2\beta_1^2\eta^4\hat{C}_{\tilde{F}}^2}{(1-\lambda)^2}  + 4\eta^2\beta_1^2 \mathbb{E}[\|\bar{u}_t\|^2]  \\
			& \leq \frac{72\alpha_1^2 \beta_1^2\eta^4\hat{C}_{\tilde{F}}^2}{(1-\lambda)^4}   + 4\eta^2\beta_1^2 \mathbb{E}[\|\bar{u}_t\|^2]  \ ,  \\
		\end{aligned}
	\end{equation}
	where $(s_1)$ holds due to $\|AB\|_F\leq \|A\|_2\|B\|_F$ and  $\|I-W\|_2\leq 2$,  $(s_1)$ holds due to Lemma~\ref{lemma_consensus_z_f_bounded_gradient_norm} and Lemma~\ref{lemma_consensus_x_bounded_gradient_norm}.
	
\end{proof}

\begin{lemma} \label{lemma_incremental_y_bounded_gradient_norm}
	Given Assumptions~\ref{assumption_graph}-\ref{assumption_lower_smooth},   the following inequality holds. 
	\begin{equation}
		\begin{aligned}
			& \quad \frac{1}{K}\mathbb{E}[\|Y_{t+1}-Y_{t}\|_F^2]\leq \frac{72\alpha_2^2\beta_2^2\eta^4\hat{C}_{g_y}^2}{(1-\lambda)^4} + 4\eta^2\beta_2^2 \mathbb{E}[\|\bar{v}_{t}\|^2] \ . \\
		\end{aligned}
	\end{equation}
\end{lemma}


\subsubsection{Characterization of Gradient Estimators}
\begin{lemma} \label{lemma_hyper_momentum_var_bounded_gradient_norm}
	Given Assumptions~\ref{assumption_graph}-\ref{assumption_lower_smooth},   the following inequality holds. 
	\begin{equation}
		\begin{aligned}
			& \mathbb{E}[\|(\Delta_t^{\tilde{F}}-   U_t)\frac{1}{K}\mathbf{1}\|^2]\leq  (1-\alpha_1\eta)\mathbb{E} [ \|(\Delta_{t-1}^{\tilde{F}}-  U_{t-1})\frac{1}{K}\mathbf{1} \|^2 ] + \frac{4\eta\beta_1^2{L}_{\tilde{F}}^2}{\alpha_1}\mathbb{E}[\|\bar{u}_{t-1}\|^2]+ \frac{4\eta\beta_2^2{L}_{\tilde{F}}^2}{\alpha_1} \mathbb{E}[\|\bar{v}_{t-1}\|^2]\\
			& \quad+\frac{72\alpha_1 \beta_1^2\eta^3\hat{C}_{\tilde{F}}^2{L}_{\tilde{F}}^2}{(1-\lambda)^4}   +\frac{72\alpha_2^2\beta_2^2\eta^3\hat{C}_{g_y}^2{L}_{\tilde{F}}^2}{\alpha_1(1-\lambda)^4}  +\frac{\alpha_1^2\eta^2 \sigma_{\tilde{F}}^2}{K} \ . \\
		\end{aligned}
	\end{equation}
	
\end{lemma}

\begin{proof}
	\begin{equation}
		\begin{aligned}
			& \quad \mathbb{E}[\|(\Delta_t^{\tilde{F}}-   U_t)\frac{1}{K}\mathbf{1}\|^2]\\
			& = \mathbb{E}[\|(\Delta_t^{\tilde{F}}-   (1-\alpha_1\eta)U_{t-1} -  \alpha_1\eta\Delta_t^{\tilde{F}_{\tilde{\xi}_t}})\frac{1}{K}\mathbf{1}\|^2]\\
			& = \mathbb{E} [ \| ( (1-\alpha_1\eta)(\Delta_{t-1}^{\tilde{F}}-  U_{t-1}) +(1-\alpha_1\eta)(\Delta_{t}^{\tilde{F}}-\Delta_{t-1}^{\tilde{F}})+ \alpha_1\eta(\Delta_{t}^{\tilde{F}}- \Delta_{t}^{\tilde{F}_{\tilde{\xi}_t}}) )\frac{1}{K}\mathbf{1} \|^2 ]\\
			&\overset{(s_1)}= (1-\alpha_1\eta)^2\mathbb{E} [ \|  (\Delta_{t-1}^{\tilde{F}}-  U_{t-1})\frac{1}{K}\mathbf{1} +(\Delta_{t}^{\tilde{F}}-\Delta_{t-1}^{\tilde{F}}) \frac{1}{K}\mathbf{1} \|^2 ]+\alpha_1^2\eta^2\mathbb{E} [ \|(\Delta_{t}^{\tilde{F}}- \Delta_{t}^{\tilde{F}_{\tilde{\xi}_t}})\frac{1}{K}\mathbf{1} \|^2 ]\\
			&\overset{(s_2)}\leq (1-\alpha_1\eta) \mathbb{E} [ \|(\Delta_{t-1}^{\tilde{F}}-  U_{t-1})\frac{1}{K}\mathbf{1}\|^2 ]+\frac{1}{\alpha_1\eta}\frac{1}{K}\mathbb{E} [ \|\Delta_{t}^{\tilde{F}}-\Delta_{t-1}^{\tilde{F}} \|_F^2 ]  +  \frac{\alpha_1^2\eta^2 \sigma_{\tilde{F}}^2}{K}\\
			&\overset{(s_3)}\leq (1-\alpha_1\eta)\mathbb{E} [ \|(\Delta_{t-1}^{\tilde{F}}-  U_{t-1} )\frac{1}{K}\mathbf{1}\|^2 ] +\frac{{L}_{\tilde{F}}^2}{\alpha_1\eta}\frac{1}{K} \mathbb{E}[\|X_{t}-X_{t-1}\|_F^2]+\frac{{L}_{\tilde{F}}^2}{\alpha_1\eta}\frac{1}{K} \mathbb{E}[\|Y_{t}-Y_{t-1}\|_F^2] + \frac{\alpha_1^2\eta^2 \sigma_{\tilde{F}}^2}{K} \\
			& \overset{(s_4)}\leq (1-\alpha_1\eta)\mathbb{E} [ \|(\Delta_{t-1}^{\tilde{F}}-  U_{t-1} )\frac{1}{K}\mathbf{1}\|^2 ] +\frac{{L}_{\tilde{F}}^2}{\alpha_1\eta}\frac{72\alpha_1^2 \beta_1^2\eta^4\hat{C}_{\tilde{F}}^2}{(1-\lambda)^4}  + \frac{{L}_{\tilde{F}}^2}{\alpha_1\eta}4\eta^2\beta_1^2\mathbb{E}[\|\bar{u}_{t-1}\|^2]\\
			& \quad +\frac{{L}_{\tilde{F}}^2}{\alpha_1\eta}\frac{72\alpha_2^2\beta_2^2\eta^4\hat{C}_{g_y}^2}{(1-\lambda)^4} + \frac{{L}_{\tilde{F}}^2}{\alpha_1\eta}4\eta^2\beta_2^2 \mathbb{E}[\|\bar{v}_{t-1}\|^2] +\frac{\alpha_1^2\eta^2 \sigma_{\tilde{F}}^2}{K}  \\
			& \leq (1-\alpha_1\eta)\mathbb{E} [ \|(\Delta_{t-1}^{\tilde{F}}-  U_{t-1})\frac{1}{K}\mathbf{1} \|^2 ] + \frac{4\eta\beta_1^2{L}_{\tilde{F}}^2}{\alpha_1}\mathbb{E}[\|\bar{u}_{t-1}\|^2]+ \frac{4\eta\beta_2^2{L}_{\tilde{F}}^2}{\alpha_1} \mathbb{E}[\|\bar{v}_{t-1}\|^2]\\
			& \quad+\frac{72\alpha_1 \beta_1^2\eta^3\hat{C}_{\tilde{F}}^2{L}_{\tilde{F}}^2}{(1-\lambda)^4}   +\frac{72\alpha_2^2\beta_2^2\eta^3\hat{C}_{g_y}^2{L}_{\tilde{F}}^2}{\alpha_1(1-\lambda)^4}  +\frac{\alpha_1^2\eta^2 \sigma_{\tilde{F}}^2}{K}  \ ,  \\
		\end{aligned}
	\end{equation}
	where $(s_1)$ holds due to  $\Delta_{t}^{\tilde{F}}= \mathbb{E}[\Delta_{t}^{\tilde{F}_{\tilde{\xi}_t}}]$,  
	$(s_2)$ holds due to Lemma~\ref{lemma_hypergrad_var} and Lemma~\ref{lemma_ineqality} with  $a=\frac{\alpha_1\eta}{1-\alpha_1\eta}$, $(s_3)$ holds due to Lemma~\ref{lemma_hypergrad_smooth}, $(s_4)$ holds due to Lemma~\ref{lemma_incremental_x_bounded_gradient_norm} and Lemma~\ref{lemma_incremental_y_bounded_gradient_norm}. 
\end{proof}

\begin{lemma} \label{lemma_lower_momentum_var_bounded_gradient_norm}
	Given Assumptions~\ref{assumption_graph}-\ref{assumption_lower_smooth},   the following inequality holds. 
	\begin{equation}
		\begin{aligned}
			&  \mathbb{E} [ \|(\Delta_t^{g}-V_t )\frac{1}{K}\mathbf{1}\|^2 ]\leq (1-\alpha_2\eta)\mathbb{E} [ \| (\Delta_{t-1}^{g}-  V_{t-1})\frac{1}{K}\mathbf{1} \|^2 ] + \frac{4\eta\beta_1^2{L}_{g_{y}}^2}{\alpha_2}\mathbb{E} [\|\bar{u}_{t-1}\|^2]+\frac{4\eta\beta_2^2{L}_{g_y}^2}{\alpha_2} \mathbb{E} [\|\bar{v}_{t-1}\|^2] \\
			& \quad+\frac{72\alpha_1^2 \beta_1^2\eta^3\hat{C}_{\tilde{F}}^2{L}_{g_{y}}^2}{\alpha_2(1-\lambda)^4} +\frac{72\alpha_2\beta_2^2\eta^3\hat{C}_{g_y}^2{L}_{g_y}^2}{(1-\lambda)^4} + \frac{ \alpha_2^2\eta^2 \sigma^2 }{K}  \ . \\
		\end{aligned}
	\end{equation}
\end{lemma}

Then,  we introduce the following potential function:
\begin{equation}
	\begin{aligned}
		& \mathcal{L}_{t+1} = F(x_{t+1}) +\frac{6\beta_1L_F^2}{\beta_2\mu}\mathbb{E}[\|\bar{   {y}}_{t+1} -    {y}^{*}(\bar{   {x}}_{t+1})\| ^2 ]  \\
		& \quad  +\frac{3\beta_1}{\alpha_1}\frac{1}{K}\mathbb{E}[\|\Delta_{t+1}^{\tilde{F}}-U_{t+1}\|_F^2] + \frac{50\beta_1L_F^2}{\alpha_2\mu^2}\frac{1}{K}\mathbb{E}[\|\Delta_{t+1}^{g}-V_{t+1}\|_F^2] \ . \\
	\end{aligned}
\end{equation}

Based on Lemmas~\ref{lemma_consensus_x_bounded_gradient_norm},~\ref{lemma_consensus_y_bounded_gradient_norm},~\ref{lemma_hyper_momentum_var_bounded_gradient_norm},~\ref{lemma_lower_momentum_var_bounded_gradient_norm},  we can get
\begin{equation} 
	\begin{aligned}
		& \quad   \mathcal{L}_{t+1} - \mathcal{L}_{t}  \\
		& \leq - \frac{\eta\beta_1}{2} \mathbb{E}[\|\nabla F(\bar{  {x}}_{t})\|^2]   - \frac{\eta\beta_1L_F^2}{2}\mathbb{E}[\|\bar{y}_{t} - y^*(\bar{x}_{t})\|^2] +  \frac{3\eta\beta_1C_{g_{xy}}^2C_{f_y}^2}{\mu^2}(1-\frac{\mu}{L_{g_{y}}})^{2J}\\
		& \quad + (\frac{25\eta\beta_1^2L_{g_y}^2 }{6\beta_2\mu}w_1+ w_6\frac{{L}_{\tilde{F}}^2}{\alpha_1\eta}4\eta^2\beta_1^2 + w_7\frac{{L}_{g_{y}}^2}{\alpha_2\eta}4\eta^2\beta_1^2-  \frac{\eta\beta_1}{4}) \mathbb{E}[\|\bar{  {u}}_t\|^2]  \\
		& \quad +(  w_6\frac{{L}_{\tilde{F}}^2}{\alpha_1\eta}4\eta^2\beta_2^2+w_7\frac{{L}_{g_y}^2}{\alpha_2\eta}4\eta^2\beta_2^2 - \frac{3\eta\beta_2^2}{4} w_1)\mathbb{E}[\|\bar{v}_{t}  \|^2]   \\
		& \quad +3\eta\beta_1{L}_{\tilde{F}}^2\frac{8\alpha_1^2 \beta_1^2\eta^2\hat{C}_{\tilde{F}}^2}{(1-\lambda)^4}+ 3\eta\beta_1{L}_{\tilde{F}}^2\frac{8\alpha_2^2\beta_2^2\eta^2\hat{C}_{g_y}^2}{(1-\lambda)^4}+ w_1 \frac{25\beta_2 \eta L_{g_y}^2}{3\mu}\frac{8\alpha_1^2 \beta_1^2\eta^2\hat{C}_{\tilde{F}}^2}{(1-\lambda)^4} + w_1 \frac{25\beta_2 \eta L_{g_y}^2}{3\mu} \frac{8\alpha_2^2\beta_2^2\eta^2\hat{C}_{g_y}^2}{(1-\lambda)^4}\\
		& \quad +w_6\frac{{L}_{\tilde{F}}^2}{\alpha_1\eta}\frac{72\alpha_1^2 \beta_1^2\eta^4\hat{C}_{\tilde{F}}^2}{(1-\lambda)^4}  +w_6\frac{{L}_{\tilde{F}}^2}{\alpha_1\eta}\frac{72\alpha_2^2\beta_2^2\eta^4\hat{C}_{g_y}^2}{(1-\lambda)^4} +w_7\frac{{L}_{g_{y}}^2}{\alpha_2\eta}\frac{72\alpha_1^2 \beta_1^2\eta^4\hat{C}_{\tilde{F}}^2}{(1-\lambda)^4}+w_7\frac{{L}_{g_y}^2}{\alpha_2\eta}\frac{72\alpha_2^2\beta_2^2\eta^4\hat{C}_{g_y}^2}{(1-\lambda)^4}  \\
		& \quad   +w_6\frac{\alpha_1^2\eta^2 \sigma_{\tilde{F}}^2}{K}  + w_7\frac{ \alpha_2^2\eta^2 \sigma^2 }{K}  \\
& \leq - \frac{\eta\beta_1}{2} \mathbb{E}[\|\nabla F(\bar{  {x}}_{t})\|^2]   - \frac{\eta\beta_1L_F^2}{2}\mathbb{E}[\|\bar{y}_{t} - y^*(\bar{x}_{t})\|^2] +  \frac{3\eta\beta_1C_{g_{xy}}^2C_{f_y}^2}{\mu^2}(1-\frac{\mu}{L_{g_{y}}})^{2J}\\
& \quad + (\frac{25\eta\beta_1^2L_{g_y}^2 }{6\beta_2\mu}\frac{6\beta_1L_F^2}{\beta_2\mu}+ \frac{3\beta_1}{\alpha_1}\frac{{L}_{\tilde{F}}^2}{\alpha_1\eta}4\eta^2\beta_1^2 + \frac{150\beta_1L_F^2}{3\alpha_2\mu^2}\frac{{L}_{g_{y}}^2}{\alpha_2\eta}4\eta^2\beta_1^2-  \frac{\eta\beta_1}{4}) \mathbb{E}[\|\bar{  {u}}_t\|^2]  \\
& \quad +(  \frac{3\beta_1}{\alpha_1}\frac{{L}_{\tilde{F}}^2}{\alpha_1\eta}4\eta^2\beta_2^2+\frac{150\beta_1L_F^2}{3\alpha_2\mu^2}\frac{{L}_{g_y}^2}{\alpha_2\eta}4\eta^2\beta_2^2 - \frac{3\eta\beta_2^2}{4} \frac{6\beta_1L_F^2}{\beta_2\mu})\mathbb{E}[\|\bar{v}_{t}  \|^2]   \\
& \quad +3\eta\beta_1{L}_{\tilde{F}}^2\frac{8\alpha_1^2 \beta_1^2\eta^2\hat{C}_{\tilde{F}}^2}{(1-\lambda)^4}+ 3\eta\beta_1{L}_{\tilde{F}}^2\frac{8\alpha_2^2\beta_2^2\eta^2\hat{C}_{g_y}^2}{(1-\lambda)^4}+ \frac{6\beta_1L_F^2}{\beta_2\mu} \frac{25\beta_2 \eta L_{g_y}^2}{3\mu}\frac{8\alpha_1^2 \beta_1^2\eta^2\hat{C}_{\tilde{F}}^2}{(1-\lambda)^4} + \frac{6\beta_1L_F^2}{\beta_2\mu} \frac{25\beta_2 \eta L_{g_y}^2}{3\mu} \frac{8\alpha_2^2\beta_2^2\eta^2\hat{C}_{g_y}^2}{(1-\lambda)^4}\\
& \quad +\frac{3\beta_1}{\alpha_1}\frac{{L}_{\tilde{F}}^2}{\alpha_1\eta}\frac{72\alpha_1^2 \beta_1^2\eta^4\hat{C}_{\tilde{F}}^2}{(1-\lambda)^4}  +\frac{3\beta_1}{\alpha_1}\frac{{L}_{\tilde{F}}^2}{\alpha_1\eta}\frac{72\alpha_2^2\beta_2^2\eta^4\hat{C}_{g_y}^2}{(1-\lambda)^4} +\frac{150\beta_1L_F^2}{3\alpha_2\mu^2}\frac{{L}_{g_{y}}^2}{\alpha_2\eta}\frac{72\alpha_1^2 \beta_1^2\eta^4\hat{C}_{\tilde{F}}^2}{(1-\lambda)^4}+\frac{150\beta_1L_F^2}{3\alpha_2\mu^2}\frac{{L}_{g_y}^2}{\alpha_2\eta}\frac{72\alpha_2^2\beta_2^2\eta^4\hat{C}_{g_y}^2}{(1-\lambda)^4}  \\
& \quad   +\frac{3\beta_1}{\alpha_1}\frac{\alpha_1^2\eta^2 \sigma_{\tilde{F}}^2}{K}  + \frac{150\beta_1L_F^2}{3\alpha_2\mu^2}\frac{ \alpha_2^2\eta^2 \sigma^2 }{K}  \\
& \leq - \frac{\eta\beta_1}{2} \mathbb{E}[\|\nabla F(\bar{  {x}}_{t})\|^2]   - \frac{\eta\beta_1L_F^2}{2}\mathbb{E}[\|\bar{y}_{t} - y^*(\bar{x}_{t})\|^2] +  \frac{3\eta\beta_1C_{g_{xy}}^2C_{f_y}^2}{\mu^2}(1-\frac{\mu}{L_{g_{y}}})^{2J}\\
& \quad + (\frac{25\eta\beta_1^3L_F^2L_{y}^2 }{\beta_2^2\mu^2}+ \frac{12\eta\beta_1^3{L}_{\tilde{F}}^2}{\alpha_1^2} + \frac{200\eta\beta_1^3L_F^2{L}_{g_{y}}^2}{\alpha_2^2\mu^2}-  \frac{\eta\beta_1}{4}) \mathbb{E}[\|\bar{  {u}}_t\|^2]  \\
& \quad +(  \frac{12\eta\beta_1\beta_2^2{L}_{\tilde{F}}^2}{\alpha_1^2}+\frac{200\eta\beta_1\beta_2^2 L_F^2{L}_{g_y}^2}{\alpha_2^2\mu^2}-  \frac{9\eta\beta_1\beta_2L_F^2}{2\mu})\mathbb{E}[\|\bar{v}_{t}  \|^2]   \\
& \quad +\frac{24\alpha_1^2 \beta_1^3\eta^3\hat{C}_{\tilde{F}}^2{L}_{\tilde{F}}^2}{(1-\lambda)^4}+ \frac{24\alpha_2^2\beta_1\beta_2^2\eta^3\hat{C}_{g_y}^2{L}_{\tilde{F}}^2}{(1-\lambda)^4}+ \frac{400\alpha_1^2 \beta_1^3\eta^3\hat{C}_{\tilde{F}}^2 L_{g_y}^2L_F^2}{\mu^2(1-\lambda)^4} +  \frac{400 \alpha_2^2\beta_1\beta_2^2\eta^3\hat{C}_{g_y}^2L_F^2L_{g_y}^2}{\mu^2(1-\lambda)^4}\\
& \quad +\frac{216 \beta_1^3\eta^3\hat{C}_{\tilde{F}}^2{L}_{\tilde{F}}^2}{(1-\lambda)^4}  +\frac{216\alpha_2^2\beta_1\beta_2^2\eta^3\hat{C}_{g_y}^2{L}_{\tilde{F}}^2}{\alpha_1^2(1-\lambda)^4} +\frac{3600\alpha_1^2 \beta_1^3\eta^3\hat{C}_{\tilde{F}}^2L_F^2{L}_{g_{y}}^2}{\alpha_2^2\mu^2(1-\lambda)^4}+\frac{3600\beta_1\beta_2^2\eta^3\hat{C}_{g_y}^2L_F^2{L}_{g_y}^2}{\mu^2(1-\lambda)^4}  \\
& \quad   +\frac{3\beta_1\alpha_1\eta^2 \sigma_{\tilde{F}}^2}{K}  + \frac{ 50\beta_1\alpha_2\eta^2 \sigma^2 L_F^2}{\mu^2K} \ .   \\
	\end{aligned}
\end{equation}
By setting the coefficient of $\mathbb{E}[\|\bar{  {u}}_t\|^2]$ to be non-positive, we can get
\begin{equation}
	\begin{aligned}
		& \frac{25\eta\beta_1^3L_F^2L_{y}^2 }{\beta_2^2\mu^2}+ \frac{12\eta\beta_1^3{L}_{\tilde{F}}^2}{\alpha_1^2} + \frac{200\eta\beta_1^3L_F^2{L}_{g_{y}}^2}{\alpha_2^2\mu^2}-  \frac{\eta\beta_1}{4} \leq  0 \ , \\
		& \frac{25\beta_1^2L_F^2L_{y}^2 }{\beta_2^2\mu^2}+ \frac{12\beta_1^2{L}_{\tilde{F}}^2}{\alpha_1^2} + \frac{200\beta_1^2L_F^2{L}_{g_{y}}^2}{\alpha_2^2\mu^2}-  \frac{1}{4} \leq  0 \ . \\
	\end{aligned}
\end{equation}
By setting  $\beta_1 \leq  \frac{\beta_2\mu} {15L_FL_{y} }$, we can get  $ \frac{25\beta_1^2L_F^2L_{y}^2 }{\beta_2^2\mu^2} -  \frac{1}{4} \leq -  \frac{1}{8}$. 
Then,  we have 
\begin{equation}
	\begin{aligned}
		& \quad \frac{25\beta_1^2L_F^2L_{y}^2 }{\beta_2^2\mu^2}+ \frac{12\beta_1^2{L}_{\tilde{F}}^2}{\alpha_1^2} + \frac{200\beta_1^2L_F^2{L}_{g_{y}}^2}{\alpha_2^2\mu^2}-  \frac{1}{4} \\
		& \leq  \frac{12\beta_1^2{L}_{\tilde{F}}^2}{\alpha_1^2} + \frac{200\beta_1^2L_F^2{L}_{g_{y}}^2}{\alpha_2^2\mu^2}-  \frac{1}{8} \\
		& \leq  \frac{12\beta_1^2{L}_{\tilde{F}}^2{L}_{g_{y}}^2}{\alpha_1^2\mu^2} + \frac{200\beta_1^2L_F^2{L}_{g_{y}}^2}{\alpha_2^2\mu^2}-  \frac{1}{8}  \ , \\
	\end{aligned}
\end{equation}
where the last step holds due to ${L}_{g_{y}}/\mu>1$. By letting this upper bound non-positive, we can get
\begin{equation}
	\begin{aligned}
		&   \frac{12\beta_1^2{L}_{\tilde{F}}^2{L}_{g_{y}}^2}{\alpha_1^2\mu^2} + \frac{200\beta_1^2L_F^2{L}_{g_{y}}^2}{\alpha_2^2\mu^2}-  \frac{1}{8}  \leq  0    \ , \\
		& \frac{(12{L}_{\tilde{F}}^2/\alpha_1^2+200L_F^2/\alpha_2^2)\beta_1^2{L}_{g_{y}}^2}{\mu^2}  \leq \frac{1}{8}  \ ,  \\
		& \beta_1\leq   \frac{\mu} {4{L}_{g_{y}}\sqrt{6{L}_{\tilde{F}}^2/\alpha_1^2+100L_F^2/\alpha_2^2}} \ . \\
	\end{aligned}
\end{equation}
By setting the coefficient of $\mathbb{E}[\|\bar{v}_{t}  \|^2]  $ to be non-positive, we can get
\begin{equation}
	\begin{aligned}
		& \frac{12\eta\beta_1\beta_2^2{L}_{\tilde{F}}^2}{\alpha_1^2}+\frac{200\eta\beta_1\beta_2^2 L_F^2{L}_{g_y}^2}{\alpha_2^2\mu^2}-  \frac{9\eta\beta_1\beta_2L_F^2}{2\mu} \leq  0  \ , \\
		& \frac{12\beta_2{L}_{\tilde{F}}^2}{\alpha_1^2}+\frac{200\beta_2 L_F^2{L}_{g_y}^2}{\alpha_2^2\mu^2}-  \frac{9L_F^2}{2\mu} \leq  0  \ , \\
	\end{aligned}
\end{equation}
Because ${L}_{g_{y}}/\mu>1$, we can get
\begin{equation}
	\begin{aligned}
		&  \quad \frac{12\beta_2{L}_{\tilde{F}}^2}{\alpha_1^2}+\frac{200\beta_2 L_F^2{L}_{g_y}^2}{\alpha_2^2\mu^2}-  \frac{9L_F^2}{2\mu} \\
		& \leq \frac{12\beta_2{L}_{\tilde{F}}^2{L}_{g_y}^2}{\alpha_1^2\mu^2}+\frac{200\beta_2 L_F^2{L}_{g_y}^2}{\alpha_2^2\mu^2}-  \frac{9L_F^2}{2\mu} \\
		& = \frac{(12{L}_{\tilde{F}}^2/\alpha_1^2+200L_F^2/\alpha_2^2){L}_{g_y}^2}{\mu^2}\beta_2-  \frac{9L_F^2}{2\mu}  \ . 
	\end{aligned}
\end{equation}
By letting this upper bound non-positive, we can get
\begin{equation}
	\begin{aligned}
		&  \frac{(12{L}_{\tilde{F}}^2/\alpha_1^2+200L_F^2/\alpha_2^2){L}_{g_y}^2}{\mu^2}\beta_2 \leq  \frac{9L_F^2}{2\mu}  \ ,  \\
		& \beta_2 \leq  \frac{9\mu L_F^2}{2(12{L}_{\tilde{F}}^2/\alpha_1^2+200L_F^2/\alpha_2^2){L}_{g_y}^2} \ . \\
	\end{aligned}
\end{equation}
As a result, by setting
\begin{equation}
	\begin{aligned}
		& \beta_1 \leq \min\Big\{ \frac{\beta_2\mu} {15L_FL_{y} }, \frac{\mu} {4{L}_{g_{y}}\sqrt{6{L}_{\tilde{F}}^2/\alpha_1^2+100L_F^2/\alpha_2^2}} \Big\},  \\
		& \beta_2 \leq  \frac{9\mu L_F^2}{2(12{L}_{\tilde{F}}^2/\alpha_1^2+200L_F^2/\alpha_2^2){L}_{g_y}^2} \ , \\
	\end{aligned}
\end{equation}
we can get
\begin{equation} 
	\begin{aligned}
		& \quad   \mathcal{L}_{t+1} - \mathcal{L}_{t}  \\
		& \leq - \frac{\eta\beta_1}{2} \mathbb{E}[\|\nabla F(\bar{  {x}}_{t})\|^2]   - \frac{\eta\beta_1L_F^2}{2}\mathbb{E}[\|\bar{y}_{t} - y^*(\bar{x}_{t})\|^2] +  \frac{3\eta\beta_1C_{g_{xy}}^2C_{f_y}^2}{\mu^2}(1-\frac{\mu}{L_{g_{y}}})^{2J}\\
		& \quad +\frac{24\alpha_1^2 \beta_1^3\eta^3\hat{C}_{\tilde{F}}^2{L}_{\tilde{F}}^2}{(1-\lambda)^4}+ \frac{24\alpha_2^2\beta_1\beta_2^2\eta^3\hat{C}_{g_y}^2{L}_{\tilde{F}}^2}{(1-\lambda)^4}+ \frac{400\alpha_1^2 \beta_1^3\eta^3\hat{C}_{\tilde{F}}^2 L_{g_y}^2L_F^2}{\mu^2(1-\lambda)^4} +  \frac{400 \alpha_2^2\beta_1\beta_2^2\eta^3\hat{C}_{g_y}^2L_F^2L_{g_y}^2}{\mu^2(1-\lambda)^4}\\
		& \quad +\frac{216 \beta_1^3\eta^3\hat{C}_{\tilde{F}}^2{L}_{\tilde{F}}^2}{(1-\lambda)^4}  +\frac{216\alpha_2^2\beta_1\beta_2^2\eta^3\hat{C}_{g_y}^2{L}_{\tilde{F}}^2}{\alpha_1^2(1-\lambda)^4} +\frac{3600\alpha_1^2 \beta_1^3\eta^3\hat{C}_{\tilde{F}}^2L_F^2{L}_{g_{y}}^2}{\alpha_2^2\mu^2(1-\lambda)^4}+\frac{3600\beta_1\beta_2^2\eta^3\hat{C}_{g_y}^2L_F^2{L}_{g_y}^2}{\mu^2(1-\lambda)^4}  \\
		& \quad   +\frac{3\beta_1\alpha_1\eta^2 \sigma_{\tilde{F}}^2}{K}  + \frac{ 50\beta_1\alpha_2\eta^2 \sigma^2 L_F^2}{\mu^2K}  \ .  \\
	\end{aligned}
\end{equation}
By summing $t$ from $0$ to $T-1$, we can get
\begin{equation}
	\begin{aligned}
		&  \quad \frac{1}{T}\sum_{t=0}^{T-1} \mathbb{E}[\|\nabla F(\bar{  {x}}_{t})\|^2] + L_F^2 \mathbb{E}[\|\bar{y}_{t} - y^*(\bar{x}_{t})\|^2] \\
		& \leq \frac{2(\mathcal{L}_{0} - \mathcal{L}_{T} )}{\eta\beta_1T} +  \frac{6C_{g_{xy}}^2C_{f_y}^2}{\mu^2}(1-\frac{\mu}{L_{g_{y}}})^{2J}\\
		& \quad +\frac{48\alpha_1^2 \beta_1^2\eta^2\hat{C}_{\tilde{F}}^2{L}_{\tilde{F}}^2}{(1-\lambda)^4}+ \frac{48\alpha_2^2\beta_2^2\eta^2\hat{C}_{g_y}^2{L}_{\tilde{F}}^2}{(1-\lambda)^4}+ \frac{800\alpha_1^2 \beta_1^2\eta^2\hat{C}_{\tilde{F}}^2 L_{g_y}^2L_F^2}{\mu^2(1-\lambda)^4} +  \frac{800 \alpha_2^2\beta_2^2\eta^2\hat{C}_{g_y}^2L_F^2L_{g_y}^2}{\mu^2(1-\lambda)^4}\\
		& \quad +\frac{432 \beta_1^2\eta^2\hat{C}_{\tilde{F}}^2{L}_{\tilde{F}}^2}{(1-\lambda)^4}  +\frac{432\alpha_2^2\beta_2^2\eta^2\hat{C}_{g_y}^2{L}_{\tilde{F}}^2}{\alpha_1^2(1-\lambda)^4} +\frac{7200\alpha_1^2 \beta_1^2\eta^2\hat{C}_{\tilde{F}}^2L_F^2{L}_{g_{y}}^2}{\alpha_2^2\mu^2(1-\lambda)^4}+\frac{7200\beta_2^2\eta^2\hat{C}_{g_y}^2L_F^2{L}_{g_y}^2}{\mu^2(1-\lambda)^4}  \\
		& \quad   +\frac{6\alpha_1\eta \sigma_{\tilde{F}}^2}{K}  + \frac{ 100\alpha_2\eta \sigma^2 L_F^2}{\mu^2K}  \ .  \\
	\end{aligned}
\end{equation}

Due to $\mathbb{E}[\|(\Delta_{0}^{\tilde{F}}-U_{0})\frac{1}{K}\mathbf{1}\|^2]  \leq \frac{2}{K^2}\mathbb{E}[\|\sum_{k=1}^{K}\nabla\tilde{F}^{(k)}(x_0^{(k)}, y_0^{(k)})\|^2] + \frac{2\alpha_1^2\eta^2}{K^2}\mathbb{E}[\|\sum_{k=1}^{K}\nabla\tilde{F}^{(k)}(x_0^{(k)}, y_0^{(k)}; \tilde{\xi}_0^{(k)})\|^2]   \leq    4\hat{C}_{\tilde{F}}^2$ and $\mathbb{E}[\|(\Delta_{0}^{g}-V_{0})\frac{1}{K}\mathbf{1}\|^2]  \leq 4\hat{C}_{g_y}^2$, we can get
\begin{equation}
	\begin{aligned}
		& \mathcal{L}_{0} ={ \mathbb{E}}[F(x_{0})] +\frac{6\beta_1{L}_{F}^2}{\beta_2\mu}\mathbb{E}[\|\bar{   {y}}_{0} -    {y}^{*}(\bar{   {x}}_{0})\| ^2 ]   +\frac{3\beta_1}{\alpha_1}\mathbb{E}[\|(\Delta_{0}^{\tilde{F}}-U_{0})\frac{1}{K}\mathbf{1}\|^2] + \frac{50 \beta_1{L}_{F}^2}{\alpha_2\mu^2}\mathbb{E}[\|(\Delta_{0}^{g}-V_{0})\frac{1}{K}\mathbf{1}\|^2]   \\
& \leq { \mathbb{E}}[F(x_{0})] +\frac{6\beta_1{L}_{F}^2}{\beta_2\mu}\mathbb{E}[\|\bar{   {y}}_{0} -    {y}^{*}(\bar{   {x}}_{0})\| ^2 ] +  \frac{12\beta_1 \hat{C}_{\tilde{F}}^2}{\alpha_1} + \frac{200 \beta_1{L}_{F}^2\hat{C}_{g_y}^2}{\alpha_2\mu^2} \ . \\
	\end{aligned}
\end{equation}
Finally, we can get
\begin{equation}
	\begin{aligned}
		&  \quad \frac{1}{T}\sum_{t=0}^{T-1} (\mathbb{E}[\|\nabla F(\bar{  {x}}_{t})\|^2] + L_F^2 \mathbb{E}[\|\bar{y}_{t} - y^*(\bar{x}_{t})\|^2] )\\
		& \leq \frac{2(F(x_0)- F(x_*))}{\eta\beta_1T}  +   \frac{12{L}_{F}^2}{\beta_2\mu\eta T}\mathbb{E}[\|\bar{   {y}}_{0} -    {y}^{*}(\bar{   {x}}_{0})\| ^2 ]  +  \frac{6C_{g_{xy}}^2C_{f_y}^2}{\mu^2}(1-\frac{\mu}{L_{g_{y}}})^{2J}\\
		& \quad +\frac{48\alpha_1^2 \beta_1^2\eta^2\hat{C}_{\tilde{F}}^2{L}_{\tilde{F}}^2}{(1-\lambda)^4}+ \frac{48\alpha_2^2\beta_2^2\eta^2\hat{C}_{g_y}^2{L}_{\tilde{F}}^2}{(1-\lambda)^4}+ \frac{800\alpha_1^2 \beta_1^2\eta^2\hat{C}_{\tilde{F}}^2 L_{g_y}^2L_F^2}{\mu^2(1-\lambda)^4} +  \frac{800 \alpha_2^2\beta_2^2\eta^2\hat{C}_{g_y}^2L_F^2L_{g_y}^2}{\mu^2(1-\lambda)^4}\\
		& \quad +\frac{432 \beta_1^2\eta^2\hat{C}_{\tilde{F}}^2{L}_{\tilde{F}}^2}{(1-\lambda)^4}  +\frac{432\alpha_2^2\beta_2^2\eta^2\hat{C}_{g_y}^2{L}_{\tilde{F}}^2}{\alpha_1^2(1-\lambda)^4} +\frac{7200\alpha_1^2 \beta_1^2\eta^2\hat{C}_{\tilde{F}}^2L_F^2{L}_{g_{y}}^2}{\alpha_2^2\mu^2(1-\lambda)^4}+\frac{7200\beta_2^2\eta^2\hat{C}_{g_y}^2L_F^2{L}_{g_y}^2}{\mu^2(1-\lambda)^4}  \\
		& \quad   +\frac{6\alpha_1\eta \sigma_{\tilde{F}}^2}{K}  + \frac{ 100\alpha_2\eta \sigma^2 L_F^2}{\mu^2K} +  \frac{24 \hat{C}_{\tilde{F}}^2}{\alpha_1\eta T} + \frac{400 {L}_{F}^2\hat{C}_{g_y}^2}{\alpha_2\mu^2 \eta T}\ .  \\
	\end{aligned}
\end{equation}

\newpage

\subsection{Proof of Theorem~\ref{theorem_vrdbo}}

\subsubsection{Characterization of  $F^{(k)}(x)$}
\begin{lemma}  \label{lemma_hypergrad_smooth_optimal_var}
	Given Assumptions~\ref{assumption_bi_strong},~\ref{assumption_variance},~\ref{assumption_upper_smooth_vr},~\ref{assumption_lower_smooth_vr},  the following inequalities hold. 
	\begin{equation}
		\begin{aligned}
			& \| \nabla F^{(k)}(x) - {\nabla} F^{(k)}(x, y)    \| \leq  {L}_F \|y - y^*(x)\|  \ ,  \\
			& \| \nabla F^{(k)}(x_1) - {\nabla} F^{(k)}(x_2)    \| \leq  {L}_{F}^{*} \|x_1 - x_2\|  \ ,  \\
			& \|y^*(x_1)-y^*(x_2)\| \leq L_y \|x_1 - x_2\| \ , 
		\end{aligned}
	\end{equation}
	where ${L}_F=\ell_{f_x}+\frac{\ell_{f_y}c_{g_{xy}}}{\mu}+\frac{c_{f_y}\ell_{g_{xy}}}{\mu}+\frac{\ell_{g_{yy}}c_{f_{y}}c_{g_{xy}}}{\mu^2}$ $L_F^*=L_F+\frac{L_Fc_{g_{xy}}}{\mu}$, $L_y=\frac{c_{g_{xy}}}{\mu}$.
\end{lemma}
These inequalities can also be proved by following Lemma~2.2 in \cite{ghadimi2018approximation} when given  the mean-square Lipschitz smoothness assumption, since $\|\mathbb{E}[a]\|^2\leq \mathbb{E}[\|a\|^2]$ holds for any random variable $a$. 

\subsubsection{Characterization of  $\nabla \tilde{F}^{(k)}(x, y)$}
\begin{lemma}  \label{lemma_hessian_bound_var}
	Given Assumptions~\ref{assumption_bi_strong},~\ref{assumption_variance},~\ref{assumption_upper_smooth_vr},~\ref{assumption_lower_smooth_vr},  the following inequalities hold. 
	\begin{equation}
		\begin{aligned}
			& \mathbb{E}\Big[\frac{J}{\ell_{g_{y}}}\prod_{j=1}^{\tilde{J}}(I-\frac{1}{\ell_{g_{y}}}\nabla_{yy}^2g^{(k)}(x, y; \zeta_j))\Big] = \frac{1}{\ell_{g_{y}}} \sum_{j=0}^{J-1}\Big(I-\frac{1}{\ell_{g_{y}}}\nabla_{yy}^2g^{(k)}(x, y)\Big)^{j}  \ , \\
			& \Big\|\Big(\nabla_{yy}^2g^{(k)}(x, y)\Big)^{-1}- \mathbb{E}\Big[\frac{J}{\ell_{g_{y}}}\prod_{j=1}^{\tilde{J}}(I-\frac{1}{\ell_{g_{y}}}\nabla_{yy}^2g^{(k)}(x, y; \zeta_j))\Big]\Big\| \leq \frac{1}{\mu}(1-\frac{\mu}{\ell_{g_{y}}})^{J}  \ , \\
			&  \mathbb{E}\Big [\Big\|\Big(\nabla_{yy}^2g^{(k)}(x, y)\Big)^{-1}- \frac{J}{\ell_{g_{y}}}\prod_{j=1}^{\tilde{J}}(I-\frac{1}{\ell_{g_{y}}}\nabla_{yy}^2g^{(k)}(x, y; \zeta_j))\Big\| \Big]\leq \frac{2}{\mu}  \ , \\
			& \mathbb{E}\Big[\Big\| \frac{J}{\ell_{g_{y}}}\prod_{j=1}^{\tilde{J}}(I-\frac{1}{\ell_{g_{y}}}\nabla_{yy}^2g^{(k)}(x, y; \zeta_j))\Big\| \Big] \leq \frac{1}{\mu}  \ . \\
		\end{aligned}
	\end{equation}
\end{lemma}
These inequalities can  be proved by following Lemma~3.2 in \cite{ghadimi2018approximation} when given  the mean-square Lipschitz smoothness assumption since $\|\mathbb{E}[a]\|^2\leq \mathbb{E}[\|a\|^2]$ holds for any random variable $a$.

\begin{lemma} (Bias) \label{lemma_hypergrad_bias_var}
	Given Assumptions~\ref{assumption_bi_strong},~\ref{assumption_variance},~\ref{assumption_upper_smooth_vr},~\ref{assumption_lower_smooth_vr},  the approximation error of $\nabla \tilde{F}^{(k)}(x, y)$ for $\nabla F^{(k)}(x, y)$ can be bounded as follows:
	\begin{equation}
		\begin{aligned}
			& \|\nabla F^{(k)}(x, y)   -\nabla \tilde{F}^{(k)}(x, y)  \| \leq  \frac{c_{g_{xy}}c_{f_y}}{\mu}(1-\frac{\mu}{\ell_{g_{y}}})^{J} \ . 
		\end{aligned}
	\end{equation}
\end{lemma}
These inequalities can  be proved by following Lemma~3.2 in \cite{ghadimi2018approximation} when given  the mean-square Lipschitz smoothness assumption since $\|\mathbb{E}[a]\|^2\leq \mathbb{E}[\|a\|^2]$ holds for any random variable $a$. 

\newpage
\begin{lemma} (Variance) \label{lemma_hypergrad_var_var}
	Given Assumptions~\ref{assumption_bi_strong},~\ref{assumption_variance},~\ref{assumption_upper_smooth_vr},~\ref{assumption_lower_smooth_vr},   the variance of the stochastic hypergradient can be bounded as follows: 
	\begin{equation}
		\begin{aligned}
			& \mathbb{E}[\|{\nabla} \tilde{F} ^{(k)}(x, y) -  {\nabla} \tilde{F} ^{(k)}(x, y; \tilde{\xi})  \|^2] \leq \sigma_{\tilde{F}}^2 \ ,
		\end{aligned}
	\end{equation}
	where $\sigma_{\tilde{F}}^2=4\sigma^2 +  \frac{4c_{f_y}^2\sigma^2}{\mu^2}  + \frac{4c_{g_{xy}}^2\sigma^2}{\mu^2} + \frac{16c_{g_{xy}}^2c_{f_y}^2}{\mu^2}$. 
\end{lemma}

\begin{proof}
	\begin{equation}
		\begin{aligned}
			&  \quad \mathbb{E}[\|{\nabla} \tilde{F} ^{(k)}(x, y) -  {\nabla} \tilde{F} ^{(k)}(x, y; \tilde{\xi})  \|^2] \\
			& =  \mathbb{E}\Big[\Big\|\nabla_x f^{(k)}(x, y) -  \nabla_{xy}^2g^{(k)}(x, y) \mathbb{E}\Big[\frac{J}{\ell_{g_{y}}}\prod_{j=1}^{\tilde{J}}(I-\frac{1}{\ell_{g_{y}}}\nabla_{yy}^2g^{(k)}(x, y; \zeta_j))\Big]\nabla_y f^{(k)}(x, y) \\
			& \quad -  \nabla_x f^{(k)}(x, y; \xi) +  \nabla_{xy}^2g^{(k)}(x, y; \zeta_0) \Big(\frac{J}{\ell_{g_{y}}}\prod_{j=1}^{\tilde{J}}(I-\frac{1}{\ell_{g_{y}}}\nabla_{yy}^2g^{(k)}(x, y; \zeta_j))\Big)\nabla_y f^{(k)}(x, y; \xi)  \Big\|^2\Big] \\
			& = \mathbb{E}\Big[\Big\|\nabla_x f^{(k)}(x, y) - \nabla_x f^{(k)}(x, y; \xi) \\
			& \quad + \nabla_{xy}^2g^{(k)}(x, y) \mathbb{E}\Big[\frac{J}{\ell_{g_{y}}}\prod_{j=1}^{\tilde{J}}(I-\frac{1}{\ell_{g_{y}}}\nabla_{yy}^2g^{(k)}(x, y; \zeta_j))\Big]\nabla_y f^{(k)}(x, y)  \\
			& \quad - \nabla_{xy}^2g^{(k)}(x, y; \zeta_0) \mathbb{E}\Big[\frac{J}{\ell_{g_{y}}}\prod_{j=1}^{\tilde{J}}(I-\frac{1}{\ell_{g_{y}}}\nabla_{yy}^2g^{(k)}(x, y; \zeta_j))\Big]\nabla_y f^{(k)}(x, y)  \\
			& \quad + \nabla_{xy}^2g^{(k)}(x, y; \zeta_0) \mathbb{E}\Big[\frac{J}{\ell_{g_{y}}}\prod_{j=1}^{\tilde{J}}(I-\frac{1}{\ell_{g_{y}}}\nabla_{yy}^2g^{(k)}(x, y; \zeta_j))\Big]\nabla_y f^{(k)}(x, y)  \\
			& \quad - \nabla_{xy}^2g^{(k)}(x, y; \zeta_0) \mathbb{E}\Big[\frac{J}{\ell_{g_{y}}}\prod_{j=1}^{\tilde{J}}(I-\frac{1}{\ell_{g_{y}}}\nabla_{yy}^2g^{(k)}(x, y; \zeta_j))\Big]\nabla_y f^{(k)}(x, y; \xi)  \\
			& \quad + \nabla_{xy}^2g^{(k)}(x, y; \zeta_0) \mathbb{E}\Big[\frac{J}{\ell_{g_{y}}}\prod_{j=1}^{\tilde{J}}(I-\frac{1}{\ell_{g_{y}}}\nabla_{yy}^2g^{(k)}(x, y; \zeta_j))\Big]\nabla_y f^{(k)}(x, y; \xi)  \\
			& \quad  - \nabla_{xy}^2g^{(k)}(x, y; \zeta_0) \Big(\frac{J}{\ell_{g_{y}}}\prod_{j=1}^{\tilde{J}}(I-\frac{1}{\ell_{g_{y}}}\nabla_{yy}^2g^{(k)}(x, y; \zeta_j))\Big)\nabla_y f^{(k)}(x, y; \xi)  \Big\|^2 \Big]\\
			& \leq 4\sigma^2 + 4 \mathbb{E}\Big[\Big\|\Big(\nabla_{xy}^2g^{(k)}(x, y) - \nabla_{xy}^2g^{(k)}(x, y; \zeta_0)\Big) \mathbb{E}\Big[\frac{J}{\ell_{g_{y}}}\prod_{j=1}^{\tilde{J}}(I-\frac{1}{\ell_{g_{y}}}\nabla_{yy}^2g^{(k)}(x, y; \zeta_j))\Big]\nabla_y f^{(k)}(x, y)  \Big\|^2 \Big]\\
			& \quad + 4\mathbb{E}\Big[\Big\|\nabla_{xy}^2g^{(k)}(x, y; \zeta_0) \mathbb{E}\Big[\frac{J}{\ell_{g_{y}}}\prod_{j=1}^{\tilde{J}}(I-\frac{1}{\ell_{g_{y}}}\nabla_{yy}^2g^{(k)}(x, y; \zeta_j))\Big]\Big(\nabla_y f^{(k)}(x, y)- \nabla_y f^{(k)}(x, y; \xi)\Big)\Big\|^2 \Big]\\
			& \quad + 4\mathbb{E}\Big[\Big\|\nabla_{xy}^2g^{(k)}(x, y; \zeta_0) \Big(\mathbb{E}\Big[\frac{J}{\ell_{g_{y}}}\prod_{j=1}^{\tilde{J}}(I-\frac{1}{\ell_{g_{y}}}\nabla_{yy}^2g^{(k)}(x, y; \zeta_j))\Big]\\
			& \quad \quad -\frac{J}{\ell_{g_{y}}}\prod_{j=1}^{\tilde{J}}(I-\frac{1}{\ell_{g_{y}}}\nabla_{yy}^2g^{(k)}(x, y; \zeta_j))\Big)\nabla_y f^{(k)}(x, y; \xi)  \Big \|^2  \Big]\\
			&\overset{(s_1)} \leq 4\sigma^2 +  \frac{4c_{f_y}^2\sigma^2}{\mu^2}  + \frac{4c_{g_{xy}}^2\sigma^2}{\mu^2} + \frac{16c_{g_{xy}}^2c_{f_y}^2}{\mu^2} \ , 
		\end{aligned}
	\end{equation}
	where $(s_1)$ holds due to Lemma~\ref{lemma_hessian_bound_var},  Assumptions~\ref{assumption_variance},~\ref{assumption_upper_smooth_vr},~\ref{assumption_lower_smooth_vr}, and the following inequality.
	\begin{equation}
		\begin{aligned}
			& \quad \Big\|\mathbb{E}\Big[\frac{J}{\ell_{g_{y}}}\prod_{j=1}^{\tilde{J}}(I-\frac{1}{\ell_{g_{y}}}\nabla_{yy}^2g^{(k)}(x, y; \zeta_j))\Big]\Big\| = \Big\|\frac{1}{\ell_{g_{y}}} \sum_{j=0}^{J-1}\Big(I-\frac{1}{\ell_{g_{y}}}\nabla_{yy}^2g^{(k)}(x, y)\Big)^{j} \Big\| \\
			& \leq \frac{1}{\ell_{g_{y}}} \sum_{j=0}^{J-1} \Big\|\Big(I-\frac{1}{\ell_{g_{y}}}\nabla_{yy}^2g^{(k)}(x, y)\Big)^{j} \Big\| \\
			& \overset{(s_1)}\leq \frac{1}{\ell_{g_{y}}} \sum_{j=0}^{J-1} \Big(1-\frac{\mu}{\ell_{g_{y}}}\Big)^{j}  \\
			& \leq \frac{1}{\mu} \ , 
		\end{aligned}
	\end{equation}
	where $(s_1)$ holds due to Assumption~\ref{assumption_bi_strong}. 
\end{proof}

\begin{lemma} (Smoothness) \cite{khanduri2021near} \label{lemma_hypergrad_smooth_var}
	Given Assumptions~\ref{assumption_bi_strong},~\ref{assumption_variance},~\ref{assumption_upper_smooth_vr},~\ref{assumption_lower_smooth_vr},  the approximated hypergradient ${\nabla}\tilde{F}^{(k)}(x, y; \tilde{xi})$ is ${L}_{\tilde{F}}$-Lipschitz continuous, i.e., 
	\begin{equation}
		\begin{aligned}
			&  \mathbb{E}[\|{\nabla}\tilde{F}^{(k)}(x_1, y_1; \tilde{\xi})-{\nabla}\tilde{F}^{(k)}(x_{2}, y_{2}; \tilde{\xi})\|^2 ]\leq {L}_{\tilde{F}}^2(\|x_1-x_{2}\|^2+\|y_1-y_{2}\|^2),
		\end{aligned}
	\end{equation}
	where ${L}_{\tilde{F}}^2= 2\ell_{f_x}^2 + \frac{6(c_{g_{xy}}^2\ell_{f_y}^2+c_{f_y}^2\ell_{g_{xy}}^2)J}{2\mu\ell_{g_{y}} -\mu^2} + \frac{6c_{g_{xy}}^2c_{f_y}^2\ell_{g_{y}}^2J^3}{(2\mu\ell_{g_{y}} -\mu^2)(\ell_{g_{y}}-\mu)^2}$, for any $(x_1, y_1)$, $(x_2, y_2)\in\mathbb{R}^{d_x}\times \mathbb{R}^{d_y}$. 
\end{lemma}

\subsubsection{Characterization of Gradient Estimators}

\begin{lemma} \label{lemma_hyper_storm_var_mean}
	Given Assumptions~\ref{assumption_graph},~\ref{assumption_bi_strong},~\ref{assumption_variance},~\ref{assumption_upper_smooth_vr},~\ref{assumption_lower_smooth_vr},  the following inequality holds.
	\begin{equation}
		\begin{aligned}
			& \quad \mathbb{E} [ \|(\Delta_t^{\tilde{F}}-   U_t) \frac{1}{K} \mathbf{1} \|^2 ]\leq (1-\alpha_1 \eta^2)\mathbb{E} [ \|(\Delta_{t-1}^{\tilde{F}}-  U_{t-1})\frac{1}{K} \mathbf{1}  \|^2 ] +2L_{\tilde{F}}^2\frac{1}{K^2} \mathbb{E}[\|X_{t}-X_{t-1} \|_F^2]\\
			& \quad +2L_{\tilde{F}}^2\frac{1}{K^2} \mathbb{E}[\|Y_{t}-Y_{t-1} \|_F^2] + \frac{2\alpha_1^2 \eta^4 \sigma_{\tilde{F}}^2 }{K} \ . \\
		\end{aligned}
	\end{equation}
	
\end{lemma}

\begin{proof}
	\begin{equation}
		\begin{aligned}
			& \quad\mathbb{E} [ \|(\Delta_t^{\tilde{F}}-   U_t) \frac{1}{K} \mathbf{1} \|^2 ]\\
			& =\mathbb{E} [ \| (\Delta_t^{\tilde{F}}-   (1-\alpha_1 \eta^2)(U_{t-1}- \Delta_{t-1}^{\tilde{F}_{\tilde{\xi}_t}}) -  \Delta_t^{\tilde{F}_{\tilde{\xi}_t}})\frac{1}{K} \mathbf{1}   \|^2 ]\\
			& = \mathbb{E} [ \| ( (1-\alpha_1 \eta^2)(\Delta_{t-1}^{\tilde{F}}-  U_{t-1}) +(1-\alpha_1 \eta^2)(\Delta_{t-1}^{\tilde{F}_{\tilde{\xi}_t}}- \Delta_{t-1}^{\tilde{F}})- (\Delta_{t}^{\tilde{F}_{\tilde{\xi}_t}}- \Delta_{t}^{\tilde{F}} ) )\frac{1}{K} \mathbf{1}  \|_F^2 ]\\
			& \overset{(s_1)} =  (1-\alpha_1 \eta^2)^2\mathbb{E} [ \| (\Delta_{t-1}^{\tilde{F}}-  U_{t-1})\frac{1}{K} \mathbf{1} \|^2 ] +\mathbb{E} [ \|(1-\alpha_1 \eta^2)(\Delta_{t-1}^{\tilde{F}_{\tilde{\xi}_t}}- \Delta_{t-1}^{\tilde{F}}- \Delta_{t}^{\tilde{F}_{\tilde{\xi}_t}}+ \Delta_{t}^{\tilde{F}})\frac{1}{K} \mathbf{1} \\
			& \quad \quad - \alpha_1 \eta^2(\Delta_{t}^{\tilde{F}_{\tilde{\xi}_t}}- \Delta_{t}^{\tilde{F}}) \frac{1}{K} \mathbf{1} \|^2 ]\\
			& \leq (1-\alpha_1 \eta^2)^2\mathbb{E} [ \|( \Delta_{t-1}^{\tilde{F}}-  U_{t-1})\frac{1}{K} \mathbf{1}   \|^2 ] +2(1-\alpha_1 \eta^2)^2\frac{1}{K^2}\mathbb{E} [ \|\Delta_{t-1}^{\tilde{F}}-\Delta_{t-1}^{\tilde{F}_{\tilde{\xi}_t}}- \Delta_{t}^{\tilde{F}}+ \Delta_{t}^{\tilde{F}_{\tilde{\xi}_t}}  \|_F^2 ]\\
			& \quad + 2\alpha_1^2 \eta^4\frac{1}{K^2}\mathbb{E} [ \| \Delta_{t}^{\tilde{F}}- \Delta_{t}^{\tilde{F}_{\tilde{\xi}_t}} \|_F^2 ]\\
			& \overset{(s_2)}\leq (1-\alpha_1 \eta^2)^2\mathbb{E} [ \| (\Delta_{t-1}^{\tilde{F}}-  U_{t-1})\frac{1}{K} \mathbf{1} \|^2 ] +2(1-\alpha_1 \eta^2)^2\frac{1}{K^2}\mathbb{E} [ \|\Delta_{t}^{\tilde{F}_{\tilde{\xi}_t}}-\Delta_{t-1}^{\tilde{F}_{\tilde{\xi}_t}}  \|_F^2 ] + \frac{2\alpha_1^2 \eta^4 \sigma_{\tilde{F}}^2 }{K}\\
			&\overset{(s_3)}\leq (1-\alpha_1 \eta^2)\mathbb{E} [ \|(\Delta_{t-1}^{\tilde{F}}-  U_{t-1})\frac{1}{K} \mathbf{1}  \|^2 ] +2(1-\alpha_1 \eta^2)^2L_{\tilde{F}}^2\frac{1}{K^2} \mathbb{E}[\|X_{t}-X_{t-1} \|_F^2]\\
			& \quad +2(1-\alpha_1 \eta^2)^2L_{\tilde{F}}^2\frac{1}{K^2} \mathbb{E}[\|Y_{t}-Y_{t-1} \|_F^2] + \frac{2\alpha_1^2 \eta^4 \sigma_{\tilde{F}}^2 }{K}\\
			& \leq (1-\alpha_1 \eta^2)\mathbb{E} [ \|(\Delta_{t-1}^{\tilde{F}}-  U_{t-1})\frac{1}{K} \mathbf{1}  \|^2 ] +2L_{\tilde{F}}^2\frac{1}{K^2} \mathbb{E}[\|X_{t}-X_{t-1} \|_F^2]\\
			& \quad +2L_{\tilde{F}}^2\frac{1}{K^2} \mathbb{E}[\|Y_{t}-Y_{t-1} \|_F^2] + \frac{2\alpha_1^2 \eta^4 \sigma_{\tilde{F}}^2 }{K}\ , \\
		\end{aligned}
	\end{equation}
	where $(s_1)$ holds due to  $\Delta_{t}^{\tilde{F}}= \mathbb{E}[\Delta_{t}^{\tilde{F}_{\tilde{\xi}_t}}]$ and $\Delta_{t-1}^{\tilde{F}}= \mathbb{E}[\Delta_{t-1}^{\tilde{F}_{\tilde{\xi}_t}}]$,  
	$(s_2)$ holds due to Lemma~\ref{lemma_hypergrad_var_var}, $(s_3)$ holds due to Lemma~\ref{lemma_hypergrad_smooth_var}. 
\end{proof}

\begin{lemma} \label{lemma_lower_storm_var_mean}
	Given Assumptions~\ref{assumption_graph},~\ref{assumption_bi_strong},~\ref{assumption_variance},~\ref{assumption_upper_smooth_vr},~\ref{assumption_lower_smooth_vr},   the following inequality holds.
	\begin{equation}
		\begin{aligned}
			&  \mathbb{E} [ \|(\Delta_t^{g}-V_t ) \frac{1}{K}\mathbf{1}\|^2 ]\leq (1-\alpha_2\eta^2)\mathbb{E} [ \| (\Delta_{t-1}^{g}-  V_{t-1})\frac{1}{K}\mathbf{1} \|^2 ]+2{\ell}_{g_{y}}^2\frac{1}{K^2} \mathbb{E}[\|X_t - X_{t-1}\|_F^2]\\
			& \quad +2{\ell}_{g_{y}}^2\frac{1}{K^2} \mathbb{E}[\|Y_{t}-Y_{t-1}\|_F^2] + \frac{2\alpha_2^2\eta^4 \sigma^2}{K}\ . \\
		\end{aligned}
	\end{equation}
	
\end{lemma}

\begin{proof}
	\begin{equation}
		\begin{aligned}
			& \quad\mathbb{E} [ \|(\Delta_t^{g}-V_t ) \frac{1}{K}\mathbf{1}\|^2 ]\\
			& = \mathbb{E} [ \| (\Delta_t^{g} -  (1-\alpha_2\eta^2)(V_{t-1}-\Delta_{t-1}^{g_{\zeta_t}}) -  \Delta_t^{g_{\zeta_t}} )\frac{1}{K}\mathbf{1} \|^2 ]\\
			& = \mathbb{E} [ \|  (1-\alpha_2\eta^2)(\Delta_{t-1}^{g}-  V_{t-1})\frac{1}{K}\mathbf{1} +(1-\alpha_2\eta^2)(\Delta_{t-1}^{g_{\zeta_t}}-\Delta_{t-1}^{g} )- (\Delta_t^{g_{\zeta_t}}- \Delta_t^{g} )\frac{1}{K}\mathbf{1} \|^2 ]\\
			& = \mathbb{E} [ \|  (1-\alpha_2\eta^2)(\Delta_{t-1}^{g}-  V_{t-1})\frac{1}{K}\mathbf{1} +(1-\alpha_2\eta^2)(\Delta_{t-1}^{g_{\zeta_t}}-\Delta_{t-1}^{g}- \Delta_{t}^{g_{\zeta_t}}+\Delta_{t}^{g}  ) \frac{1}{K}\mathbf{1}- \alpha_2\eta^2(\Delta_t^{g_{\zeta_t}}- \Delta_t^{g} ) \frac{1}{K}\mathbf{1}\|^2 ]\\
			&\leq (1-\alpha_2\eta^2)^2\mathbb{E} [ \| (\Delta_{t-1}^{g}-  V_{t-1})\frac{1}{K}\mathbf{1} \|^2 ] \\
			& \quad + 2(1-\alpha_2\eta^2)^2\frac{1}{K^2}\mathbb{E} [ \| \Delta_{t-1}^{g_{\zeta_t}}-\Delta_{t-1}^{g}- \Delta_{t}^{g_{\zeta_t}}+\Delta_{t}^{g}   \|_F^2 ] +2\alpha_2^2\eta^4\frac{1}{K^2}\mathbb{E} [ \|\Delta_t^{g}- \Delta_t^{g_{\zeta_t}} \|_F^2 ]\\
			&\leq (1-\alpha_2\eta^2)^2\mathbb{E} [ \| (\Delta_{t-1}^{g}-  V_{t-1})\frac{1}{K}\mathbf{1} \|^2 ] + 2(1-\alpha_2\eta^2)^2\frac{1}{K^2}\mathbb{E} [ \| \Delta_{t-1}^{g_{\zeta_t}}- \Delta_{t}^{g_{\zeta_t}}\|_F^2 ] +2\alpha_2^2\eta^4\frac{1}{K^2}\mathbb{E} [ \|\Delta_t^{g}- \Delta_t^{g_{\zeta_t}}\|_F^2 ]\\
			&\overset{(s_1)}\leq (1-\alpha_2\eta^2)\mathbb{E} [ \| (\Delta_{t-1}^{g}-  V_{t-1})\frac{1}{K}\mathbf{1} \|^2 ]+2{\ell}_{g_{y}}^2\frac{1}{K^2} \mathbb{E}[\|X_t - X_{t-1}\|_F^2]\\
			& \quad +2{\ell}_{g_{y}}^2\frac{1}{K^2} \mathbb{E}[\|Y_{t}-Y_{t-1}\|_F^2] + \frac{2\alpha_2^2\eta^4 \sigma^2}{K} \ , \\
		\end{aligned}
	\end{equation}
	where $(s_1)$ holds due to Assumption~\ref{assumption_lower_smooth_vr} and Assumption~\ref{assumption_variance}. 
\end{proof}

\begin{lemma} \label{lemma_hyper_storm_var}
	Given Assumptions~\ref{assumption_graph},~\ref{assumption_bi_strong},~\ref{assumption_variance},~\ref{assumption_upper_smooth_vr},~\ref{assumption_lower_smooth_vr},  the following inequality holds.
	\begin{equation}
		\begin{aligned}
			& \quad \frac{1}{K}\mathbb{E} [ \|\Delta_t^{\tilde{F}}-   U_t \|_F^2 ]\leq (1-\alpha_1 \eta^2)\frac{1}{K}\mathbb{E} [ \|\Delta_{t-1}^{\tilde{F}}-  U_{t-1}  \|_F^2 ] +2(1-\alpha_1 \eta^2)^2L_{\tilde{F}}^2\frac{1}{K} \mathbb{E}[\|X_{t}-X_{t-1} \|_F^2]\\
			& \quad +2(1-\alpha_1 \eta^2)^2L_{\tilde{F}}^2\frac{1}{K} \mathbb{E}[\|Y_{t}-Y_{t-1} \|_F^2] + 2\alpha_1^2 \eta^4 \sigma_{\tilde{F}}^2 \ . \\
		\end{aligned}
	\end{equation}
	
\end{lemma}

\begin{proof}
	\begin{equation}
		\begin{aligned}
			& \quad\frac{1}{K} \mathbb{E} [ \|\Delta_t^{\tilde{F}}-   U_t  \|_F^2 ]\\
			& =\frac{1}{K} \mathbb{E} [ \| \Delta_t^{\tilde{F}}-   (1-\alpha_1 \eta^2)(U_{t-1}- \Delta_{t-1}^{\tilde{F}_{\tilde{\xi}_t}}) -  \Delta_t^{\tilde{F}_{\tilde{\xi}_t}}   \|_F^2 ]\\
			& = \frac{1}{K}\mathbb{E} [ \|  (1-\alpha_1 \eta^2)(\Delta_{t-1}^{\tilde{F}}-  U_{t-1}) +(1-\alpha_1 \eta^2)(\Delta_{t-1}^{\tilde{F}_{\tilde{\xi}_t}}- \Delta_{t-1}^{\tilde{F}})- (\Delta_{t}^{\tilde{F}_{\tilde{\xi}_t}}- \Delta_{t}^{\tilde{F}} )   \|_F^2 ]\\
			& \overset{(s_1)} =  (1-\alpha_1 \eta^2)^2\frac{1}{K}\mathbb{E} [ \| \Delta_{t-1}^{\tilde{F}}-  U_{t-1} \|_F^2 ] +\frac{1}{K}\mathbb{E} [ \|(1-\alpha_1 \eta^2)(\Delta_{t-1}^{\tilde{F}_{\tilde{\xi}_t}}- \Delta_{t-1}^{\tilde{F}}- \Delta_{t}^{\tilde{F}_{\tilde{\xi}_t}}+ \Delta_{t}^{\tilde{F}}) \\
			& \quad \quad - \alpha_1 \eta^2(\Delta_{t}^{\tilde{F}_{\tilde{\xi}_t}}- \Delta_{t}^{\tilde{F}})  \|_F^2 ]\\
			& \leq (1-\alpha_1 \eta^2)^2\frac{1}{K}\mathbb{E} [ \| \Delta_{t-1}^{\tilde{F}}-  U_{t-1}   \|_F^2 ] +2(1-\alpha_1 \eta^2)^2\frac{1}{K}\mathbb{E} [ \|\Delta_{t-1}^{\tilde{F}}-\Delta_{t-1}^{\tilde{F}_{\tilde{\xi}_t}}- \Delta_{t}^{\tilde{F}}+ \Delta_{t}^{\tilde{F}_{\tilde{\xi}_t}}  \|_F^2 ]\\
			& \quad + 2\alpha_1^2 \eta^4\frac{1}{K}\mathbb{E} [ \| \Delta_{t}^{\tilde{F}}- \Delta_{t}^{\tilde{F}_{\tilde{\xi}_t}} \|_F^2 ]\\
			& \overset{s_2}\leq (1-\alpha_1 \eta^2)^2\frac{1}{K}\mathbb{E} [ \| \Delta_{t-1}^{\tilde{F}}-  U_{t-1} \|_F^2 ] +2(1-\alpha_1 \eta^2)^2\frac{1}{K}\mathbb{E} [ \|\Delta_{t}^{\tilde{F}_{\tilde{\xi}_t}}-\Delta_{t-1}^{\tilde{F}_{\tilde{\xi}_t}}  \|_F^2 ] + 2\alpha_1^2 \eta^4\sigma_{\tilde{F}}^2\\
			&\overset{(s_3)}\leq (1-\alpha_1 \eta^2)\frac{1}{K}\mathbb{E} [ \|\Delta_{t-1}^{\tilde{F}}-  U_{t-1}  \|_F^2 ] +2(1-\alpha_1 \eta^2)^2L_{\tilde{F}}^2\frac{1}{K} \mathbb{E}[\|X_{t}-X_{t-1} \|_F^2]\\
			& \quad +2(1-\alpha_1 \eta^2)^2L_{\tilde{F}}^2\frac{1}{K} \mathbb{E}[\|Y_{t}-Y_{t-1} \|_F^2] + 2\alpha_1^2 \eta^4 \sigma_{\tilde{F}}^2 \ , \\
		\end{aligned}
	\end{equation}
	where $(s_1)$ holds due to  $\Delta_{t}^{\tilde{F}}= \mathbb{E}[\Delta_{t}^{\tilde{F}_{\tilde{\xi}_t}}]$ and $\Delta_{t-1}^{\tilde{F}}= \mathbb{E}[\Delta_{t-1}^{\tilde{F}_{\tilde{\xi}_t}}]$,  
	$(s_2)$ holds due to Lemma~\ref{lemma_hypergrad_var_var}, $(s_3)$ holds due to Lemma~\ref{lemma_hypergrad_smooth_var}. 
\end{proof}

\begin{lemma} \label{lemma_lower_storm_var}
	Given Assumptions~\ref{assumption_graph},~\ref{assumption_bi_strong},~\ref{assumption_variance},~\ref{assumption_upper_smooth_vr},~\ref{assumption_lower_smooth_vr},   the following inequality holds.
	\begin{equation}
		\begin{aligned}
			&  \frac{1}{K}\mathbb{E} [ \|\Delta_t^{g}-V_t \|_F^2 ]\leq (1-\alpha_2\eta^2)\frac{1}{K}\mathbb{E} [ \| \Delta_{t-1}^{g}-  V_{t-1}\|_F^2 ]+2(1-\alpha_2\eta^2)^2{\ell}_{g_{y}}^2\frac{1}{K} \mathbb{E}[\|X_t - X_{t-1}\|_F^2]\\
			& \quad +2(1-\alpha_2\eta^2)^2{\ell}_{g_{y}}^2\frac{1}{K} \mathbb{E}[\|Y_{t}-Y_{t-1}\|_F^2] + 2\alpha_2^2\eta^4 \sigma^2 \ . \\
		\end{aligned}
	\end{equation}
	
\end{lemma}

\begin{proof}
	\begin{equation}
		\begin{aligned}
			& \quad \frac{1}{K}\mathbb{E} [ \|\Delta_t^{g}-V_t \|_F^2 ]\\
			& = \frac{1}{K}\mathbb{E} [ \| \Delta_t^{g} -  (1-\alpha_2\eta^2)(V_{t-1}-\Delta_{t-1}^{g_{\zeta_t}}) -  \Delta_t^{g_{\zeta_t}}  \|_F^2 ]\\
			& = \frac{1}{K}\mathbb{E} [ \|  (1-\alpha_2\eta^2)(\Delta_{t-1}^{g}-  V_{t-1}) +(1-\alpha_2\eta^2)(\Delta_{t-1}^{g_{\zeta_t}}-\Delta_{t-1}^{g} )- (\Delta_t^{g_{\zeta_t}}- \Delta_t^{g} ) \|_F^2 ]\\
			& = \frac{1}{K}\mathbb{E} [ \|  (1-\alpha_2\eta^2)(\Delta_{t-1}^{g}-  V_{t-1}) +(1-\alpha_2\eta^2)(\Delta_{t-1}^{g_{\zeta_t}}-\Delta_{t-1}^{g}- \Delta_{t}^{g_{\zeta_t}}+\Delta_{t}^{g}  ) - \alpha_2\eta^2(\Delta_t^{g_{\zeta_t}}- \Delta_t^{g} ) \|_F^2 ]\\
			&\leq (1-\alpha_2\eta^2)^2\frac{1}{K}\mathbb{E} [ \| \Delta_{t-1}^{g}-  V_{t-1} \|_F^2 ] \\
			& \quad + 2(1-\alpha_2\eta^2)^2\frac{1}{K}\mathbb{E} [ \| \Delta_{t-1}^{g_{\zeta_t}}-\Delta_{t-1}^{g}- \Delta_{t}^{g_{\zeta_t}}+\Delta_{t}^{g}   \|_F^2 ] +2\alpha_2^2\eta^4\frac{1}{K}\mathbb{E} [ \|\Delta_t^{g}- \Delta_t^{g_{\zeta_t}} \|_F^2 ]\\
			&\leq (1-\alpha_2\eta^2)^2\frac{1}{K}\mathbb{E} [ \| \Delta_{t-1}^{g}-  V_{t-1} \|_F^2 ] + 2(1-\alpha_2\eta^2)^2\frac{1}{K}\mathbb{E} [ \| \Delta_{t-1}^{g_{\zeta_t}}- \Delta_{t}^{g_{\zeta_t}}\|_F^2 ] +2\alpha_2^2\eta^4\frac{1}{K}\mathbb{E} [ \|\Delta_t^{g}- \Delta_t^{g_{\zeta_t}}\|_F^2 ]\\
			&\overset{(s_1)}\leq (1-\alpha_2\eta^2)\frac{1}{K}\mathbb{E} [ \| \Delta_{t-1}^{g}-  V_{t-1}\|_F^2 ]+2(1-\alpha_2\eta^2)^2{\ell}_{g_{y}}^2\frac{1}{K} \mathbb{E}[\|X_t - X_{t-1}\|_F^2]\\
			& \quad +2(1-\alpha_2\eta^2)^2{\ell}_{g_{y}}^2\frac{1}{K} \mathbb{E}[\|Y_{t}-Y_{t-1}\|_F^2] + 2\alpha_2^2\eta^4 \sigma^2 \ , \\
		\end{aligned}
	\end{equation}
	where $(s_1)$ holds due to Assumption~\ref{assumption_lower_smooth_vr}and Assumption~\ref{assumption_variance}. 
\end{proof}

\subsubsection{Characterization of Consensus Errors}
\begin{lemma}
	Given Assumptions~\ref{assumption_graph},~\ref{assumption_bi_strong},~\ref{assumption_variance},~\ref{assumption_upper_smooth_vr},~\ref{assumption_lower_smooth_vr},  the following inequality holds. 
	\begin{equation}
		\begin{aligned}
			&  \frac{1}{K}\mathbb{E}[\|Z_{t}^{\tilde{F}} - \bar{Z}_{t}^{\tilde{F}}\|_F^2]\leq \lambda\frac{1}{K}\mathbb{E}[\|Z_{t-1}^{\tilde{F}} - \bar{Z}_{t-1}^{\tilde{F}} \|_F^2]  + \frac{2\ell_{\tilde{F}}^2}{1-\lambda}\frac{1}{K}\mathbb{E}[\|X_{t} - X_{t-1}\|_F^2] + \frac{2\ell_{\tilde{F}}^2}{1-\lambda}\frac{1}{K}\mathbb{E}[\|Y_{t} - Y_{t-1}\|_F^2]
			\\ &\quad +\frac{2\alpha_1^2\eta^4}{1-\lambda}\frac{1}{K}\mathbb{E}[\| U_{t-1}- \Delta_{t-1}^{\tilde{F}}\|_F^2]+\frac{\alpha_1^2\eta^4\sigma_{\tilde{F}}^2}{1-\lambda} \ . \\
		\end{aligned}
	\end{equation}
\end{lemma}

\begin{proof}
	\begin{equation}
		\begin{aligned}
			& \quad \frac{1}{K}\mathbb{E}[\|Z_{t}^{\tilde{F}} - \bar{Z}_{t}^{\tilde{F}}\|_F^2] \\
			& = \frac{1}{K}\mathbb{E}[\|Z_{t-1}^{\tilde{F}}W +U_{t} - U_{t-1} - \bar{Z}_{t-1}^{\tilde{F}}- \bar{U}_{t} +\bar{U}_{t-1}\|_F^2] \\
			&\overset{(s_1)} \leq \frac{1}{K}\lambda\mathbb{E}[\|Z_{t-1}^{\tilde{F}} - \bar{Z}_{t-1}^{\tilde{F}} \|_F^2]  + \frac{1}{1-\lambda}\frac{1}{K}\mathbb{E}[\|U_{t} - U_{t-1} - \bar{U}_{t} +\bar{U}_{t-1}\|_F^2]\\
			& \leq \lambda\frac{1}{K}\mathbb{E}[\|Z_{t-1}^{\tilde{F}} - \bar{Z}_{t-1}^{\tilde{F}} \|_F^2]  + \frac{1}{1-\lambda}\frac{1}{K}\mathbb{E}[\|U_{t} - U_{t-1} \|_F^2]\\
			& = \lambda\frac{1}{K}\mathbb{E}[\|Z_{t-1}^{\tilde{F}} - \bar{Z}_{t-1}^{\tilde{F}} \|_F^2]  + \frac{1}{1-\lambda}\frac{1}{K}\mathbb{E}[\|(1-\alpha_1\eta^2)(U_{t-1} - \Delta_{t-1}^{\tilde{F}_{\tilde{\xi}_t}})  + \Delta_{t}^{\tilde{F}_{\tilde{\xi}_t}}- U_{t-1} \|_F^2]\\
			& = \lambda\frac{1}{K}\mathbb{E}[\|Z_{t-1}^{\tilde{F}} - \bar{Z}_{t-1}^{\tilde{F}} \|_F^2]  + \frac{1}{1-\lambda}\frac{1}{K}\mathbb{E}[\|\Delta_{t}^{\tilde{F}_{\tilde{\xi}_t}} - \Delta_{t-1}^{\tilde{F}_{\tilde{\xi}_t}} -\alpha_1\eta^2 (U_{t-1}- \Delta_{t-1}^{\tilde{F}}) -\alpha_1\eta^2(\Delta_{t-1}^{\tilde{F}} - \Delta_{t-1}^{\tilde{F}_{\tilde{\xi}_t}})   \|_F^2]\\
			& 	\overset{(s_1)}	\leq \lambda\frac{1}{K}\mathbb{E}[\|Z_{t-1}^{\tilde{F}} - \bar{Z}_{t-1}^{\tilde{F}} \|_F^2]  + \frac{2}{1-\lambda}\frac{1}{K}\mathbb{E}[\|\Delta_{t}^{\tilde{F}_{\tilde{\xi}_t}} - \Delta_{t-1}^{\tilde{F}_{\tilde{\xi}_t}}\|_F^2]  +\frac{2\alpha_1^2\eta^4}{1-\lambda}\frac{1}{K}\mathbb{E}[\| U_{t-1}- \Delta_{t-1}^{\tilde{F}}\|_F^2]+\frac{\alpha_1^2\eta^4\sigma_{\tilde{F}}^2}{1-\lambda}\\
			& 	\overset{(s_2)}		\leq \lambda\frac{1}{K}\mathbb{E}[\|Z_{t-1}^{\tilde{F}} - \bar{Z}_{t-1}^{\tilde{F}} \|_F^2]  + \frac{2L_{\tilde{F}}^2}{1-\lambda}\frac{1}{K}\mathbb{E}[\|X_{t} - X_{t-1}\|_F^2] + \frac{2L_{\tilde{F}}^2}{1-\lambda}\frac{1}{K}\mathbb{E}[\|Y_{t} - Y_{t-1}\|_F^2]
			\\ &\quad +\frac{2\alpha_1^2\eta^4}{1-\lambda}\frac{1}{K}\mathbb{E}[\| U_{t-1}- \Delta_{t-1}^{\tilde{F}}\|_F^2]+\frac{\alpha_1^2\eta^4\sigma_{\tilde{F}}^2}{1-\lambda} \ , \\
		\end{aligned}
	\end{equation}
	where $(s_1)$ holds due to Lemma~\ref{lemma_ineqality} with $a=\frac{1-\lambda}{\lambda}$, $(s_2)$ holds due to $\mathbb{E}[\Delta_{t-1}^{\tilde{F}_{\tilde{\xi}_t}}]=\Delta_{t-1}^{\tilde{F}}$ and Lemma~\ref{lemma_hypergrad_var_var}, $(s_3)$ holds due to Lemma~\ref{lemma_hypergrad_smooth_var}. 
\end{proof}

\begin{lemma}
	Given Assumptions~\ref{assumption_graph},~\ref{assumption_bi_strong},~\ref{assumption_variance},~\ref{assumption_upper_smooth_vr},~\ref{assumption_lower_smooth_vr},  the following inequality holds. 
	\begin{equation}
		\begin{aligned}
			&   \frac{1}{K}\mathbb{E}[\|Z_{t}^{g} - \bar{Z}_{t}^{g}\|_F^2] 	\leq \lambda \frac{1}{K}\mathbb{E}[\|Z_{t-1}^{g} - \bar{Z}_{t-1}^{g} \|_F^2]  + \frac{2\ell_{g_y}^2}{1-\lambda} \frac{1}{K}\mathbb{E}[\|X_{t} - X_{t-1}\|_F^2] + \frac{2\ell_{g_y}^2}{1-\lambda} \frac{1}{K}\mathbb{E}[\|Y_{t} - Y_{t-1}\|_F^2]
			\\ &\quad +\frac{2\alpha_2^2\eta^4}{1-\lambda} \frac{1}{K}\mathbb{E}[\| V_{t-1}- \Delta_{t-1}^{g}\|_F^2]+\frac{\alpha_2^2\eta^4\sigma^2}{1-\lambda} \ . \\
		\end{aligned}
	\end{equation}
\end{lemma}

\begin{proof}
	\begin{equation}
		\begin{aligned}
			& \quad  \frac{1}{K}\mathbb{E}[\|Z_{t}^{g} - \bar{Z}_{t}^{g}\|_F^2] \\
			& =  \frac{1}{K}\mathbb{E}[\|Z_{t-1}^{g}W +V_{t} - V_{t-1} - \bar{Z}_{t-1}^{g}- \bar{V}_{t} +\bar{V}_{t-1}\|_F^2] \\
			& \overset{(s_1)}  \leq \lambda \frac{1}{K}\mathbb{E}[\|Z_{t-1}^{g} - \bar{Z}_{t-1}^{g} \|_F^2]  + \frac{1}{1-\lambda} \frac{1}{K}\mathbb{E}[\|V_{t} - V_{t-1} - \bar{V}_{t} +\bar{V}_{t-1}\|_F^2]\\
			& \leq \lambda \frac{1}{K}\mathbb{E}[\|Z_{t-1}^{g} - \bar{Z}_{t-1}^{g} \|_F^2]  + \frac{1}{1-\lambda} \frac{1}{K}\mathbb{E}[\|V_{t} - V_{t-1} \|_F^2]\\
			& = \lambda \frac{1}{K}\mathbb{E}[\|Z_{t-1}^{g} - \bar{Z}_{t-1}^{g} \|_F^2]  + \frac{1}{1-\lambda} \frac{1}{K}\mathbb{E}[\|(1-\alpha_2\eta^2)(V_{t-1} - \Delta_{t-1}^{g_{\zeta_t}})  + \Delta_{t}^{g_{\zeta_t}}- V_{t-1} \|_F^2]\\
			& = \lambda \frac{1}{K}\mathbb{E}[\|Z_{t-1}^{g} - \bar{Z}_{t-1}^{g} \|_F^2]  + \frac{1}{1-\lambda} \frac{1}{K}\mathbb{E}[\|\Delta_{t}^{g_{\zeta_t}} - \Delta_{t-1}^{g_{\zeta_t}} -\alpha_2\eta^2 (V_{t-1}- \Delta_{t-1}^{g}) -\alpha_2\eta^2(\Delta_{t-1}^{g} - \Delta_{t-1}^{g_{\zeta_t}})   \|_F^2]\\
			& 	\overset{(s_2)}  	\leq \lambda \frac{1}{K}\mathbb{E}[\|Z_{t-1}^{g} - \bar{Z}_{t-1}^{g} \|_F^2]  + \frac{2}{1-\lambda} \frac{1}{K}\mathbb{E}[\|\Delta_{t}^{g_{\zeta_t}} - \Delta_{t-1}^{g_{\zeta_t}}\|_F^2]  +\frac{2\alpha_2^2\eta^4}{1-\lambda} \frac{1}{K}\mathbb{E}[\| V_{t-1}- \Delta_{t-1}^{g}\|_F^2]+\frac{\alpha_2^2\eta^4\sigma^2}{1-\lambda}\\
			& 	\overset{(s_3)}  	\leq \lambda \frac{1}{K}\mathbb{E}[\|Z_{t-1}^{g} - \bar{Z}_{t-1}^{g} \|_F^2]  + \frac{2\ell_{g_y}^2}{1-\lambda} \frac{1}{K}\mathbb{E}[\|X_{t} - X_{t-1}\|_F^2] + \frac{2\ell_{g_y}^2}{1-\lambda} \frac{1}{K}\mathbb{E}[\|Y_{t} - Y_{t-1}\|_F^2]
			\\ &\quad +\frac{2\alpha_2^2\eta^4}{1-\lambda} \frac{1}{K}\mathbb{E}[\| V_{t-1}- \Delta_{t-1}^{g}\|_F^2]+\frac{\alpha_2^2\eta^4\sigma^2}{1-\lambda} \ . \\
		\end{aligned}
	\end{equation}
	where $(s_1)$ holds due to Lemma~\ref{lemma_ineqality} with $a=\frac{1-\lambda}{\lambda}$, $(s_2)$ holds due to $\mathbb{E}[\Delta_{t-1}^{g_{\zeta_t}}]=\Delta_{t-1}^{g}$ and Assumption~\ref{assumption_variance}, $(s_3)$ holds due to Assumption~\ref{assumption_lower_smooth_vr}. 
	
\end{proof}

Note that Lemmas~\ref{lemma_consensus_x},~\ref{lemma_consensus_y},~\ref{lemma_incremental_x},~\ref{lemma_incremental_y} still hold. 
Then, based on these lemmas, we begin to prove Theorem~\ref{theorem_vrdbo}. 

\subsubsection{Proof of Theorem~\ref{theorem_vrdbo}}
\begin{proof}
	
	At first, we introduce the following potential function for investigating the convergence rate of Algorithm~\ref{alg_VRDBO}. 
	\begin{equation}
		\begin{aligned}
			& \mathcal{L}_{t+1} = {\mathbb{E}}[F(x_{t+1})] +  \frac{6\beta_1{L}_{F}^2}{\beta_2\mu}\mathbb{E}[\|\bar{   {y}}_{t+1} -    {y}^{*}(\bar{   {x}}_{t+1})\| ^2 ] \\
			& \quad   + \frac{2\beta_1\ell_{g_y}^2((51+48/\alpha_1K){L}_{\tilde{F}}^2+(98+800/\alpha_2K){L}_{F}^2)}{\mu^2(1-\lambda^2)} \frac{1}{K} \mathbb{E}[\|X_{t+1} - \bar{X}_{t+1}\|_F^2 ] \\
			& \quad + \frac{2\beta_1\ell_{g_y}^2((51+48/\alpha_1K){L}_{\tilde{F}}^2+(98+800/\alpha_2K){L}_{F}^2)}{\mu^2(1-\lambda^2)} \frac{1}{K}\mathbb{E}[\|Y_{t+1} - \bar{Y}_{t+1}\|_F^2 ] \\
			& \quad +\beta_1(1-\lambda) \frac{1}{K}\mathbb{E}[\| Z^{\tilde{F}}_{t+1}-\bar{Z}^{\tilde{F}}_{t+1} \|_F^2]  +\frac{\beta_1(1-\lambda) L_F^2}{\mu^2}\frac{1}{K}\mathbb{E}[\| Z^{g}_{t+1}-\bar{Z}^{g}_{t+1} \|_F^2]  \\
			& \quad + 2\beta_1 \frac{1}{K}\mathbb{E}[\|\Delta_{t+1}^{\tilde{F}}-U_{t+1}\|_F^2] + \frac{2\beta_1L_F^2}{\mu^2}\frac{1}{K}\mathbb{E}[\|\Delta_{t+1}^{g}-V_{t+1}\|_F^2]  \\
			& \quad + \frac{3\beta_1}{\alpha_1\eta} \mathbb{E}[\|(\Delta_{t+1}^{\tilde{F}}-U_{t+1})\frac{1}{K}\mathbf{1}\|^2] + \frac{50\beta_1{L}_{F}^2}{\alpha_2\eta\mu^2}\mathbb{E}[\|(\Delta_{t+1}^{g}-V_{t+1})\frac{1}{K}\mathbf{1}\|^2] \ . \\
		\end{aligned}
	\end{equation}


	According to the aforementioned lemmas, Eq.~(\ref{lemma_F_incremental}), and Eq.~(\ref{eq_y}) where  $L_{g_y}$ is replaced with $\ell_{g_{y}}$,  we can get
	\begin{equation}
		\begin{aligned}
			& \quad \mathcal{L}_{t+1} - \mathcal{L}_{t} \\
			& \leq - \frac{\eta\beta_1}{2} \mathbb{E}[\|\nabla F(\bar{  {x}}_{t})\|^2] - \frac{\eta\beta_1L_F^2}{2} \mathbb{E}[\|\bar{y}_{t} - y^*(\bar{x}_{t})\|^2] +  \frac{3\eta\beta_1C_{g_{xy}}^2C_{f_y}^2}{\mu^2}(1-\frac{\mu}{L_{g_{y}}})^{2J}\\
			& \quad +( \frac{6\beta_1{L}_{F}^2}{\beta_2\mu}\frac{25\eta\beta_1^2L_{y}^2 }{6\beta_2\mu}- \frac{\eta\beta_1}{4}+4\eta^2\beta_1^2\tilde{w})   \mathbb{E}[\|\bar{  {u}}_t\|^2]  + ( 4\eta^2\beta_2^2\tilde{w}- \frac{6\beta_1{L}_{F}^2}{\beta_2\mu}\frac{3\eta\beta_2^2}{4}) \mathbb{E}[\|\bar{v}_{t}  \|^2]  \\
			& \quad    + (\frac{2\eta \beta_1^2}{1-\lambda^2} \frac{2\beta_1\ell_{g_y}^2((51+48/\alpha_1K){L}_{\tilde{F}}^2+(98+800/\alpha_2K){L}_{F}^2)}{\mu^2(1-\lambda^2)} - \beta_1(1-\lambda)^2  +  4\eta^2\beta_1^2\tilde{w}) \frac{1}{K} \mathbb{E}[\|Z_t^{\tilde{F}} - \bar{ Z}_t^{\tilde{F}}\|_F^2] \\ 
			& \quad + (\frac{2\eta \beta_2^2}{1-\lambda^2}\frac{2\beta_1\ell_{g_y}^2((51+48/\alpha_1K){L}_{\tilde{F}}^2+(98+800/\alpha_2K){L}_{F}^2)}{\mu^2(1-\lambda^2)} -     \frac{\beta_1(1-\lambda)^2L_F^2}{\mu^2}    + 4\eta^2\beta_2^2\tilde{w})\frac{1}{K}\mathbb{E}[\|Z_t^{g} - \bar{Z}_t^{g}\|_F^2] \\
			& \quad + \beta_1\alpha_1^2\eta^4\sigma_{\tilde{F}}^2+   \frac{\beta_1\alpha_2^2\eta^4\sigma^2 L_F^2}{\mu^2}  +4\beta_1\alpha_1^2 \eta^3 \sigma_{\tilde{F}}^2  +  \frac{4\beta_1\alpha_2^2\eta^3 \sigma^2L_F^2}{\mu^2}  + \frac{6\beta_1\alpha_1 \eta^3 \sigma_{\tilde{F}}^2 }{K} + \frac{100\beta_1\alpha_2\eta^3 \sigma^2{L}_{F}^2}{\mu^2K} \ , \\
		\end{aligned}
	\end{equation}
where $\tilde{w} =2 \beta_1L_{\tilde{F}}^2+  \frac{2\beta_1 L_F^2\ell_{g_y}^2}{\mu^2}    + \frac{4\beta_1 L_{\tilde{F}}^2}{\eta}+\frac{4\beta_1L_F^2{\ell}_{g_{y}}^2  }{\eta\mu^2} + \frac{\beta_1(6/\alpha_1K)L_{\tilde{F}}^2}{\eta}+\frac{(100/\alpha_2K)\beta_1{L}_{F}^2{\ell}_{g_{y}}^2}{\eta\mu^2}$. 

By setting the coefficient of $\mathbb{E}[\|Z_t^{\tilde{F}} - \bar{ Z}_t^{\tilde{F}}\|_F^2] $ to be non-positive, we can get
\begin{equation}
	\begin{aligned}
		& \frac{2 \beta_1^2}{1-\lambda^2} \frac{2\ell_{g_y}^2((51+48/\alpha_1K){L}_{\tilde{F}}^2+(98+800/\alpha_2K){L}_{F}^2)}{\mu^2(1-\lambda^2)} -(1-\lambda)^2 \\
		& \quad +  4\eta\beta_1^2\Big(\frac{2 L_{\tilde{F}}^2\ell_{g_y}^2}{\mu^2}+  \frac{2 L_F^2\ell_{g_y}^2}{\mu^2}    + \frac{4 L_{\tilde{F}}^2\ell_{g_y}^2}{\eta\mu^2}+\frac{4L_F^2{\ell}_{g_{y}}^2  }{\eta\mu^2} + \frac{(6/\alpha_1K)L_{\tilde{F}}^2\ell_{g_y}^2}{\eta\mu^2}+\frac{(100/\alpha_2K){L}_{F}^2{\ell}_{g_{y}}^2}{\eta\mu^2} \Big)  \leq  0  \ . \\
	\end{aligned}
\end{equation}
Due to $\eta<1$ and $\lambda<1$, we can get
\begin{equation}
	\begin{aligned}
		& \frac{2 \beta_1^2}{1-\lambda^2} \frac{2\ell_{g_y}^2((51+48/\alpha_1K){L}_{\tilde{F}}^2+(98+800/\alpha_2K){L}_{F}^2)}{\mu^2(1-\lambda^2)} -(1-\lambda)^2 \\
		& \quad +  4\eta\beta_1^2\Big(\frac{2 L_{\tilde{F}}^2\ell_{g_y}^2}{\mu^2}+  \frac{2 L_F^2\ell_{g_y}^2}{\mu^2}    + \frac{4 L_{\tilde{F}}^2\ell_{g_y}^2}{\eta\mu^2}+\frac{4L_F^2{\ell}_{g_{y}}^2  }{\eta\mu^2} + \frac{(6/\alpha_1K)L_{\tilde{F}}^2\ell_{g_y}^2}{\eta\mu^2}+\frac{(100/\alpha_2K){L}_{F}^2{\ell}_{g_{y}}^2}{\eta\mu^2} \Big)   \\
		& \leq \frac{2\beta_1^2}{1-\lambda^2} \frac{2\ell_{g_y}^2((51+48/\alpha_1K){L}_{\tilde{F}}^2+(98+800/\alpha_2K){L}_{F}^2)}{\mu^2(1-\lambda^2)} -(1-\lambda)^2 \\
		& \quad +  4\beta_1^2\Big(\frac{2 L_{\tilde{F}}^2\ell_{g_y}^2}{\mu^2}+  \frac{2 L_F^2\ell_{g_y}^2}{\mu^2}    + \frac{4 L_{\tilde{F}}^2\ell_{g_y}^2}{\mu^2}+\frac{4L_F^2{\ell}_{g_{y}}^2  }{\mu^2} + \frac{(6/\alpha_1K)L_{\tilde{F}}^2\ell_{g_y}^2}{\mu^2}+\frac{(100/\alpha_2K){L}_{F}^2{\ell}_{g_{y}}^2}{\mu^2} \Big)  \\
		& \leq  \frac{4\beta_1^2\ell_{g_y}^2((51+48/\alpha_1K){L}_{\tilde{F}}^2+(98+800/\alpha_2K){L}_{F}^2)}{\mu^2(1-\lambda^2)^2} -(1-\lambda)^2  +  4\beta_1^2{\ell}_{g_{y}}^2 \frac{(6+6/\alpha_1K)L_{\tilde{F}}^2+(6+100/\alpha_2K){L}_{F}^2}{\mu^2}  \\
		& \leq   4\beta_1^2{\ell}_{g_{y}}^2 \frac{(57+54/\alpha_1K)L_{\tilde{F}}^2+(104+900/\alpha_2K){L}_{F}^2}{\mu^2(1-\lambda^2)^2}  - (1-\lambda)^2  \ . 
	\end{aligned}
\end{equation}
By letting this upper bound non-positive, we can get
\begin{equation}
	\begin{aligned}
		&     4\beta_1^2{\ell}_{g_{y}}^2 \frac{(57+54/\alpha_1K)L_{\tilde{F}}^2+(104+900/\alpha_2K){L}_{F}^2}{\mu^2(1-\lambda^2)^2} \leq  (1-\lambda)^2  \\
		& \beta_1 \leq \frac{\mu(1-\lambda)^2  } {2{\ell}_{g_{y}}\sqrt{(57+54/\alpha_1K)L_{\tilde{F}}^2+(104+900/\alpha_2K){L}_{F}^2} } \ . \\
	\end{aligned}
\end{equation}

By setting the coefficient of $\mathbb{E}[\|Z_t^{g} - \bar{Z}_t^{g}\|_F^2]  $ to be non-positive, we can get
\begin{equation}
	\begin{aligned}
		& \frac{2 \beta_2^2}{1-\lambda^2}\frac{2\ell_{g_y}^2((51+48/\alpha_1K){L}_{\tilde{F}}^2+(98+800/\alpha_2K){L}_{F}^2)}{\mu^2(1-\lambda^2)} -  \frac{(1-\lambda)^2 L_F^2}{\mu^2}   \\
		& \quad  + 4\eta\beta_2^2\Big(\frac{2 L_{\tilde{F}}^2\ell_{g_y}^2}{\mu^2}+  \frac{2 L_F^2\ell_{g_y}^2}{\mu^2}    + \frac{4 L_{\tilde{F}}^2\ell_{g_y}^2}{\eta\mu^2}+\frac{4L_F^2{\ell}_{g_{y}}^2  }{\eta\mu^2} + \frac{(6/\alpha_1K)L_{\tilde{F}}^2\ell_{g_y}^2}{\eta\mu^2}+\frac{(100/\alpha_2K){L}_{F}^2{\ell}_{g_{y}}^2}{\eta\mu^2} \Big) \leq 0   \ . \\
	\end{aligned}
\end{equation}
Due to $\eta<1$, $\lambda<1$, and $\ell_{g_y}/\mu>1$, we can get
\begin{equation}
	\begin{aligned}
		& \frac{2 \beta_2^2}{1-\lambda^2}\frac{2\ell_{g_y}^2((51+48/\alpha_1K){L}_{\tilde{F}}^2+(98+800/\alpha_2K){L}_{F}^2)}{\mu^2(1-\lambda^2)} -  \frac{(1-\lambda)^2 L_F^2}{\mu^2}   \\
		& \quad  + 4\eta\beta_2^2\Big(\frac{2 L_{\tilde{F}}^2\ell_{g_y}^2}{\mu^2}+  \frac{2 L_F^2\ell_{g_y}^2}{\mu^2}    + \frac{4 L_{\tilde{F}}^2\ell_{g_y}^2}{\eta\mu^2}+\frac{4L_F^2{\ell}_{g_{y}}^2  }{\eta\mu^2} + \frac{(6/\alpha_1K)L_{\tilde{F}}^2\ell_{g_y}^2}{\eta\mu^2}+\frac{(100/\alpha_2K){L}_{F}^2{\ell}_{g_{y}}^2}{\eta\mu^2} \Big)  \\
		& \leq \frac{4 \beta_2^2\ell_{g_y}^2((51+48/\alpha_1K){L}_{\tilde{F}}^2+(98+800/\alpha_2K){L}_{F}^2)}{\mu^2(1-\lambda^2)^2} -    \frac{(1-\lambda)^2 L_F^2}{\mu^2}   \\
		& \quad  + 4\beta_2^2{\ell}_{g_{y}}^2 \frac{(6+6/\alpha_1K)L_{\tilde{F}}^2+(6+100/\alpha_2K){L}_{F}^2}{\mu^2}\\
		&  \leq    4\beta_2^2{\ell}_{g_{y}}^2 \frac{(57+54/\alpha_1K)L_{\tilde{F}}^2+(104+900/\alpha_2K){L}_{F}^2}{\mu^2(1-\lambda^2)^2} -   \frac{(1-\lambda)^2 L_F^2}{\mu^2} \ .  \\
	\end{aligned}
\end{equation}
By letting this upper bound non-positive, we can get
\begin{equation}
		   \beta_2 \leq   \frac{(1-\lambda)^2 L_F}{2{\ell}_{g_{y}}\sqrt{(57+54/\alpha_1K)L_{\tilde{F}}^2+(104+900/\alpha_2K){L}_{F}^2}}  \ . \\
\end{equation}

By setting the coefficient of $\mathbb{E}[\|\bar{  {u}}_t\|^2]$ to be non-positive, we can get
\begin{equation}
	\begin{aligned}
		&  \frac{6{L}_{F}^2}{\beta_2\mu} \frac{25\beta_1^2L_{y}^2 }{6\beta_2\mu}- \frac{1}{4}+4\eta\beta_1^2\Big(\frac{2 L_{\tilde{F}}^2\ell_{g_y}^2}{\mu^2}+  \frac{2 L_F^2\ell_{g_y}^2}{\mu^2}    + \frac{4 L_{\tilde{F}}^2\ell_{g_y}^2}{\eta\mu^2}+\frac{4L_F^2{\ell}_{g_{y}}^2  }{\eta\mu^2} + \frac{(6/\alpha_1K)L_{\tilde{F}}^2\ell_{g_y}^2}{\eta\mu^2}+\frac{(100/\alpha_2K){L}_{F}^2{\ell}_{g_{y}}^2}{\eta\mu^2} \Big)  \leq  0 \ .  \\
	\end{aligned}
\end{equation}
Due to $\eta_1<1$ and $\ell_{g_y}/\mu>1$, we can get
\begin{equation}
	\begin{aligned}
		&  \frac{6{L}_{F}^2}{\beta_2\mu} \frac{25\beta_1^2L_{y}^2 }{6\beta_2\mu}- \frac{1}{4}+4\eta\beta_1^2\Big(\frac{2 L_{\tilde{F}}^2\ell_{g_y}^2}{\mu^2}+  \frac{2 L_F^2\ell_{g_y}^2}{\mu^2}    + \frac{4 L_{\tilde{F}}^2\ell_{g_y}^2}{\eta\mu^2}+\frac{4L_F^2{\ell}_{g_{y}}^2  }{\eta\mu^2} + \frac{(6/\alpha_1K)L_{\tilde{F}}^2\ell_{g_y}^2}{\eta\mu^2}+\frac{(100/\alpha_2K){L}_{F}^2{\ell}_{g_{y}}^2}{\eta\mu^2} \Big)  \\
		&  \leq \frac{6{L}_{F}^2}{\beta_2\mu} \frac{25\beta_1^2L_{y}^2 }{6\beta_2\mu}- \frac{1}{4}+4\beta_1^2{\ell}_{g_{y}}^2 \frac{(6+6/\alpha_1K)L_{\tilde{F}}^2+(6+100/\alpha_2K){L}_{F}^2}{\mu^2}  \ .\\
	\end{aligned}
\end{equation}
By setting $\beta_1 \leq  \frac{\beta_2\mu}{15{L}_{F}L_{y}} $, we can get  $ \frac{6{L}_{F}^2}{\beta_2\mu} \frac{25\beta_1^2L_{y}^2 }{6\beta_2\mu}- \frac{1}{4} \leq - \frac{1}{8}$. Then, we have
\begin{equation}
	\begin{aligned}
		&   \frac{6{L}_{F}^2}{\beta_2\mu} \frac{25\beta_1^2L_{y}^2 }{6\beta_2\mu}- \frac{1}{4}+4\beta_1^2{\ell}_{g_{y}}^2 \frac{(6+6/\alpha_1K)L_{\tilde{F}}^2+(6+100/\alpha_2K){L}_{F}^2}{\mu^2}  \\
		& \leq 4\beta_1^2{\ell}_{g_{y}}^2 \frac{(6+6/\alpha_1K)L_{\tilde{F}}^2+(6+100/\alpha_2K){L}_{F}^2}{\mu^2}  - \frac{1}{8}  \ . \\
	\end{aligned}
\end{equation}
By letting this upper bound non-positive, we can get
\begin{equation}
	\begin{aligned}
		& \beta_1 \leq  \frac{\mu}{8{\ell}_{g_{y}}\sqrt{(3+3/\alpha_1K)L_{\tilde{F}}^2+(3+50/\alpha_2K){L}_{F}^2}} \ . \\
	\end{aligned}
\end{equation}

By setting the coefficient of $\mathbb{E}[\|\bar{  {v}}_t\|^2]$ to be non-positive, we can get
\begin{equation}
	\begin{aligned}
		& 4\eta\beta_2^2 \beta_1\Big(\frac{2 L_{\tilde{F}}^2\ell_{g_y}^2}{\mu^2}+  \frac{2 L_F^2\ell_{g_y}^2}{\mu^2}    + \frac{4 L_{\tilde{F}}^2\ell_{g_y}^2}{\eta\mu^2}+\frac{4L_F^2{\ell}_{g_{y}}^2  }{\eta\mu^2} + \frac{(6/\alpha_1K)L_{\tilde{F}}^2\ell_{g_y}^2}{\eta\mu^2}+\frac{(100/\alpha_2K){L}_{F}^2{\ell}_{g_{y}}^2}{\eta\mu^2} \Big) - \frac{6\beta_1{L}_{F}^2}{\beta_2\mu}\frac{3\beta_2^2}{4} \leq  0 \ .  \\
	\end{aligned}
\end{equation}
Due to $\eta<1$ and $\ell_{g_y}/\mu>1$, we can get
\begin{equation}
	\begin{aligned}
		& 4\eta\beta_2^2 \beta_1\Big(\frac{2 L_{\tilde{F}}^2\ell_{g_y}^2}{\mu^2}+  \frac{2 L_F^2\ell_{g_y}^2}{\mu^2}    + \frac{4 L_{\tilde{F}}^2\ell_{g_y}^2}{\eta\mu^2}+\frac{4L_F^2{\ell}_{g_{y}}^2  }{\eta\mu^2} + \frac{(6/\alpha_1K)L_{\tilde{F}}^2\ell_{g_y}^2}{\eta\mu^2}+\frac{(100/\alpha_2K){L}_{F}^2{\ell}_{g_{y}}^2}{\eta\mu^2} \Big) - \frac{6\beta_1{L}_{F}^2}{\beta_2\mu}\frac{3\beta_2^2}{4}  \\
		& \leq 4\beta_2^2 \beta_1{\ell}_{g_{y}}^2 \frac{(6+6/\alpha_1K)L_{\tilde{F}}^2+(6+100/\alpha_2K){L}_{F}^2}{\mu^2}- \frac{6\beta_1{L}_{F}^2}{\beta_2\mu}\frac{3\beta_2^2}{4}  \ .  \\
	\end{aligned}
\end{equation}
By letting this upper bound non-positive, we can get
\begin{equation}
	\begin{aligned}
		& \beta_2 \leq \frac{9\mu{L}_{F}^2}{8 {\ell}_{g_{y}}^2((6+6/\alpha_1K)L_{\tilde{F}}^2+(6+100/\alpha_2K){L}_{F}^2)} \ . 
	\end{aligned}
\end{equation}

	Therefore, by combining all these conditions, we can get 
	\begin{equation}
		\begin{aligned}
			& \beta_1\leq \min\Big\{\frac{\beta_2\mu}{15{L}_{F}L_{y}}, \frac{\mu}{8{\ell}_{g_{y}}\sqrt{(3+3/\alpha_1K)L_{\tilde{F}}^2+(3+50/\alpha_2K){L}_{F}^2}} ,    \frac{\mu(1-\lambda)^2  } {2{\ell}_{g_{y}}\sqrt{(57+54/\alpha_1K)L_{\tilde{F}}^2+(104+900/\alpha_2K){L}_{F}^2} } \Big\} \  , \\
			& \beta_2 \leq \min\Big\{\frac{(1-\lambda)^2 L_F}{2{\ell}_{g_{y}}\sqrt{(57+54/\alpha_1K)L_{\tilde{F}}^2+(104+900/\alpha_2K){L}_{F}^2}} , \frac{9\mu{L}_{F}^2}{8 {\ell}_{g_{y}}^2((6+6/\alpha_1K)L_{\tilde{F}}^2+(6+100/\alpha_2K){L}_{F}^2)} \Big\} \ . 
		\end{aligned}
	\end{equation}
	As a result, we can get
\begin{equation}
	\begin{aligned}
		& \quad \mathcal{L}_{t+1} - \mathcal{L}_{t} \\
		& \leq - \frac{\eta\beta_1}{2} \mathbb{E}[\|\nabla F(\bar{  {x}}_{t})\|^2] - \frac{\eta\beta_1L_F^2}{2} \mathbb{E}[\|\bar{y}_{t} - y^*(\bar{x}_{t})\|^2] +  \frac{3\eta\beta_1C_{g_{xy}}^2C_{f_y}^2}{\mu^2}(1-\frac{\mu}{L_{g_{y}}})^{2J}\\
		& \quad + \beta_1\alpha_1^2\eta^4\sigma_{\tilde{F}}^2+   \frac{\beta_1\alpha_2^2\eta^4\sigma^2 L_F^2}{\mu^2}  +4\beta_1\alpha_1^2 \eta^3 \sigma_{\tilde{F}}^2  +  \frac{4\beta_1\alpha_2^2\eta^3 \sigma^2L_F^2}{\mu^2}  + \frac{6\beta_1\alpha_1 \eta^3 \sigma_{\tilde{F}}^2 }{K} + \frac{100\beta_1\alpha_2\eta^3 \sigma^2{L}_{F}^2}{\mu^2K} \ . \\
	\end{aligned}
\end{equation}

	By summing over $t$ from $0$ to $T-1$, we can get 
	\begin{equation}
		\begin{aligned}
			&  \quad \frac{1}{T}\sum_{t=0}^{T-1}\mathbb{E}[\|\nabla F(\bar{  {x}}_{t})\|^2  + L_F^2\|\bar{y}_{t} - y^*(\bar{x}_{t})\|^2] \\
			& \leq \frac{2(\mathcal{L}_{0} - \mathcal{L}_{T} )}{\eta\beta_1 T}+  \frac{6C_{g_{xy}}^2C_{f_y}^2}{\mu^2}(1-\frac{\mu}{\ell_{g_{y}}})^{2J} \\
			& \quad + 2\alpha_1^2\eta^3\sigma_{\tilde{F}}^2+   \frac{2\alpha_2^2\eta^3\sigma^2 L_F^2}{\mu^2}  +8\alpha_1^2 \eta^2 \sigma_{\tilde{F}}^2  +  \frac{8\alpha_2^2\eta^2 \sigma^2L_F^2}{\mu^2}  + \frac{12\alpha_1 \eta^2 \sigma_{\tilde{F}}^2 }{K} + \frac{200\alpha_2\eta^2 \sigma^2{L}_{F}^2}{\mu^2K} \ . \\
		\end{aligned}
	\end{equation}
	In terms of the initialisation,  we can get
	\begin{equation}
	\begin{aligned}
		& \mathcal{L}_{0} = {\mathbb{E}}[F(x_{0})] +  \frac{6\beta_1{L}_{F}^2}{\beta_2\mu}\mathbb{E}[\|\bar{   {y}}_{0} -    {y}^{*}(\bar{   {x}}_{0})\| ^2 ] \\
		& \quad +\beta_1(1-\lambda) \frac{1}{K}\mathbb{E}[\| Z^{\tilde{F}}_{0}-\bar{Z}^{\tilde{F}}_{0} \|_F^2]  +\frac{\beta_1(1-\lambda) L_F^2}{\mu^2}\frac{1}{K}\mathbb{E}[\| Z^{g}_{0}-\bar{Z}^{g}_{0} \|_F^2]  \\
		& \quad + 2\beta_1 \frac{1}{K}\mathbb{E}[\|\Delta_{0}^{\tilde{F}}-U_{0}\|_F^2] + \frac{2\beta_1L_F^2}{\mu^2}\frac{1}{K}\mathbb{E}[\|\Delta_{0}^{g}-V_{0}\|_F^2]  \\
		& \quad + \frac{3\beta_1}{\alpha_1\eta} \mathbb{E}[\|(\Delta_{0}^{\tilde{F}}-U_{0})\frac{1}{K}\mathbf{1}\|^2] + \frac{50\beta_1{L}_{F}^2}{\alpha_2\eta\mu^2}\mathbb{E}[\|(\Delta_{0}^{g}-V_{0})\frac{1}{K}\mathbf{1}\|^2] \ . \\
	\end{aligned}
\end{equation}

As for $\frac{1}{K}\mathbb{E}[\| Z^{\tilde{F}}_{0}-\bar{Z}^{\tilde{F}}_{0} \|_F^2]$, we have
\begin{equation}
	\begin{aligned}
		& \quad \frac{1}{K}\mathbb{E}[\| Z^{\tilde{F}}_{0}-\bar{Z}^{\tilde{F}}_{0} \|_F^2]  \\
		& = \frac{1}{K}\mathbb{E}[\| \Delta_{0}^{\tilde{F}_{\tilde{\xi}_0}}-\bar{\Delta}_{0}^{\tilde{F}_{\tilde{\xi}_0}} \|_F^2]  \\
		& = \frac{1}{K}\sum_{k=1}^{K}  \mathbb{E}[\|\nabla\tilde{F}^{(k)}(x_0, y_0; \tilde{\xi}_0^{(k)}) - \frac{1}{K} \sum_{k'=1}^{K}\nabla\tilde{F}^{(k')}(x_0, y_0; \tilde{\xi}_0^{(k')})  \|_F^2]  \\
		& = \frac{1}{K}\sum_{k=1}^{K}  \mathbb{E}[\|\nabla\tilde{F}^{(k)}(x_0, y_0; \tilde{\xi}_0^{(k)}) - \nabla\tilde{F}^{(k)}(x_0, y_0) +  \nabla\tilde{F}^{(k)}(x_0, y_0) \\
		& \quad  -  \frac{1}{K} \sum_{k'=1}^{K}\nabla\tilde{F}^{(k')}(x_0, y_0)  + \frac{1}{K} \sum_{k'=1}^{K}\nabla\tilde{F}^{(k')}(x_0, y_0)  - \frac{1}{K} \sum_{k'=1}^{K}\nabla\tilde{F}^{(k')}(x_0, y_0; \tilde{\xi}_0^{(k')})  \|_F^2]  \\
		& = \frac{1}{K}\sum_{k=1}^{K}  \mathbb{E}[\|\nabla\tilde{F}^{(k)}(x_0, y_0; \tilde{\xi}_0^{(k)}) - \nabla\tilde{F}^{(k)}(x_0, y_0)  + \frac{1}{K} \sum_{k'=1}^{K}\nabla\tilde{F}^{(k')}(x_0, y_0)  - \frac{1}{K} \sum_{k'=1}^{K}\nabla\tilde{F}^{(k')}(x_0, y_0; \tilde{\xi}_0^{(k')})  \|_F^2]  \\
		& \leq  2\sigma_{\tilde{F}}^2 \   .
	\end{aligned}
\end{equation}
Similarly, we can get $\frac{1}{K}\mathbb{E}[\| Z^{g}_{0}-\bar{Z}^{g}_{0} \|_F^2]  \leq  2\sigma^2$, $\frac{1}{K}\mathbb{E}[\|\Delta_{0}^{\tilde{F}}-U_{0}\|_F^2]  = \frac{1}{K}\mathbb{E}[\|\Delta_{0}^{\tilde{F}}-\Delta_{0}^{\tilde{F}_{\tilde{\xi}_0}}\|_F^2] \leq \sigma_{\tilde{F}}^2$, $ \frac{1}{K}\mathbb{E}[\|\Delta_{0}^{g}-V_{0}\|_F^2]  \leq \sigma^2 $, $\mathbb{E}[\|(\Delta_{0}^{\tilde{F}}-U_{0})\frac{1}{K}\mathbf{1}\|^2]  = \mathbb{E}[\|(\Delta_{0}^{\tilde{F}}-\Delta_{0}^{\tilde{F}_{\tilde{\xi}_0}})\frac{1}{K}\mathbf{1}\|^2] \leq \frac{\sigma_{\tilde{F}}^2}{K}$, $\mathbb{E}[\|(\Delta_{0}^{g}-V_{0})\frac{1}{K}\mathbf{1}\|^2]  \leq \frac{\sigma^2 }{K}$.  When the mini-batch size is set to $B$, we can get
\begin{equation}
	\begin{aligned}
		& \mathcal{L}_{0} \leq  {\mathbb{E}}[F(x_{0})] +  \frac{6\beta_1{L}_{F}^2}{\beta_2\mu}\mathbb{E}[\|\bar{   {y}}_{0} -    {y}^{*}(\bar{   {x}}_{0})\| ^2 ] \\
		& \quad +\frac{2\beta_1 \sigma_{\tilde{F}}^2}{B} +\frac{2\beta_1L_F^2\sigma^2}{B\mu^2}   + \frac{2\beta_1 \sigma_{\tilde{F}}^2}{B} + \frac{2\beta_1L_F^2\sigma^2}{B\mu^2} + \frac{3\beta_1}{\alpha_1\eta} \frac{\sigma_{\tilde{F}}^2}{BK} + \frac{50\beta_1{L}_{F}^2}{\alpha_2\eta\mu^2} \frac{\sigma^2 }{BK} \\
		& =  {\mathbb{E}}[F(x_{0})] +  \frac{6\beta_1{L}_{F}^2}{\beta_2\mu}\mathbb{E}[\|\bar{   {y}}_{0} -    {y}^{*}(\bar{   {x}}_{0})\| ^2 ] +\frac{4\beta_1 \sigma_{\tilde{F}}^2}{B} +\frac{4\beta_1L_F^2\sigma^2}{B\mu^2}  + \frac{3\beta_1}{\alpha_1\eta} \frac{\sigma_{\tilde{F}}^2}{BK} + \frac{50\beta_1{L}_{F}^2}{\alpha_2\eta\mu^2} \frac{\sigma^2 }{BK}  \ . \\
	\end{aligned}
\end{equation}
Then, we can get
	\begin{equation}
	\begin{aligned}
		&  \quad \frac{1}{T}\sum_{t=0}^{T-1}(\mathbb{E}[\|\nabla F(\bar{  {x}}_{t})\|^2  + L_F^2\|\bar{y}_{t} - y^*(\bar{x}_{t})\|^2] )\\
		& \leq \frac{2(F(x_0) - F(x_*) )}{\eta\beta_1 T} +  \frac{12{L}_{F}^2}{\eta\beta_2 T\mu}\mathbb{E}[\|\bar{   {y}}_{0} -    {y}^{*}(\bar{   {x}}_{0})\| ^2 ] +  \frac{6C_{g_{xy}}^2C_{f_y}^2}{\mu^2}(1-\frac{\mu}{\ell_{g_{y}}})^{2J} \\
		& \quad +\frac{8 \sigma_{\tilde{F}}^2 }{\eta TB} + \frac{8L_F^2\sigma^2}{\eta TB\mu^2}  + \frac{6\sigma_{\tilde{F}}^2}{\alpha_1\eta^2 TBK}  + \frac{100{L}_{F}^2\sigma^2}{\alpha_2\eta^2 TBK\mu^2}  \\
		& \quad + 2\alpha_1^2\eta^3\sigma_{\tilde{F}}^2+   \frac{2\alpha_2^2\eta^3\sigma^2 L_F^2}{\mu^2}  +8\alpha_1^2 \eta^2 \sigma_{\tilde{F}}^2  +  \frac{8\alpha_2^2\eta^2 \sigma^2L_F^2}{\mu^2}  + \frac{12\alpha_1 \eta^2 \sigma_{\tilde{F}}^2 }{K} + \frac{200\alpha_2\eta^2 \sigma^2{L}_{F}^2}{\mu^2K} \ ,  \\
	\end{aligned}
\end{equation}
	which completes the proof.

\end{proof}

\vspace{-20pt}
\subsubsection{Proof of Corollary \ref{corollary_vrdbo}}
Corollary \ref{corollary_vrdbo} can be proved by following the proof of Corollary~\ref{corollary_mdbo}.

\subsection{Additional Lemmas}
\begin{lemma} \label{lemma_ineqality}
	For any matrices $X\in \mathbb{R}^{m\times n}$, $Y\in \mathbb{R}^{m\times n}$, the following inequality hold
	\begin{equation}
		\begin{aligned}
			&  \|X+Y\|_F^2 \leq (1+a)\|X\|_F^2 + (1+\frac{1}{a}) \|Y\|_F^2   \ , \\
		\end{aligned}
	\end{equation}
	for any $a>0$. 
\end{lemma}


%% file: sample_paper.bbl
\begin{thebibliography}{43}
\providecommand{\natexlab}[1]{#1}
\providecommand{\url}[1]{\texttt{#1}}
\expandafter\ifx\csname urlstyle\endcsname\relax
  \providecommand{\doi}[1]{doi: #1}\else
  \providecommand{\doi}{doi: \begingroup \urlstyle{rm}\Url}\fi

\bibitem[Chen et~al.(2021)Chen, Sun, and Yin]{chen2021tighter}
T.~Chen, Y.~Sun, and W.~Yin.
\newblock Tighter analysis of alternating stochastic gradient method for
  stochastic nested problems.
\newblock \emph{arXiv preprint arXiv:2106.13781}, 2021.

\bibitem[Chen et~al.(2022)Chen, Huang, and Ma]{chen2022decentralized}
X.~Chen, M.~Huang, and S.~Ma.
\newblock Decentralized bilevel optimization.
\newblock \emph{arXiv preprint arXiv:2206.05670}, 2022.

\bibitem[Cutkosky and Orabona(2019)]{cutkosky2019momentum}
A.~Cutkosky and F.~Orabona.
\newblock Momentum-based variance reduction in non-convex sgd.
\newblock \emph{Advances in neural information processing systems}, 32, 2019.

\bibitem[Fang et~al.(2018)Fang, Li, Lin, and Zhang]{fang2018spider}
C.~Fang, C.~J. Li, Z.~Lin, and T.~Zhang.
\newblock Spider: Near-optimal non-convex optimization via stochastic
  path-integrated differential estimator.
\newblock \emph{Advances in Neural Information Processing Systems}, 31, 2018.

\bibitem[Feurer and Hutter(2019)]{feurer2019hyperparameter}
M.~Feurer and F.~Hutter.
\newblock Hyperparameter optimization.
\newblock In \emph{Automated machine learning}, pages 3--33. Springer, Cham,
  2019.

\bibitem[Franceschi et~al.(2017)Franceschi, Donini, Frasconi, and
  Pontil]{franceschi2017forward}
L.~Franceschi, M.~Donini, P.~Frasconi, and M.~Pontil.
\newblock Forward and reverse gradient-based hyperparameter optimization.
\newblock In \emph{International Conference on Machine Learning}, pages
  1165--1173. PMLR, 2017.

\bibitem[Franceschi et~al.(2018)Franceschi, Frasconi, Salzo, Grazzi, and
  Pontil]{franceschi2018bilevel}
L.~Franceschi, P.~Frasconi, S.~Salzo, R.~Grazzi, and M.~Pontil.
\newblock Bilevel programming for hyperparameter optimization and
  meta-learning.
\newblock In \emph{International Conference on Machine Learning}, pages
  1568--1577. PMLR, 2018.

\bibitem[Gao(2022{\natexlab{a}})]{gao2022convergence}
H.~Gao.
\newblock On the convergence of momentum-based algorithms for federated
  stochastic bilevel optimization problems.
\newblock \emph{arXiv preprint arXiv:2204.13299}, 2022{\natexlab{a}}.

\bibitem[Gao(2022{\natexlab{b}})]{gao2022decentralized}
H.~Gao.
\newblock Decentralized stochastic gradient descent ascent for finite-sum
  minimax problems.
\newblock \emph{arXiv preprint arXiv:2212.02724}, 2022{\natexlab{b}}.

\bibitem[Gao and Huang(2020)]{gao2020periodic}
H.~Gao and H.~Huang.
\newblock Periodic stochastic gradient descent with momentum for decentralized
  training.
\newblock \emph{arXiv preprint arXiv:2008.10435}, 2020.

\bibitem[Gao and Huang(2021)]{gao2021fast}
H.~Gao and H.~Huang.
\newblock Fast training method for stochastic compositional optimization
  problems.
\newblock \emph{Advances in Neural Information Processing Systems},
  34:\penalty0 25334--25345, 2021.

\bibitem[Gao et~al.(2021)Gao, Xu, and Vucetic]{gao2021sample}
H.~Gao, H.~Xu, and S.~Vucetic.
\newblock Sample efficient decentralized stochastic frank-wolfe methods for
  continuous dr-submodular maximization.
\newblock In Z.~Zhou, editor, \emph{Proceedings of the Thirtieth International
  Joint Conference on Artificial Intelligence}, pages 3501--3507, 2021.

\bibitem[Gao et~al.(2023)Gao, Thai, and Wu]{gao2023decentralized}
H.~Gao, M.~T. Thai, and J.~Wu.
\newblock When decentralized optimization meets federated learning.
\newblock \emph{IEEE Network}, 2023.

\bibitem[Ghadimi and Wang(2018)]{ghadimi2018approximation}
S.~Ghadimi and M.~Wang.
\newblock Approximation methods for bilevel programming.
\newblock \emph{arXiv preprint arXiv:1802.02246}, 2018.

\bibitem[Grazzi et~al.(2020)Grazzi, Franceschi, Pontil, and
  Salzo]{grazzi2020iteration}
R.~Grazzi, L.~Franceschi, M.~Pontil, and S.~Salzo.
\newblock On the iteration complexity of hypergradient computation.
\newblock In \emph{International Conference on Machine Learning}, pages
  3748--3758. PMLR, 2020.

\bibitem[Guo and Yang(2021)]{guo2021randomized}
Z.~Guo and T.~Yang.
\newblock Randomized stochastic variance-reduced methods for stochastic bilevel
  optimization.
\newblock \emph{arXiv e-prints}, pages arXiv--2105, 2021.

\bibitem[Guo et~al.(2021)Guo, Xu, Yin, Jin, and Yang]{guo2021stochastic}
Z.~Guo, Y.~Xu, W.~Yin, R.~Jin, and T.~Yang.
\newblock On stochastic moving-average estimators for non-convex optimization.
\newblock \emph{arXiv preprint arXiv:2104.14840}, 2021.

\bibitem[Hong et~al.(2020)Hong, Wai, Wang, and Yang]{hong2020two}
M.~Hong, H.-T. Wai, Z.~Wang, and Z.~Yang.
\newblock A two-timescale framework for bilevel optimization: Complexity
  analysis and application to actor-critic.
\newblock \emph{arXiv preprint arXiv:2007.05170}, 2020.

\bibitem[Ji et~al.(2021)Ji, Yang, and Liang]{ji2021bilevel}
K.~Ji, J.~Yang, and Y.~Liang.
\newblock Bilevel optimization: Convergence analysis and enhanced design.
\newblock In \emph{International Conference on Machine Learning}, pages
  4882--4892. PMLR, 2021.

\bibitem[Khanduri et~al.(2021{\natexlab{a}})Khanduri, Zeng, Hong, Wai, Wang,
  and Yang]{khanduri2021momentum}
P.~Khanduri, S.~Zeng, M.~Hong, H.-T. Wai, Z.~Wang, and Z.~Yang.
\newblock A momentum-assisted single-timescale stochastic approximation
  algorithm for bilevel optimization.
\newblock \emph{arXiv e-prints}, pages arXiv--2102, 2021{\natexlab{a}}.

\bibitem[Khanduri et~al.(2021{\natexlab{b}})Khanduri, Zeng, Hong, Wai, Wang,
  and Yang]{khanduri2021near}
P.~Khanduri, S.~Zeng, M.~Hong, H.-T. Wai, Z.~Wang, and Z.~Yang.
\newblock A near-optimal algorithm for stochastic bilevel optimization via
  double-momentum.
\newblock \emph{Advances in Neural Information Processing Systems}, 34,
  2021{\natexlab{b}}.

\bibitem[Koloskova et~al.(2019)Koloskova, Stich, and
  Jaggi]{koloskova2019decentralized}
A.~Koloskova, S.~Stich, and M.~Jaggi.
\newblock Decentralized stochastic optimization and gossip algorithms with
  compressed communication.
\newblock In \emph{International Conference on Machine Learning}, pages
  3478--3487. PMLR, 2019.

\bibitem[Li et~al.(2019)Li, Yang, Wang, and Zhang]{li2019communication}
X.~Li, W.~Yang, S.~Wang, and Z.~Zhang.
\newblock Communication efficient decentralized training with multiple local
  updates.
\newblock \emph{arXiv preprint arXiv:1910.09126}, 5, 2019.

\bibitem[Li et~al.(2021)Li, Hanzely, and Richt{\'a}rik]{li2021zerosarah}
Z.~Li, S.~Hanzely, and P.~Richt{\'a}rik.
\newblock Zerosarah: Efficient nonconvex finite-sum optimization with zero full
  gradient computation.
\newblock \emph{arXiv preprint arXiv:2103.01447}, 2021.

\bibitem[Lian et~al.(2017)Lian, Zhang, Zhang, Hsieh, Zhang, and
  Liu]{lian2017can}
X.~Lian, C.~Zhang, H.~Zhang, C.-J. Hsieh, W.~Zhang, and J.~Liu.
\newblock Can decentralized algorithms outperform centralized algorithms? a
  case study for decentralized parallel stochastic gradient descent.
\newblock \emph{Advances in Neural Information Processing Systems}, 30, 2017.

\bibitem[Liu et~al.(2018)Liu, Simonyan, and Yang]{liu2018darts}
H.~Liu, K.~Simonyan, and Y.~Yang.
\newblock Darts: Differentiable architecture search.
\newblock \emph{arXiv preprint arXiv:1806.09055}, 2018.

\bibitem[Lu et~al.(2019)Lu, Zhang, Sun, and Hong]{lu2019gnsd}
S.~Lu, X.~Zhang, H.~Sun, and M.~Hong.
\newblock Gnsd: A gradient-tracking based nonconvex stochastic algorithm for
  decentralized optimization.
\newblock In \emph{2019 IEEE Data Science Workshop (DSW)}, pages 315--321.
  IEEE, 2019.

\bibitem[Mokhtari et~al.(2018)Mokhtari, Hassani, and
  Karbasi]{mokhtari2018decentralized}
A.~Mokhtari, H.~Hassani, and A.~Karbasi.
\newblock Decentralized submodular maximization: Bridging discrete and
  continuous settings.
\newblock In \emph{International Conference on Machine Learning}, pages
  3616--3625. PMLR, 2018.

\bibitem[Pu and Nedi{\'c}(2021)]{pu2021distributed}
S.~Pu and A.~Nedi{\'c}.
\newblock Distributed stochastic gradient tracking methods.
\newblock \emph{Mathematical Programming}, 187\penalty0 (1):\penalty0 409--457,
  2021.

\bibitem[Rajeswaran et~al.(2019)Rajeswaran, Finn, Kakade, and
  Levine]{rajeswaran2019meta}
A.~Rajeswaran, C.~Finn, S.~M. Kakade, and S.~Levine.
\newblock Meta-learning with implicit gradients.
\newblock \emph{Advances in neural information processing systems}, 32, 2019.

\bibitem[Sun et~al.(2020)Sun, Lu, and Hong]{sun2020improving}
H.~Sun, S.~Lu, and M.~Hong.
\newblock Improving the sample and communication complexity for decentralized
  non-convex optimization: Joint gradient estimation and tracking.
\newblock In \emph{International conference on machine learning}, pages
  9217--9228. PMLR, 2020.

\bibitem[Tang et~al.(2019)Tang, Lian, Qiu, Yuan, Zhang, Zhang, and
  Liu]{tang2019deepsqueeze}
H.~Tang, X.~Lian, S.~Qiu, L.~Yuan, C.~Zhang, T.~Zhang, and J.~Liu.
\newblock Deepsqueeze: Decentralization meets error-compensated compression.
\newblock \emph{arXiv preprint arXiv:1907.07346}, 2019.

\bibitem[Tsaknakis et~al.(2020)Tsaknakis, Hong, and
  Liu]{tsaknakis2020decentralized}
I.~Tsaknakis, M.~Hong, and S.~Liu.
\newblock Decentralized min-max optimization: Formulations, algorithms and
  applications in network poisoning attack.
\newblock In \emph{ICASSP 2020-2020 IEEE International Conference on Acoustics,
  Speech and Signal Processing (ICASSP)}, pages 5755--5759. IEEE, 2020.

\bibitem[Vogels et~al.(2020)Vogels, Karimireddy, and
  Jaggi]{vogels2020powergossip}
T.~Vogels, S.~P. Karimireddy, and M.~Jaggi.
\newblock Powergossip: Practical low-rank communication compression in
  decentralized deep learning.
\newblock \emph{arXiv preprint arXiv:2008.01425}, 2020.

\bibitem[Wai et~al.(2017)Wai, Lafond, Scaglione, and
  Moulines]{wai2017decentralized}
H.-T. Wai, J.~Lafond, A.~Scaglione, and E.~Moulines.
\newblock Decentralized frank--wolfe algorithm for convex and nonconvex
  problems.
\newblock \emph{IEEE Transactions on Automatic Control}, 62\penalty0
  (11):\penalty0 5522--5537, 2017.

\bibitem[Xian et~al.(2021)Xian, Huang, Zhang, and Huang]{xian2021faster}
W.~Xian, F.~Huang, Y.~Zhang, and H.~Huang.
\newblock A faster decentralized algorithm for nonconvex minimax problems.
\newblock \emph{Advances in Neural Information Processing Systems},
  34:\penalty0 25865--25877, 2021.

\bibitem[Xin et~al.(2021)Xin, Khan, and Kar]{xin2021hybrid}
R.~Xin, U.~Khan, and S.~Kar.
\newblock A hybrid variance-reduced method for decentralized stochastic
  non-convex optimization.
\newblock In \emph{International Conference on Machine Learning}, pages
  11459--11469. PMLR, 2021.

\bibitem[Yang et~al.(2021)Yang, Ji, and Liang]{yang2021provably}
J.~Yang, K.~Ji, and Y.~Liang.
\newblock Provably faster algorithms for bilevel optimization.
\newblock \emph{Advances in Neural Information Processing Systems}, 34, 2021.

\bibitem[Yang et~al.(2022)Yang, Zhang, and Wang]{yang2022decentralized}
S.~Yang, X.~Zhang, and M.~Wang.
\newblock Decentralized gossip-based stochastic bilevel optimization over
  communication networks.
\newblock \emph{arXiv preprint arXiv:2206.10870}, 2022.

\bibitem[Yu et~al.(2019)Yu, Jin, and Yang]{yu2019linear}
H.~Yu, R.~Jin, and S.~Yang.
\newblock On the linear speedup analysis of communication efficient momentum
  sgd for distributed non-convex optimization.
\newblock In \emph{International Conference on Machine Learning}, pages
  7184--7193. PMLR, 2019.

\bibitem[Zhan et~al.(2022)Zhan, Wu, and Gao]{zhan2022efficient}
W.~Zhan, G.~Wu, and H.~Gao.
\newblock Efficient decentralized stochastic gradient descent method for
  nonconvex finite-sum optimization problems.
\newblock In \emph{Proceedings of the AAAI Conference on Artificial
  Intelligence}, volume~36, pages 9006--9013, 2022.

\bibitem[Zhang et~al.(2021{\natexlab{a}})Zhang, Liu, Zhu, and
  Bentley]{zhang2021low}
X.~Zhang, J.~Liu, Z.~Zhu, and E.~S. Bentley.
\newblock Low sample and communication complexities in decentralized learning:
  A triple hybrid approach.
\newblock In \emph{IEEE INFOCOM 2021-IEEE Conference on Computer
  Communications}, pages 1--10. IEEE, 2021{\natexlab{a}}.

\bibitem[Zhang et~al.(2021{\natexlab{b}})Zhang, Liu, Liu, Zhu, and
  Lu]{ZhangLLZL21}
X.~Zhang, Z.~Liu, J.~Liu, Z.~Zhu, and S.~Lu.
\newblock Taming communication and sample complexities in decentralized policy
  evaluation for cooperative multi-agent reinforcement learning.
\newblock In \emph{NeurIPS}, pages 18825--18838, 2021{\natexlab{b}}.

\end{thebibliography}
